%% file: iclr2027_conference.tex
\documentclass{article} % For LaTeX2e
\usepackage{iclr2027_conference,times}

\input{math_commands.tex}

\usepackage[hidelinks]{hyperref}
\usepackage{url}
\usepackage{comment}
\usepackage{xcolor}

\usepackage{algorithm}
\usepackage{algpseudocode}

\usepackage{graphicx}
\usepackage{booktabs}
\usepackage{amssymb, amsthm}
\usepackage{dsfont}
\usepackage{enumitem}
\usepackage{wrapfig}
\usepackage{multicol, multirow, array}
\usepackage{etoc}
\usepackage{adjustbox}
\usepackage{placeins}

\newtheorem{lemma}{Lemma}

\newtheorem{theorem}{Theorem}
\newtheorem{remark}{Remark}
\newtheorem{corollary}{Corollary}
\newtheorem{assumption}{Assumption}

\renewcommand{\eqref}[1]{Eq.~(\ref{#1})}

\title{ParetoTransport: Generative Optimization by Mass Transport Toward The Pareto Front}

\author{
\makebox[\textwidth][c]{%
Stephanie Holly\thanks{Corresponding author: \texttt{holly@ml.jku.at}}
\qquad
Sepp Hochreiter
\qquad
Werner Zellinger
}\\
\makebox[\textwidth][c]{LIT AI Lab and Institute for Machine Learning}\\
\makebox[\textwidth][c]{JKU Linz, Austria}
}

\iclrfinalcopy % Uncomment for camera-ready version, but NOT for submission.
\begin{document}

\maketitle
\fancyhead{}

\begin{abstract}
Offline multi-objective optimization requires not only moving the objective vectors of candidate designs toward the Pareto front, but also distributing them effectively along it. Generative methods have recently emerged as a natural approach because they learn a distribution over feasible designs while allowing generation to be steered toward promising designs.
Existing methods, however, largely retain classical sample-wise guidance strategies, leaving the distribution-level modeling capability of generative methods underused. 
We propose ParetoTransport, a training-free guidance method for pre-trained flow-matching models that explicitly specifies and refines a population-level distribution in objective space. 
ParetoTransport guides a flow-matching sampler to iteratively transport the empirical offline distribution toward the Pareto front, with Wasserstein matching to intermediate proxy distributions.
This directly controls distributional displacement and mass allocation along the front. We establish a convergence result and demonstrate state-of-the-art performance on standard offline MOO benchmarks, extending recent evaluations beyond hypervolume to generational distance, inverted generational distance, and Wasserstein distance.
\end{abstract}

\section{Introduction}
\label{sec:intro}

Offline multi-objective optimization (MOO) seeks a set of Pareto-optimal designs that represent diverse trade-offs among multiple potentially conflicting objectives, while optimization relies only on a fixed dataset of previously evaluated designs, without access to additional objective evaluations.
This setting is particularly relevant when evaluating new designs is expensive or impractical, as in molecular and materials design, where efficacy, toxicity, synthesizability, strength, conductivity, or cost must be balanced
~\citep{Wang2021naturedistillation, Du2024naturesurvey, Pogue2023naturesuperconducting, Zeni2025naturematerial}.

Recently, diffusion and flow-matching models have emerged as promising approaches to offline MOO because they learn a distribution over feasible designs and enable an approximate inverse relation from objective values to feasible designs.
Training-free guidance adapts such pre-trained models to task-specific objectives at inference time, separating generative modeling from task-specific optimization and allowing the model to be reused across optimization tasks without retraining. \citep{Benhamu2024dflow, Li2026hardflow}. 
At the population level, this turns optimization into the problem of controlling the induced objective-space distribution: its displacement toward the Pareto front and its allocation of probability mass along it.

Although methods have been proposed to control displacement and mass allocation during generation, most of them control the induced objective-space distribution only indirectly through sample-wise mechanisms. 
Typical approaches steer individual samples through (a) sample-wise update directions derived from scalarized objectives~\citep{Yuan2025paretoflow} or preference models~\citep{Annadani2025pgd}, (b) diversity based on repulsive interactions~\citep{Hotegni2026spread}, or (c) selected conditioning targets~\citep{Shrestha2026paretoconditioned}.
All of these methods have in common that they do not explicitly
specify and refine the geometry and mass allocation of the generated objective-space distribution, although such distribution-level control is a natural capability of generative inverse modeling.

\begin{figure}[t]
  \centering
  \includegraphics[width=\linewidth]{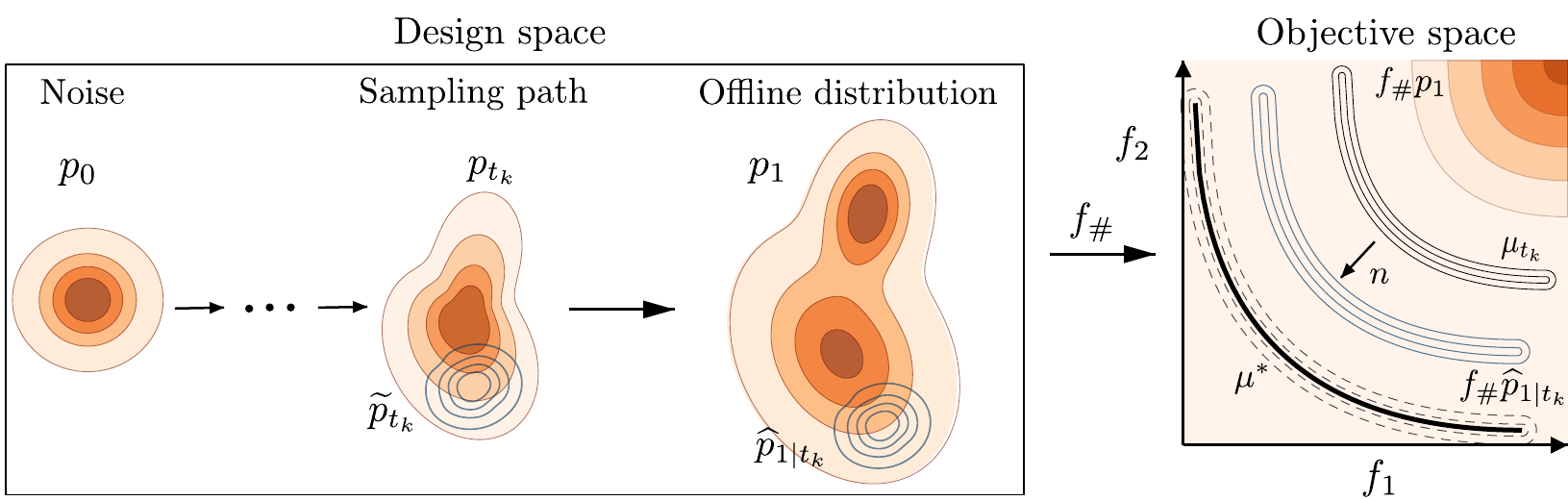}
  \setlength{\abovecaptionskip}{0pt}
  \caption{ParetoTransport steers the probability path $(p_{t_k})_{k=1}^K$ of a pre-trained flow model from its learned offline distribution $p_1$ toward a guided distribution, whose induced objective-space distribution approximates the uniform distribution $\mu^\ast$ on the Pareto front. At each step $t_k$, the estimated terminal objective-space distribution $f_\# \widehat{p}_{1\mid t_k}$ is transported toward the Pareto front through directional improvement along $n$, while Wasserstein matching to the current proxy $\mu_{t_k}$ controls its mass allocation.}
  \label{fig1}
\end{figure}

In this work, we introduce ParetoTransport, a training-free guidance method for pre-trained flow matching models that explicitly refines the geometry and mass allocation of the generated objective-space distribution. We start by formulating offline MOO as a mass-transport problem.  
ParetoTransport approaches this problem by iteratively transporting the empirical offline distribution toward the Pareto front, combining directional improvement
with Wasserstein matching to control probability-mass allocation. %
Under general assumptions, we establish a Wasserstein convergence guarantee toward the Pareto front distribution.
Across standard offline MOO benchmarks, ParetoTransport achieves state-of-the-art performance among training-free generative inverse methods and competitive performance relative to forward methods, and extends recent evaluations beyond hypervolume (HV) to generational distance (GD), inverted generational distance (IGD), and the Wasserstein distance $W_2$.

Our contributions are summarized as follows:
\begin{itemize}
    \item We formulate offline MOO as a sampling problem and motivate $W_2$ as a natural quality measure by relating it to the classical MOO metrics GD and IGD (Corollary~\ref{corollary:ot}).
    
    \item We propose ParetoTransport, a training-free guidance method for pre-trained flow matching models that approaches offline MOO by mass transport in objective space, iterating (i) Wasserstein control of probability-mass allocation with (ii) normal-direction improvement. 
    
    \item Under general assumptions, we establish a Wasserstein 
    convergence result toward the Pareto front distribution (Theorem~\ref{thm:wasserstein-contraction}).
    
    \item We demonstrate state-of-the-art performance across $24$ datasets, extending recent generative evaluations beyond HV to the classical MOO metrics GD and IGD, as well as to the $2$-Wasserstein distance between
    the generated objective-space distribution and the Pareto front.
\end{itemize}

\section{Problem: Offline Multi-Objective Optimization}
\label{sec:problem}
The goal of multi-objective optimization (MOO) is 
to identify designs $x\in\mathcal{X}\subset\mathbb{R}^d$ that minimize multiple objective functions $f_i:\mathcal{X} \rightarrow \mathbb{R}$, $i=1,\ldots,m$. Since the objective functions may conflict, there is generally no unique minimizer,
and MOO instead aims to identify the set of \textit{Pareto-optimal} designs.

A design $x \in \mathcal{X}$ is said to \textit{dominate} a design $x^\prime \in \mathcal{X}$ if
$f_i(x)\leq f_i(x^\prime)$ for all $i\in\{1,\ldots,m\}$, with strict inequality for at least one $i$.
A design $x\in \mathcal{X}$ is \textit{Pareto-optimal} if no other design $x^\prime\in \mathcal{X}$ dominates it.
The set of all Pareto-optimal designs is called the \textit{Pareto set}, and its image under
$f:=(f_1,\ldots,f_m)$ is the \textit{Pareto front}, denoted by $\mathcal{M}^\ast$.

\textit{Offline} MOO considers scientific and engineering problems in which optimization relies only on a fixed dataset\footnote{Following standard learning-theoretic assumptions~\citep{mohri2018foundations,shalevshwartz2014understanding}, we regard the designs $x_1,\ldots,x_N$ as independent and identically distributed samples from an unknown data distribution $p_1$ on $\mathcal{X}$.} $(x_i,f(x_i))_{i=1}^N$ of previously evaluated designs, without access to further evaluations of the objective function $f$.

\section{Background: Training-Free Guidance for Flow Matching}
The goal of flow matching (FM) \citep{Lipman2023flowmatching, Albergo2023building, Liu2022flow} is to learn a time-dependent velocity field $u: [0,1] \times \mathbb{R}^d \rightarrow \mathbb{R}^d$ that transforms a simple \textit{initial} distribution $p_0$ into a complex \textit{terminal} distribution $p_1$. The flow map $\phi: [0,1] \times \mathbb{R}^d \rightarrow \mathbb{R}^d$ is defined by $\phi_t(a):=x_t$, where $x:[0,1]\rightarrow\mathbb{R}^d$ is the solution of the ordinary differential equation (ODE)
\begin{equation}
\label{eq:flow_ode}
\frac{d x_t}{dt} = u_t(x_t), \qquad x_0=a.
\end{equation}
This flow map induces a probability path $p_t:=(\phi_t)_{\#}p_0$ that satisfies the continuity equation \mbox{$\partial_t p_t + \nabla\cdot(u_t p_t)=0$}.
The velocity field $u_t$ is learned through a conditional probability path $p_{t|Z}$, where $Z=(X_0,X_1) \sim \pi$ and $\pi \in \Gamma(p_0,p_1)$. A standard choice is an affine conditional probability path $X_t:= \alpha_t X_1 + \beta_t X_0 \sim p_{t|Z}$ with scheduler $(\alpha_t, \beta_t)$ satisfying $\alpha_0 = \beta_1 = 0, \alpha_1=\beta_0=1$. This yields the simple conditional velocity field $u_{t\mid Z}(X_t)=\dot{\alpha}_t X_1 + \dot{\beta}_t X_0$. The resulting conditional regression objective has the same gradient as the marginal regression objective \citep{Lipman2023flowmatching}, and therefore can be used to learn the marginal velocity field $u_t$.
At inference time, sampling is performed by drawing an initial sample $x_0\sim p_0$ and integrating the ODE in \eqref{eq:flow_ode} with the learned velocity field. Further details are provided in Appendix~\ref{app:flow_matching}.

\textbf{Training-Free Guidance} incorporates a task-specific loss $J:\mathbb{R}^d \rightarrow \mathbb{R}$ into the sampling path of a pre-trained generative model at inference time~\citep{Liu2023flowgrad, Benhamu2024dflow, Wang2025ocflow}. Given an initial sample $x_0 \sim p_0$, the guided sampling process can be formulated as 
\begin{align}
\label{eq:traj-opt}
\begin{aligned}
\min_{(\widetilde{x}_t,\,c_t)_{t\in [0,1]}}
\quad &
J(\widetilde{x}_1)
+\lambda\int_{0}^{1}\frac{1}{2}\|c_t\|^2\,dt
\\
&
\begin{aligned}[t]
\text{s.t.}\quad
&\widetilde{x}_0=x_0,\\
&
\frac{d \widetilde x_t}{dt} = u_t(\widetilde x_t) + c_t \qquad t\in[0,1] .
\end{aligned}
\end{aligned}
\end{align}
The optimization seeks a guidance velocity field $c: [0,1] \rightarrow \mathbb{R}^d$ that steers the induced trajectory $(\widetilde{x}_t)_{t\in [0,1]}$ toward a favorable terminal sample $\widetilde{x}_1$. The loss function $J$ evaluates this terminal sample, while the second term penalizes large guidance-induced deviations from the pre-trained sampling dynamics.

Recent work~\citep{Li2026hardflow, Webber2026flowmpc, Pourya2026flower} approximates the full-horizon optimization problem in \eqref{eq:traj-opt} by a sequence of single-step subproblems. At each intermediate time $t$, a posterior mean estimator, denoted by $\widehat{x}_{1\mid t}$, estimates the terminal sample at which the terminal loss is evaluated, with its gradient propagated back to the intermediate sample. \citet{Li2026hardflow} and \citet{Pourya2026flower} additionally derive a local mapping from terminal space back to time $t$, which we denote by $\widehat{x}_{t\mid 1}$. %and refer to as a local reverse parameterization. 
This allows each subproblem to be optimized directly in terminal space, avoiding differentiation through the terminal estimator $\widehat{x}_{1\mid t}$ and hence through the learned model. Further details are provided in Appendix~\ref{app:flow_matching}.

\section{Offline MOO as Sampling Problem}
\label{sec:oflline_moo_as_sampling_problem}
In this work, we view the MOO goal of identifying Pareto-optimal designs (see Section~\ref{sec:problem}) as a statistical sampling problem.
This view implies a refined formulation of the goal of MOO algorithms \citep{CHEN2011classicalmoo}, namely, to identify solution sets that simultaneously achieve convergence toward the Pareto front and uniform coverage along it.

Viewing a solution set as a statistical sample rather than only as a finite set, the convergence of solution sets becomes concentration of the samples' distribution near the Pareto front and coverage becomes allocation of probability mass along it.
Accordingly,
\textit{the goal of MOO algorithms becomes to find a design-space distribution $p$ whose induced objective-space distribution $f_{\#}p$ approximates the uniform distribution $\mu^{\ast}$ on the Pareto front $\mathcal{M}^{\ast}$.}

\subsection{A Quality Measure: The Wasserstein Distance}

The sampling view of MOO above leads to a natural quality measure for generative MOO methods: the $2$-Wasserstein distance\footnote{The squared $2$-Wasserstein distance between probability distributions $\mu,\nu\in\mathcal P_2(\mathbb R^m)$ is defined as
$
W_2^2(\mu,\nu)
:=
\inf_{\gamma\in\Gamma(\mu,\nu)}
\int \|x-y\|^2\,d\gamma(x,y),
$
where $\Gamma(\mu,\nu)$ denotes the set of couplings of $\mu$ and $\nu$.} $W_2$.
Indeed, the 2-Wasserstein distance is related to the two well-established MOO metrics GD and IGD\footnote{The squared GD~\citep{Lamont1999GD} between a solution set $A$ and a discrete approximation $Z$ of the Pareto front is defined by $\mathrm{GD}^2(A,Z) := \frac{1}{|A|} \sum_{a \in A} d(a,Z)^2$ The squared IGD~\citep{Coello2004IGD} is defined by $\mathrm{IGD}^2(A,Z):=\mathrm{GD}^2(Z,A)$.} as follows.
For a generative MOO method that learned to sample from a distribution $p$ with objective-space distribution $\mu:=f_{\#}p$, the value $W_2(\mu,\mu^\ast)$ upper bounds the continuous versions of GD and IGD,
\[
\mathrm{GD}(\mu\|\mu^\ast)
:=
\left(\int d(x,\operatorname{supp}\mu^\ast)^2\,d\mu(x)\right)^{1/2},
\qquad
\mathrm{IGD}(\mu\|\mu^\ast)
:= \mathrm{GD}(\mu^\ast\|\mu),
\]
where \(d(x,A):=\inf_{a\in A}\|x-a\|\) is the distance from \(x\) to the set \(A\).%
~\\

\begin{corollary}
\label{corollary:ot}
For $\mu,\nu\in\mathcal P_2(\mathbb R^m)$,
$\mathrm{GD}(\mu\|\nu)\leq W_2(\mu,\nu)$ and
$\mathrm{IGD}(\mu\|\nu)\leq W_2(\mu,\nu)$.
\end{corollary}
% $W_2^2(\cdot,\mu)$ penalizes not only deviation from $\mu$ but also how mass is distributed along it.
The proof of our Corollary~\ref{corollary:ot} is provided in
Appendix~\ref{app:proof}.

\section{Our Method: Offline MOO by Mass Transport}
\label{sec:mass_transport}

The formulation of MOO as a sampling problem in Section~\ref{sec:oflline_moo_as_sampling_problem} motivates us to approach offline MOO by mass transport, which leads to our novel algorithm: \textit{ParetoTransport}.

\subsection{Approach by Mass Transport}

Denote by $\mu_0 \in\mathcal{P}_2(\mathbb{R}^m)$ the empirical distribution over the offline front, i.e., $\mu_0:=\frac{1}{|D(\mathrm{best})|} \sum_{x\in D(\mathrm{best})} \delta_{f(x)}$, where $D(\mathrm{best})$ denotes the set of all non-dominated designs in the offline dataset. Further, denote by $\mu^\ast$ the uniform distribution on the (unknown) Pareto front $\mathcal M^\ast$.

\textit{Following the sampling view of MOO (Section~\ref{sec:oflline_moo_as_sampling_problem}), our approach is to iteratively transport $\mu_0$ toward $\mu^\ast$ in terms of the 2-Wasserstein distance in objective space.}

\begin{figure}[!htbp]
  \centering
  \includegraphics[width=0.65\linewidth]{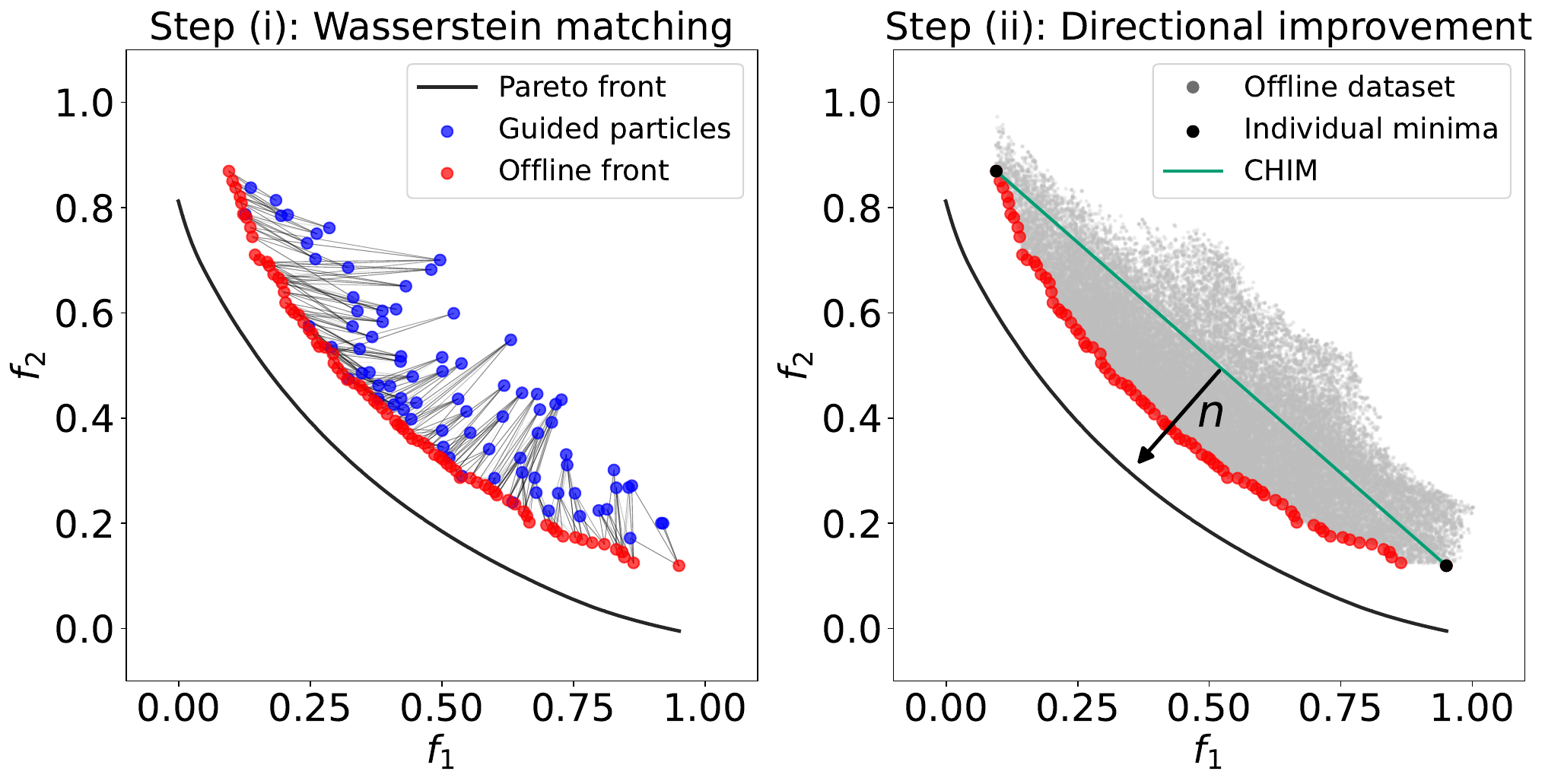}
    \caption{\textbf{Step~(i):} Wasserstein matching of probability-mass allocation. \textbf{Step~(ii):} Normal-direction improvement with the CHIM normal $n$ as a surrogate for the unknown projection direction toward the Pareto front.}
  \label{fig:mechanism}
\end{figure}

To realize this transport, we start from the offline front distribution $\mu_0$ and introduce a sequence $(\mu_k)_{k\geq0}$ of intermediate proxy distributions by alternating between two steps:
\[
\mu_k
\xrightarrow{\text{step~(i):~$W_2$ matching}}
q_{k+1}
\xrightarrow{\text{step~(ii):~directional improvement}}
\mu_{k+1}.
\]
Step (i) is a matching step that aims to find a design distribution\footnote{We  assume that $q$ is chosen from a set $\mathcal{Q}$ of admissible design-space distributions for which~\eqref{eq:matching} admits a minimizer, e.g., $\mathcal{Q}\subseteq\mathcal P_2(\mathcal X)$ is nonempty and weakly closed, with compact $\mathcal X \subseteq \mathbb R^d$ and continuous $f:\mathcal X\to\mathbb R^m$ \citep{villani2009optimal}.}
\begin{equation}
\label{eq:matching}
q_{k+1}\in
\operatorname*{argmin}_{q\in \mathcal Q}
W_2^2(f_\#q,\mu_k)
\end{equation}
whose induced objective-space distribution $f_\# q_{k+1}$ matches the current proxy $\mu_k$.
Step (ii) is a directional improvement step that moves $f_\# q_{k+1}$ toward the Pareto front and defines the next intermediate distribution $\mu_{k+1}:=(T_\eta\circ f)_\#q_{k+1}$ by the map
\begin{equation}
\label{eq:projection-update}
T_\eta(y):=y+\eta\bigl(\Pi_{\cal M^\ast}(y)-y\bigr),
\end{equation}
where $\eta\in(0,1]$ and $\Pi_{M^\ast}$ denotes the projection onto $\cal M^\ast$. Thus, $T_\eta$ moves each objective vector by a fraction $\eta$ of the displacement toward its closest point on the Pareto front.

The two-step procedure described above solves the MOO problem in the following sense, where we use $\bar{\mu}:=(\Pi_{\mathcal M^\ast})_\#\mu_0$ to denote the pushforward of the offline front distribution $\mu_0$ under the projection $\Pi_{\mathcal{M}^\ast}$.

\begin{theorem}
\label{thm:wasserstein-contraction}
If Step (i) achieves exact matching, $W_2(f_\#q_{k+1},\mu_k)=0$ for all $k$, and the technical Assumptions (1)--(3) in Appendix~\ref{app:proof} hold, then
\begin{equation}
\label{eq:w2-bound}
W_2(\mu_k,\mu^\ast)
\leq
(1-\eta)^k
W_2(\mu_0,\bar\mu)
+
W_2(\bar\mu,\mu^\ast).
\end{equation}
\end{theorem}
The proof is provided in Appendix~\ref{app:proof}. 

Theorem~\ref{thm:wasserstein-contraction} provides a quality guarantee for our iterative procedure.
Assume that the offline front distribution is sufficiently representative in the sense that its projection $\bar\mu$ onto the Pareto front satisfies $W_2(\bar\mu,\mu^\ast)<\epsilon$ for some small $\epsilon$.
Then, Theorem~\ref{thm:wasserstein-contraction} guarantees
$\limsup_{k\to\infty} W_2(\mu_k,\mu^\ast)<\epsilon$.
Moreover, the first term in Eq.~\ref{eq:w2-bound} decreases exponentially with $k$.

\subsection{Our ParetoTransport Algorithm}
\label{sec:ParetoTransport}
We introduce \textit{ParetoTransport}, a training-free guidance method for pre-trained flow matching models that lifts guidance from individual sampling trajectories to the induced probability path.
Existing training-free generative methods for offline MOO are sample-wise:
ParetoFlow~\citep{Yuan2025paretoflow} uses sample-wise scalarization weights to guide sampling from a flow matching model, while PGD~\citep{Annadani2025pgd} uses a preference-based classifier to guide sampling from a diffusion model.
ParetoTransport instead formulates guidance over the induced probability path, combining Wasserstein control with directional transport to realize the mass-transport iteration (Section~\ref{sec:mass_transport}) during sampling.

In particular, we lift the guided sampling problem in \eqref{eq:traj-opt} from an individual trajectory $(\widetilde x_t)_{t\in[0,1]}$ to the probability path $(\widetilde p_t)_{t\in[0,1]}$:
\begin{equation}
\label{eq:distributional-opt}
\begin{aligned}
\min_{\big(\widetilde p_t(\cdot),\,c_t(\cdot)\big)_{t\in[0,1]}}
\quad &
J(\widetilde p_1)
+
\lambda
\int_0^1
\int
\frac{1}{2}\|c_t(x)\|^2
\,d\widetilde p_t(x)\,dt
\\
\text{s.t.}\quad &
\widetilde p_0=p_0,
\\
&
\partial_t\widetilde p_t
+
\nabla\cdot\big((u_t+c_t)\widetilde p_t\big)
=0,
\qquad t\in[0,1].
\end{aligned}
\end{equation}
The optimization seeks a guidance velocity field $c_t:\mathbb R^d\to\mathbb R^d$ that steers the induced probability path
$(\widetilde p_t)_{t\in[0,1]}$ toward a favorable terminal distribution $\widetilde p_1$.
The functional $J:\mathcal{P}_2(\mathbb{R}^d)\to\mathbb{R}$ evaluates this terminal distribution, while the second term penalizes large guidance-induced deviations from the pre-trained sampling dynamics.

\begin{algorithm}[htbp!]
\caption{ParetoTransport}
\label{alg:FancyName}
\begin{algorithmic}[1]
\State Initialize guided particles $\widetilde x_{t_1}^{i}:=x_0^{i}, \qquad i=1,\dots,N$
%\State Initialize archive $(y_{t_1}^j)_{j=1}^M:=(y_0^j)_{j=1}^M$
\State Initialize proxy $\mu_{t_1}:=\mu_0$
\For{$k=1$ to $K-1$}
    \State Take one Euler step: $x_{t_{k+1}}^{i}:= \widetilde x_{t_{k}}^{i} + \Delta t_k u_{t_k}^{\theta}(\widetilde x_{t_{k}}^{i}), \qquad i=1,\dots,N$
    \If{$t_{k+1} \leq t_{\mathrm{guided}}$}
        \State $\widetilde x_{t_{k+1}}^{i}:=x_{t_{k+1}}^{i}, \qquad i=1,\dots,N$
    \Else
        \State Estimate terminal particles:
        $\widehat{x}_{1}^{i}:=\widehat{x}^{\theta}_{1\mid t_{k+1}}(x_{t_{k+1}}^{i}), \qquad i=1,\dots,N$
        \State Solve the optimization problem in terminal space:
        \[
        \big(x^{i}_{\mathrm{opt}}\big)_{i=1}^N
        \in
        \arg\min_{(x^{i}_1)_{i=1}^N}
        J_k\big(\frac{1}{N}\sum_{i=1}^N \delta_{x^{i}_1}\big)
        %D(\frac{1}{N} \sum_{i=1}^N \delta_{x^{i}_1}\|\mu_{t_k})
        + \frac{\lambda}{N}\sum_{i=1}^N \frac{\alpha_{t_{k+1}}^2}{2 \Delta t_k} \|x_1^{i} -\widehat{x}^{i}_{1}\|^2,
        \]
        \State Update proxy: $\mu_{t_{k+1}}:=\mathcal{R}\big(\mu_{t_k},\big(f(x^{i}_{\mathrm{opt}})\big)_{i=1}^N\big)$   
        \State Recover guided intermediate particles: 
        $\widetilde x_{t_{k+1}}^{i}:=\widehat{x}_{t_{k+1}\mid 1}(x_{\mathrm{opt}}^{i}), \qquad i=1,\dots,N$
    \EndIf 
\EndFor
\State \Return Guided terminal particles $(\widetilde x_1^{i})_{i=1}^{N}$
\end{algorithmic}
\end{algorithm}

In practice, we discretize the probability path on a grid $0=t_1<\ldots<t_K=1$ with $\Delta t_k:=t_{k+1}-t_k$, and use forward Euler. We introduce a guidance onset time $t_{\mathrm{guided}}$: for $t\leq t_{\mathrm{guided}}$, sampling consists of plain Euler updates, while for $t > t_{\mathrm{guided}}$, each sampling step $t_k\to t_{k+1}$ is coupled with an iteration of the mass-transport procedure. We denote the corresponding intermediate proxy by $\mu_{t_k}$.
%At iteration $k$, Wasserstein matching controls probability-mass allocation relative to $\mu_{t_k}$, while the transport direction promotes improvement.
We realize the matching and directional improvement steps jointly through design-space optimization rather than by explicitly transporting the proxy in objective space.
This avoids constructing increasingly optimistic objective-space vectors that need not correspond to feasible designs. At each sampling step $k$, we define the loss functional in \eqref{eq:distributional-opt} as
%the directed loss functional \(D:\mathcal{P}_2(\mathbb{R}^m)\times\mathcal{P}_2(\mathbb{R}^m)\rightarrow \mathbb{R}\),
\begin{equation}
\label{eq:directed-discrepancy}
J_k(q)
:=
W_2^2(f_\#q,\mu_{t_k})
+ \gamma L_n(f_\#q\| \mu_{t_k}),
\qquad
L_n(\nu\| \mu):=
-
\left(
\mathbb E_{Y\sim\nu}[Y]
-
\mathbb E_{Y\sim\mu}[Y]
\right)^\top n ,
\end{equation}
where $q$ denotes the terminal design-space distribution.
Thus, the terminal loss functional is updated at each sampling step according to the current proxy distribution $\mu_{t_k}$.
The Wasserstein term controls the probability-mass allocation in terms of the current proxy, while the directional term encourages transport along the surrogate projection direction $n$. Since the projection direction is unknown, we use the CHIM (Convex Hull of Individual Minima; \citealp{Das1998nbi}) normal $n$ as a global surrogate\footnote{If the individual-minimum objective vectors are affinely independent, then the CHIM spans an $(m-1)$-dimensional affine hyperplane with a unique normal direction up to orientation. We resolve the sign ambiguity by choosing the orientation with nonnegative inner product with the simultaneous-decrease direction
$-\mathbf 1_m$.}.
We use the same notation $\gamma$ because it controls the strength of the directional improvement along the normal direction, analogously to the displacement fraction in the transport map of Section~\ref{sec:mass_transport}.

We represent each intermediate distribution \(\widetilde p_{t_k}\) by a particle system $(\widetilde x^{\smash{i}}_{t_k})_{i=1}^N$, approximating the probability path by a set of particle trajectories. Following the single-trajectory approximation of~\citet{Li2026hardflow, Pourya2026flower}, we apply the same construction to the particle system, replacing the full-horizon optimization problem by a sequence of single-step subproblems. 
The particle trajectories are coupled through the distributional loss. %For the Wasserstein term, let $(x_1^{\smash{i}})_{i=1}^N$ denote terminal samples in design space and $(y^{\smash{j}}_{t_k})_{j=1}^M$ the support of the current proxy in objective space.
We denote by $\mathcal R$ the proxy-update operator that combines the
current proxy with the optimized objective vectors to construct the next proxy distribution, updating its support and assigning probability masses according to local crowding so that sparsely represented regions receive more mass and densely represented regions receive less. 
The pseudocode is provided in Algorithm~\ref{alg:FancyName}; further details on its derivation and implementation are provided in Appendices~\ref{app:flow_matching} and~\ref{app:algorithmic-details}.

\section{Related Work}

A common forward paradigm in offline optimization is to learn surrogate models that map candidate designs to objective values and then optimize these surrogates to identify promising designs~\citep{Xue2024OFFMOO, Kim2025MOOReview}. 
In single-objective optimization, surrogates enable gradient-based optimization \citep{Trabucco2022DesignBench} and evolutionary computation~\citep{Jones1998}.
In offline MOO, this extends to multiple objectives by learning one or more surrogate models and optimizing them with multi-objective search algorithms, commonly evolutionary algorithms~\citep{Knowles2006, jin2011, wang2018, yang2019}, such as NSGA-II~\citep{Deb2002NSGA2}. Recently, generative methods have emerged as an inverse paradigm, using desired objective information to directly generate candidate designs. This can be realized through conditional generative modeling, where objective information is already incorporated during training \citep{Kumar2020Gen,Hotegni2026spread, Shrestha2026paretoconditioned}, or through training-free guidance, where objective information steers a pre-trained unconditional generative model only during sampling \citep{Yuan2025paretoflow, Annadani2025pgd}. Training-free guidance is particularly attractive because it separates generative modeling from task-specific optimization, allowing the same pre-trained generative model to be reused across different optimization objectives without retraining.

\section{Experiments}
We evaluate our method on standard offline MOO benchmarks \citep{Xue2024OFFMOO}, ranging from synthetic tasks \citep{Zitzler2000ZDT, Deb2002DTLZ} to real-world engineering tasks \citep{Tanabe2020RE}.
%, and neural architecture search (\textbf{MO-NAS}) tasks. 
Each task provides an offline dataset of $60$k designs and their corresponding objective vectors. 
The Synthetic benchmark suite contains $12$ tasks widely used in MOO, including the \textbf{ZDT} \citep{Zitzler2000ZDT} and the \textbf{DTLZ} \citep{Deb2002DTLZ} tasks. The objective functions are analytical, and their Pareto fronts are known in closed form. Each task has $2$--$3$ objectives and $10$--$30$ design variables.
The \textbf{RE} benchmark suite \citep{Tanabe2020RE} contains $15$ tasks based on real-world engineering design problems, including bar truss, welded beam, and disc brake design. Each task has $2$--$6$ objectives and $3$--$7$ design variables. For evaluation, we use Pareto front approximations provided with the RE benchmark. 
Further details on the benchmarks are provided in Appendix \ref{app:dataset-details}.

\textbf{Baselines.} We compare our method against state-of-the-art offline MOO methods.
A common paradigm for offline MOO is to first learn surrogate models on the offline dataset and then use the learned surrogates to identify Pareto-optimal designs. 
Existing methods differ in both the learned surrogate models and the optimization algorithm used to search the design space.
Inverse methods optimize the learned surrogates using a generative model. \textbf{ParetoFlow} \citep{Yuan2025paretoflow} trains one surrogate model per objective and then uses these models to guide the sampling process of a pre-trained flow matching model. \textbf{Preference Guided Diffusion} \citep{Annadani2025pgd} instead trains a preference model that predicts the probability of whether a design dominates another design, and then guides the sampling process of a pre-trained diffusion model. 
Forward methods optimize the learned surrogates using an evolutionary algorithm such as NSGA-II \citep{Deb2002NSGA2}. Following \citet{Xue2024OFFMOO}, we consider two surrogate learning strategies for the forward methods. \textbf{Multi-Head} trains a single surrogate model using multi-task learning methods, including \textbf{GradNorm} \citep{Chen2018GradNorm} and \textbf{PCGrad} \citep{Yu2020PCGrad}. \textbf{Multiple-Models} trains one surrogate model per objective using offline single-objective optimization methods, including \textbf{COM} \citep{Trabucco2021COM}, \textbf{IOM} \citep{Qi2022IOM}, \textbf{ICT} \citep{Yuan2023ICT}, \textbf{RoMA} \citep{Yu2021RoMA}, and \textbf{TriMentoring} \citep{Chen2023TriMentoring}.

\textbf{Evaluation.}
We evaluate all methods using Hypervolume (HV), Generational Distance (GD), Inverted Generational Distance (IGD), and the Wasserstein distance $W_2$. HV measures the dominated objective-space volume and can be strongly influenced by individual solutions, while GD and IGD provide complementary set-based measures of convergence to and coverage of the Pareto front, respectively. 
Motivated by our sampling view of MOO in Section~\ref{sec:oflline_moo_as_sampling_problem}, we further report $W_2$ between the generated objective-space distribution and the uniform distribution on the Pareto front. %By Corollary~\ref{corollary:ot}, $W_2$ upper bounds the continuous versions of both GD and IGD and, 
Unlike the set-based metrics, $W_2$ additionally accounts for probability-mass allocation along the front. To our knowledge, this is the first offline MOO evaluation to use $W_2$ against a Pareto-front distribution. Further details on the metrics are provided in Appendix~\ref{app:moo-metrics}.
\subsection{Training details}
For the surrogate models, we follow the multiple-model configuration of~\citet{Xue2024OFFMOO} and train one surrogate per objective. Each surrogate is a three-layer MLP with two $2048$-dimensional hidden layers and ReLU activations, trained with MSE loss and Adam with learning rate $10^{-3}$ for $200$ epochs.
For the flow matching model, we use an architecture based on~\citet{Yuan2025paretoflow}, consisting of a four-layer MLP with three $256$-dimensional hidden layers and SeLU activations. The model is trained with Adam with learning rate $10^{-3}$ for up to $200$ epochs.

At inference, we integrate the learned velocity field using explicit Euler with $K=100$ steps, generate a population of size $N=256$, and use a proxy support of size $M=256$.
Following~\citet{Yuan2025paretoflow}, guidance is applied only for $t>t_{\mathrm{guided}}=0.8$; earlier steps are plain Euler updates.
At each guided step, we perform $10$ Adam steps with learning rate $0.01$ to minimize the directed loss functional in~\eqref{eq:directed-discrepancy}, using $\gamma=0.5$ and $\lambda = 10^{-4}$.
Additional training details are provided in Appendix~\ref{app:training-details}.

\subsection{Results}
All metrics are computed in the normalized objective space using the fixed min--max statistics of the offline dataset. We evaluate each method over $5$ random seeds using a population of $N=256$ candidate designs.
Table~\ref{tab:HV-representative} reports HV on a representative subset of tasks across all three benchmark families. 
ParetoTransport remains competitive in terms of HV with both forward and inverse methods.
Overall, we observe only limited differences in HV performance across methods.
This motivates additionally evaluating GD, IGD, and $W_2$ to assess convergence, coverage, and probability-mass allocation.

\input{tables/hv_representation}
\input{tables/rank_summary}

Table~\ref{tab:rank_summary} reports the average rank across tasks within each benchmark family for HV, GD, IGD, and $W_2$. 
ParetoTransport achieves the best average rank on all four metrics for both ZDT and DTLZ. On the RE tasks, it remains competitive, achieving the best average rank in GD and $W_2$, and the third-best rank in both HV and IGD. Most notably, ParetoTransport attains the best average $W_2$ rank across all three benchmark families, supporting our approach of coupling transport toward the Pareto front with explicit Wasserstein control of probability-mass allocation. It also achieves the best average GD rank across all three families, indicating consistently strong convergence.
Complete per-task results are provided in Appendix~\ref{app:additional-results}.

\subsection{Ablations}
We study the contributions of the matching and directional improvement terms in~\eqref{eq:directed-discrepancy}, as well as the role of $\gamma$. 

\textbf{Wasserstein matching and directional improvement.}
Table~\ref{tab:ablation-components} reports HV, GD, IGD, and $W_2$ for the offline front, denoted by D(best). We report the corresponding changes $\Delta$ relative to D(best) when using the matching term $W_2$, the directional improvement term $L_n$, or both. Results are averaged over $5$ random seeds per task and then across all tasks within each benchmark family.
The directional improvement term $L_n$ alone yields the largest reduction in GD on both benchmark families, consistent with its role in moving the candidate distribution toward the Pareto front. Combining $L_n$ with Wasserstein control yields the largest improvements in IGD, $W_2$, and HV on both families, at a small trade-off in GD relative to $L_n$ alone. Overall, the two terms are complementary: $L_n$ is particularly effective for convergence, while their combination provides the best performance across coverage and probability-mass allocation.
Complete per-task ablation results are provided in Appendix~\ref{app:additional-ablations}.

\input{tables/ablation_components}

\input{tables/dtlz_gamma_gd}
\textbf{Directional improvement weight $\gamma$.}
We further study the effect of $\gamma$ in \eqref{eq:directed-discrepancy}.
%To make $\gamma$ reflect the relative weight of transport and Wasserstein control, 
We rescale the $L_n$ gradient to match the norm of the $W_2$ gradient; details are provided in Appendix~\ref{app:additional-ablations}.
Table~\ref{tab:dtlz-gamma-gd} shows that increasing $\gamma$ initially improves GD, consistent with stronger directional improvement toward the Pareto front. 
For large $\gamma$, however, GD increases again, indicating that stronger directional transport does not continue to improve convergence.
%The corresponding IGD and $W_2$ results show a similar preference for intermediate values of $\gamma$. 
Complete results across $\gamma$, including GD, IGD, $W_2$, and HV, are provided in Appendix~\ref{app:additional-ablations}.

%\begin{figure}[!htbp]
%  \centering
%\includegraphics[width=0.8\linewidth{iclr2027/figures/bars_GD_DTLZ.pdf}
%\caption{Ablation of the transport strength $\gamma$ on %DTLZ tasks. GD is computed over a population of $N=256$ %candidate designs and averaged over $5$ random seeds.}
%  \label{fig:ablation-gd-dtlz}
%\end{figure}

\section{Conclusion}
In this work, we approached offline MOO by mass transport toward the Pareto front and introduced a training-free guidance method for pre-trained flow matching models that explicitly refines the geometry and mass allocation of the generated objective-space distribution.
More broadly, generative models inherently operate on distributions over designs, making it natural to treat the induced distribution itself as the object of optimization rather than solely generating individual candidate solutions independently.
Beyond distribution-level control, generative methods are also particularly promising for offline MOO problems where feasible candidate designs live on low-dimensional manifolds embedded in high-dimensional ambient spaces \citep{Song2019Manifold}.  Forward optimization methods such as gradient-based or evolutionary approaches typically rely on operations directly in the ambient space, which may fail to preserve this manifold structure and move candidate designs outside the data-supported region.

\textbf{Limitations.} We leave the many-objective setting ($m \geq 4$) for future work, where a single global CHIM normal may be insufficient to capture the higher-dimensional Pareto-front geometry.

\subsubsection*{Acknowledgments}
We thank Alexandru-Ciprian Z\u{a}voianu for helpful discussions during the early stages of this work.

The ELLIS Unit Linz, the LIT AI Lab, the Institute for Machine Learning, are supported by the Federal State Upper Austria. We thank the projects LCM-COMET K2 Center, FWF AIRI FG 9-N (10.55776/FG9), AI4GreenHeatingGrids (FFG-899943), ELISE (H2020-ICT-2019-3 ID: 951847), Stars4Waters (HORIZON-CL6-2021-CLIMATE-01-01), FWF Bilateral Artificial Intelligence ([10.55776/COE12]). We thank NXAI GmbH, Silicon Austria Labs (SAL), FILL Gesellschaft mbH, Google, ZF Friedrichshafen AG, Robert Bosch GmbH, Merck Healthcare KGaA, GLS (Univ.~Waterloo), Borealis AG, T\"{U}V Austria, TRUMPF and the NVIDIA Corporation.

\bibliography{iclr2027_conference}
\bibliographystyle{iclr2027_conference}

\appendix
\section{Appendix}

\subsection{Training-free guidance of flow-matching models}
\label{app:flow_matching}
Flow matching (FM) \citep{Lipman2023flowmatching, Albergo2023building, Liu2022flow} is an efficient and simulation-free framework for generative modeling. The goal is to learn a time-dependent velocity field $u: [0,1] \times \mathbb{R}^d \rightarrow \mathbb{R}^d$ that transforms a simple initial distribution $p_0$ into a complex terminal distribution $p_1$. The flow map $\phi: [0,1] \times \mathbb{R}^d \rightarrow \mathbb{R}^d$ is defined by $\phi_t(x_0):=x(t)$, where $x:[0,1]\rightarrow\mathbb{R}^d$ is the solution of the ordinary differential equation (ODE)
\begin{equation}
\label{app:eq:flow_ode}
\frac{d x(t)}{dt} = u_t\big(x(t)\big), \qquad x(0)=x_0.
\end{equation}
This flow map induces a probability path $p_t:=(\phi_t)_{\#}p_0$ that satisfies the continuity equation \mbox{$\partial_t p_t + \nabla\cdot(u_t p_t)=0$}.
FM learns a model $u^{\theta}_t$ by minimizing the regression loss w.r.t.~the ground truth velocity field $u_t$,
\[
\mathcal{L}_{\mathrm{FM}}(\theta)
:=
\mathbb{E}_{t \sim \mathcal{U}[0,1],\, X_t \sim p_t}\!
\left\|u^{\theta}_t(X_t)-u_t(X_t)\right\|^2.
\]
Since the velocity field $u_t$ is unknown, the FM loss is intractable.
FM therefore introduces a conditional probability path $p_{t|Z}$ with $Z=(X_0,X_1) \sim \pi$, where $\pi \in \Gamma(p_0,p_1)$ is a coupling of $X_0 \sim p_{0}$ and $X_{1} \sim p_{1}$. A standard choice is an affine conditional probability path $X_t:= \alpha_t X_1 + \beta_t X_0$ with scheduler $(\alpha_t, \beta_t)$ satisfying $\alpha_0 = \beta_1 = 0, \alpha_1=\beta_0=1$. The induced velocity field $u_{t\mid Z}$ is then given by $u_{t\mid Z}(X_t)=\dot{\alpha}_t X_1 + \dot{\beta}_t X_0$. 
This yields a tractable conditional FM loss
\[
\mathcal{L}_{\mathrm{CFM}}(\theta)
:=
\mathbb{E}_{t \sim \mathcal{U}[0,1],\, Z \sim \pi,\, X_t \sim p_{t\mid Z}}
\left\|
u_{t\mid Z}^{\theta}(X_t)-u_{t\mid Z}(X_t)
\right\|^2.
\]
The marginal and conditional velocity fields are related through
\[
u_t(x)
=
%\mathbb{E}_{Z\sim\pi,\,X_t\mid Z\sim p_{t\mid Z}}
%\left[
%u_{t\mid Z}(X_t\mid Z)
%\mid X_t=x
%\mathbb{E}_{Z \sim p_{Z\mid t}(\cdot \mid x)}
%\left[
%u_{t\mid Z}(x)
%\right].
\mathbb{E}
\left[
u_{t\mid Z}(X_t) \mid X_t =x
\right].
\]
Moreover, it has been shown that $\nabla \mathcal{L}_{\mathrm{FM}}(\theta)
= \nabla \mathcal{L}_{\mathrm{CFM}}(\theta)$ \citep{Lipman2023flowmatching}. Therefore, the velocity field $u^\theta_t$ can be trained by minimizing the conditional FM loss $\mathcal{L}_{\mathrm{CFM}}(\theta)$. At inference time, sampling is performed by drawing an initial sample $x_0\sim p_0$ and integrating the ODE in (\ref{app:eq:flow_ode}) with the learned velocity field.

\textbf{Training-free guidance.} A key challenge in training-free guidance is the approximation of the full-horizon optimization problem by a sequence of single-step subproblems. Recent work \citep{Li2026hardflow, Pourya2026flower} further introduces a reverse parameterization that reformulates the optimization in terms of the terminal sample, thereby avoiding differentiation of the terminal loss through the neural network.

The full-horizon optimization problem \citep{Benhamu2024dflow} is formulated as
\begin{align}
\begin{aligned}
\min_{\big(\widetilde{x}(t),\,c(t)\big)_{t\in [0,1]}}
\quad &
J\big(\widetilde{x}(1)\big)
+\lambda\int_{0}^{1}\frac{1}{2}\|c(t)\|^2\,dt
\\
&
\begin{aligned}[t]
\text{s.t.}\quad
&\widetilde{x}(0)=x_0,\\
&
\frac{d \widetilde x(t)}{dt} = u_t\big(\widetilde x(t)\big) + c(t) \qquad t\in[0,1].
\end{aligned}
\end{aligned}
\end{align}
In practice, the continuous-time dynamics are discretized on a time grid $0=t_1< \ldots<t_K=1$, $\Delta{t_k}:=t_{k+1}-t_k$. Using forward Euler, the full-horizon optimization problem becomes
\begin{align}
\label{eq:discrete-opt}
\begin{aligned}
\min_{(\widetilde{x}_{t_k})_{k=1}^K,\,(c_{t_k})_{k=1}^{K-1}}
\quad &
J\big(\widetilde{x}_1\big)
+\lambda \sum_{k=1}^{K-1}\frac{\Delta t_k}{2}\|c_{t_k}\|^2
\\
&
\begin{aligned}[t]
\text{s.t.}\quad
&\widetilde{x}_0=x_0,\\
&
\widetilde{x}_{t_{k+1}}
=
\widetilde{x}_{t_k}
+\Delta t_k
\left(
u_{t_k}(\widetilde{x}_{t_k})
+c_{t_k}
\right) \qquad k=1,\ldots,K-1.
\end{aligned}
\end{aligned}
\end{align}
Although \eqref{eq:discrete-opt} is written in terms of both the trajectory and guidance variables, the Euler recursion uniquely determines the trajectory from the guidance variables once the initial samples are fixed. Hence, the problem can equivalently be viewed as a full-horizon optimization over $(c_{t_k})_{k=1}^{K-1}$ alone.

Prior work \citep{Li2026hardflow, Webber2026flowmpc} approximates this discrete-time full-horizon problem by a sequence of single-step subproblems. 
At step $k$, future guidance velocities $c_{t_{k+1}},\ldots,c_{t_{K-1}}$ are ignored, and the posterior mean $\widehat{x}_{1\mid t_{k+1}}$ is used to approximate the terminal sample resulting from the current guidance step. Hence, rather than jointly optimizing the guidance velocities over the entire sampling path, each subproblem optimizes only $c_{t_k}$.
Given $x_0\sim p_0$ and $\widetilde{x}_0:=x_0$, the resulting recursion
for $k=1,\ldots,K-1$ is
\begin{align}
\begin{aligned}
\label{eq:opt_single_step}
c_{\mathrm{opt}}
&:=
\arg\min_{c_{t_k}}
\quad
J(\widehat{x}_{1})
+\lambda\frac{\Delta t_k}{2}\|c_{t_k}\|^2
\\ &
\phantom{\qquad J(\widehat{x}_{1}}
\text{s.t.}\quad 
 \widehat{x}_{1}
=
\widehat{x}_{1\mid t_{k+1}}
\big(
\widetilde{x}_{t_k}
+\Delta t_k
\left(
u_{t_k}(\widetilde{x}_{t_k})
+c_{t_k}
\right)
\big),
\\
\widetilde{x}_{t_{k+1}}
&:=
\widetilde{x}_{t_k}
+\Delta t_k
\left(
u_{t_k}(\widetilde{x}_{t_k})
+c_{\mathrm{opt}}
\right).
\end{aligned}
\end{align}

Since the terminal loss depends on $c_{t_k}$ only through the estimated terminal sample, \citet{Li2026hardflow} reformulate each subproblem directly in terminal space. The optimized terminal sample is then mapped back to the sampling path using the local reverse parameterization, denoted here by $\widehat{x}_{t_{k+1} \mid 1}$. The regularization term on the guidance velocity can be equivalently expressed in terms of the terminal sample \citep{Li2026hardflow}, resulting in the following recursion for $k=1,\ldots,K-1$:
\begin{align}
\label{app:step-opt}
\begin{aligned}
\widehat{x}_{1}
&:=
\widehat{x}_{1\mid t_{k+1}}\big(\widetilde{x}_{t_k}
+\Delta t_k
u_{t_k}(\widetilde{x}_{t_k})\big),
\\ x_{\mathrm{opt}}
&:=
\arg\min_{x_1}
\quad
J(x_1)
+\lambda\frac{\alpha_{t_{k+1}}^2}{2 \Delta t_k} \|x_1 -\widehat{x}_{1}\|^2,
\\
\widetilde{x}_{t_{k+1}}
&:=
\widehat{x}_{t_{k+1} \mid 1}(x_{\mathrm{opt}}).
\end{aligned}
\end{align}

A detailed analysis of the approximation error is provided by \citet{Li2026hardflow}.

\textbf{Posterior mean estimation.} Prior work in diffusion \citep{Kim2021noisescore, Chung2023DPS, Kim2025testtime} has leveraged Tweedie's formula to obtain the posterior mean estimator of the terminal sample. Under a Gaussian initial distribution and an affine conditional probability path, Tweedie's formula has also been adopted for flow matching \citep{Kim2025FlowDPS, Pourya2026flower}. However, flow matching also admits a direct derivation of the same posterior mean estimator that relies only on the affine conditional probability path, without explicit use of Tweedie's formula \citep{Feng2025GuidanceFM, Li2026hardflow, Webber2026flowmpc}.

\begin{lemma}\citep{Feng2025GuidanceFM}
\label{lemma:posterior_mean}
Let $Z:=(X_0,X_1)\sim\pi \in \Gamma(p_0,p_1)$ and consider the affine conditional probability path $X_t:=\alpha_tX_1+\beta_tX_0, t\in[0,1]$, where $(\alpha_t,\beta_t)$ is a differentiable scheduler satisfying $\alpha_0 = \beta_1 = 0, \alpha_1=\beta_0=1$. The induced conditional velocity field $u_{t\mid Z}$ is therefore
$u_{t\mid Z}(X_t) = \dot{\alpha}_tX_1+\dot{\beta}_tX_0$. Define $\Lambda_t:=\alpha_t\dot{\beta}_t-\dot{\alpha}_t\beta_t$.
Assume that $\Lambda_t\neq0$ for all $t\in[0,1]$. Then, for every $x\in\mathbb{R}^d$, the posterior means are given by
\[
\widehat{x}_{0\mid t}(x)
:=
\mathbb{E}[X_0\mid X_t=x]
=
\frac{-\dot{\alpha}_t\,x+\alpha_t\,u_t(x)}
{\Lambda_t}.
\]
and
\[
\widehat{x}_{1\mid t}(x)
:=
\mathbb{E}[X_1\mid X_t=x]
=
\frac{\dot{\beta}_t\,x-\beta_t\,u_t(x)}
{\Lambda_t}.
\]
Moreover, they satisfy the identity
\begin{align}
\label{app:identity}
    x
    =
    \alpha_t\widehat{x}_{1\mid t}(x)
    +
    \beta_t\widehat{x}_{0\mid t}(x).
\end{align}
\end{lemma}

\begin{corollary}
\label{cor:linear_scheduler}
Consider the linear scheduler $\alpha_t:=t,
\beta_t:=1-t$ corresponding to the affine conditional probability path $X_t:=tX_1+(1-t)X_0$.
Then the posterior means simplify to
\begin{align}
    \widehat{x}_{0 \mid t}(x)=x-t u_t(x),
    \qquad 
    \widehat{x}_{1\mid t}(x)= x+(1-t)u_t(x).
\end{align}
\end{corollary}

%Recent work \citep{Li2026hardflow, Pourya2026flower} leverages 
\begin{remark}
The identity in (\ref{app:identity}) yields an implicit characterization of the reverse parameterization. 
For $\alpha_t\neq0$, 
\[
x_1=\widehat{x}_{1\mid t}(x_t)
\quad\Longleftrightarrow\quad
x_t=\alpha_t x_1+\beta_t\widehat{x}_{0\mid t}(x_t).
\]
Therefore, recovering $x_t$ from a given terminal sample $x_1$ is equivalent to finding a fixed point of the mapping
\[
T_{x_1}(x)
:=
\alpha_t x_1
+
\beta_t\widehat{x}_{0\mid t}(x).
\]  
If $T_{x_1}$ is a contraction, the Banach fixed-point theorem guarantees that, for any initialization $x^{(0)}$, the iterates
\[
x^{(k+1)}=T_{x_1}(x^{(k)})
\]
converge to the unique fixed point.
\end{remark}

Given a reference intermediate sample $x_t^{\mathrm{ref}}$,
\citet{Li2026hardflow} approximate the solution of the fixed-point
equation by a single iteration initialized at
$x^{(0)}:=x_t^{\mathrm{ref}}$. Thus, for a given terminal sample $x_1$,
\[
\widehat{x}_{t\mid 1}
\bigl(x_1;x_t^{\mathrm{ref}}\bigr)
:=
x^{(1)}
=
T_{x_1}(x^{(0)})
=
\alpha_t x_1
+
\beta_t
\widehat{x}_{0\mid t}
\bigl(x_t^{\mathrm{ref}}\bigr).
\]
\cite{Pourya2026flower} instead define the reverse parameterization by
\[
\widehat{x}_{t\mid 1}(x_1):=\alpha_t x_1 + \beta_t x_0, 
\]
where $x_0 \sim p_0$ is independently sampled at each step.

\subsection{Algorithmic Details}
\label{app:algorithmic-details}

\subsubsection{Particle-level approximation of distributional guidance.}

We derive the particle-level approximation of the distributional
guidance problem in \eqref{eq:distributional-opt}.
We approximate the probability path $(\widetilde p_t)_{t\in[0,1]}$
by the empirical path and, with a slight abuse of notation, write $\widetilde p_t=\frac{1}{N}\sum_{i=1}^N\delta_{\widetilde x_t^i}$.
Thus, the probability path $(\widetilde p_t)_{t\in [0,1]}$ is represented by a system of particle trajectories $\{(\widetilde x_t^i)_{t\in [0,1]}\mid i=1,\ldots,N\}$.
The continuity equation in \eqref{eq:distributional-opt} then takes the form
\begin{equation}
\label{eq:particle-characteristics}
\frac{d}{dt}\widetilde x_t^i
=
u_t(\widetilde x_t^i)
+
c_t(\widetilde x_t^i),
\qquad
\widetilde x_0^i=x_0^i,
\qquad
i=1,\ldots,N.
\end{equation}
On the time grid $0=t_1<\ldots<t_K=1$, $\Delta t_k:=t_{k+1}-t_k$, the forward Euler
discretization of \eqref{eq:particle-characteristics} is
\begin{equation}
\label{eq:particle-euler}
\widetilde x_{t_{k+1}}^i
=
\widetilde x_{t_k}^i
+
\Delta t_k
\left(
u_{t_k}(\widetilde x_{t_k}^i)
+
c_{t_k}(\widetilde x_{t_k}^i)
\right).
\end{equation}
Thus, the particle-level time-discrete approximation of the guidance regularization term in \eqref{eq:distributional-opt} becomes
\begin{align*}
\lambda
\int_0^1
\int
\frac{1}{2}\|c_t(x)\|^2
\,d\widetilde p_t(x)\,dt
=
\frac{\lambda}{N}
\sum_{i=1}^N
\int_0^1
\frac{1}{2}
\|c_t(\widetilde x_t^i)\|^2\,dt
\approx
\frac{\lambda}{N}
\sum_{i=1}^N
\sum_{k=1}^{K-1}
\frac{\Delta t_k}{2}
\|c_{t_k}(\widetilde x_{t_k}^i)\|^2 .
%\label{eq:particle-kinetic}
\end{align*}

The distributional guidance problem in \eqref{eq:distributional-opt} is then approximated by
\begin{align}
\min_{\substack{
(\widetilde x_{t_k}^i)_{i,k},\,
(c_{t_k}^i)_{i,k}
}}
\quad &
J\left(
\frac{1}{N}
\sum_{i=1}^N
\delta_{\widetilde x_{1}^i}
\right)
+
\frac{\lambda}{N}
\sum_{i=1}^N
\sum_{k=1}^{K-1}
\frac{\Delta t_k}{2}
\|c_{t_k}^i\|^2
\label{eq:particle-discrete-opt}
\\[-1mm]
\text{s.t.}\quad &
\widetilde x_{0}^i=x_0^i,
\qquad i=1,\ldots,N,
\notag\\
&
\widetilde x_{t_{k+1}}^i
=
\widetilde x_{t_k}^i
+
\Delta t_k
\left(
u_{t_k}(\widetilde x_{t_k}^i)
+
c_{t_k}^i
\right),
\quad
i=1,\ldots,N,\quad
k=1,\ldots,K-1 .
\notag
\end{align}

The single-trajectory approximation described in Appendix~\ref{app:flow_matching} replaces the corresponding time-discrete full-horizon optimization problem by a sequence of single-step terminal-space subproblems.
We apply this approximation to each particle trajectory.
Thus, at step $k$, only the current guidance variables $(c_{t_k}^i)_{i=1}^N$ are considered.
For each particle, the terminal estimate is obtained by evaluating the posterior mean map at the unguided Euler update,
\begin{equation*}
\widehat{x}^{i}_{1}
:=
\widehat{x}_{1 \mid t_{k+1}}
\left(
\widetilde{x}^{i}_{t_k}
+
\Delta t_k
u_{t_k}(\widetilde{x}^{i}_{t_k})
\right),
\qquad
i=1,\ldots,N.
\end{equation*}
The particle-wise guidance variables $(c_{t_k}^i)_{i=1}^N$ are then reparameterized by terminal variables $(x_1^i)_{i=1}^N$.
The resulting subproblem is
\begin{equation*}
\begin{aligned}
\big(x^{i}_{\mathrm{opt}}\big)_{i=1}^N
\in
\arg\min_{(x^{i}_1)_{i=1}^N}
\;
J
\left(
\frac{1}{N}
\sum_{i=1}^{N}
\delta_{x^{i}_1}
\right)
+
\frac{\lambda}{N}
\sum_{i=1}^{N}
\frac{\alpha_{t_{k+1}}^2}{2\Delta t_k}
\left\|
x^{i}_1-\widehat{x}^{i}_{1}
\right\|^2 .
\end{aligned}
\end{equation*}
Unlike in the single-trajectory setting, the loss functional $J$ acts on the empirical distribution of the terminal particles and therefore couples the terminal variables in a joint optimization.  %, while the quadratic regularizer remains separable across particles.
Finally, the optimized terminal particles are mapped back to the sampling path using the local reverse parameterization,
\begin{equation*}
\widetilde{x}^{i}_{t_{k+1}}
:=
\widehat{x}_{t_{k+1} \mid 1}
\big(x^{i}_{\mathrm{opt}}\big),
\qquad
i=1,\ldots,N.
\end{equation*}

\subsubsection{Intermediate proxy distributions}
\textbf{Objective-space normalization.}
All objective-space computations in ParetoTransport are performed after min--max normalization using statistics computed once from the offline dataset $\mathcal D_{\mathrm{off}}$. 
Since the $W_2$ term uses Euclidean distances in objective space, normalization ensures that its transport cost is not dominated by differences in objective scales.
For objective $r=1,\ldots,m$, let
\[
f_r^{\min}
:=
\min_{x\in\mathcal D_{\mathrm{off}}} f_r(x),
\qquad
f_r^{\max}
:=
\max_{x\in\mathcal D_{\mathrm{off}}} f_r(x).
\]
We define the normalized objective functions
\begin{equation*}
\bar f_r(x)
:=
\frac{
f_r(x)-f_r^{\min}
}{
f_r^{\max}-f_r^{\min}
},
\qquad r=1,\ldots,m,
\end{equation*}
and keep these normalization constants fixed throughout surrogate training, flow model training, and sampling.
For simplicity, we write $f$ for the normalized objective function $\bar f :=(\bar f_1,\ldots,\bar f_m)$.

\textbf{Crowding distance.}
For a set $\{y^j\}_{j=1}^M\subset\mathbb R^m$, we compute the standard NSGA-II \citep{Deb2002NSGA2} crowding distance.
For each objective $r$, the points are sorted according to their
$r$-th objective value. For an interior point $y^j$, its contribution
along objective $r$ is
\begin{equation*}
d_{r}^{j}
:=
\frac{
y_{r}^{j^+}-y_{r}^{j^-}
}{
y_{r}^{\max}-y_{r}^{\min}
},
\end{equation*}
where $j^-$ and $j^+$ denote the adjacent points in the ordering of
objective $r$. The total crowding
distance is
\begin{equation*}
d^j
:=
\sum_{r=1}^m d_r^j.
\end{equation*}
As in NSGA-II, boundary points are assigned infinite crowding distance
during the sorting procedure.

\paragraph{Proxy-update operator.}
We denote by $\mathcal{R}$ the proxy-update operator that combines the current proxy with the optimized objective vectors to construct the next proxy distribution, updating its support and assigning probability masses according to local crowding so that sparsely represented regions receive more mass and densely represented regions receive less.

In particular, we combine the support of the current proxy with the new candidate objective vectors. Let $(y^j)_{j=1}^M$ denote the $M$ points retained after non-dominated filtering and crowding-based truncation.
We then compute the crowding distance $d^j$ for each point. Since boundary points are assigned infinite crowding distance, we replace each infinite value by the largest finite crowding distance in the archive,
\begin{equation*}
\widetilde d^j
:=
\begin{cases}
d^j,
& d^j<\infty,\\[1mm]
\displaystyle
\max_{\ell:\,d^\ell<\infty} d^\ell,
& d^j=\infty.
\end{cases}
\label{eq:finite-crowding}
\end{equation*}
Direct normalization of these scores can assign excessive mass to isolated points. We therefore impose the upper bound
\[
b^j
\leq
\frac{\kappa}{M},
\]
where $\kappa\geq1$ is a fixed cap parameter. The target marginal is
defined by
\begin{equation*}
b^j
:=
\min
\left\{
\frac{\kappa}{M},
\lambda w^j
\right\},
\qquad
j=1,\ldots,M,
\end{equation*}
where $\lambda>0$ is chosen such that
\[
\sum_{j=1}^M b^j=1.
\]
Thus, the uncapped masses remain proportional to the crowding distances,
while excess mass from capped points is redistributed proportionally
among the remaining points. We use $\kappa=5$.
The resulting next target is
\begin{equation*}
\mu_{t_{k+1}}
:=
\sum_{j=1}^M
b^j\delta_{y^j}.
\end{equation*}

\subsubsection{Loss functional}
\label{app:loss-functional}
We derive the implementation details of the loss functional in equation \eqref{eq:directed-discrepancy},
\begin{equation*}
J_k(q)
=
W_2^2(f_\#q,\mu_{t_k})
+ \gamma L_n(f_\#q\|\mu_{t_k}).
\end{equation*}
For candidate terminal particles $X:=(x_1^i)_{i=1}^N$, the terminal design-space distribution is represented by the empirical distribution $q:=\frac{1}{N}\sum_{i=1}^N \delta_{x_1^i}$.
The particle-level optimization acts on the particle locations $x_1^i$, while each particle retains mass $\frac{1}{N}$. 
Consequently, the induced objective-space distribution is $f_\# q=\frac{1}{N}\sum_{i=1}^N\delta_{f(x_1^i)}$.

\textbf{Directional improvement term.} We give details on the surrogate projection direction and the particle-level implementation of the transport term $L_n$,
\begin{equation*}
    L_n(f_\#q\|\mu_{t_k})
=
-
\left(
\mathbb E_{Y\sim f_\#q}[Y]
-
\mathbb E_{Y\sim\mu_{t_k}}[Y]
\right)^\top n. 
\end{equation*}
Following \citet{Das1998nbi}, we construct the Convex Hull of Individual Minima (CHIM) from the objective vectors corresponding to the individual objective minima. For each $r=1,\ldots,m$, let
\begin{equation*}
x_{\mathrm{IM}}^r
\in
\arg\min_{x\in\mathcal D_{\mathrm{off}}} f_r(x),
\qquad
z^r
:=
f(x_{\mathrm{IM}}^r)
\in\mathbb R^m.
\end{equation*}
The CHIM is $\mathcal H:=\operatorname{conv}\{z^1,\ldots,z^m\}$.
We use a unit normal $n$ to the affine hull of $\mathcal H$ as the
global surrogate direction in Eq.~(\ref{eq:directed-discrepancy}).
To compute the normal, for $m=2$ we set
\begin{equation*}
v:=z^2-z^1,
\qquad
\widetilde n:=(-v_2,v_1)^\top,
\end{equation*}
while for $m=3$ we use
\begin{equation*}
\widetilde n
:=
(z^2-z^1)\times(z^3-z^1).
\end{equation*}
Finally, we normalize the vector and resolve its sign ambiguity by
orienting it toward simultaneous objective decrease:
\begin{equation}
\bar n
:=
\frac{\widetilde n}{\|\widetilde n\|_2},
\qquad
n
:=
\begin{cases}
\bar n,
& \bar n^\top(-\mathbf 1_m)\geq 0,\\
-\bar n,
& \text{otherwise}.
\end{cases}
\label{eq:chim-normal}
\end{equation}
The resulting direction $n$ is computed once from the offline dataset
and kept fixed throughout guided sampling.

At sampling step $k$, let $(\widehat{x}_1^i)_{i=1}^N$ denote the posterior terminal estimates.
For any fixed reference distribution $\rho\in\mathcal P_2(\mathbb R^m)$, the particle-level transport term is
\[
L_n(f_\#q\|\rho)
=
-
\left(
\frac{1}{N}\sum_{i=1}^N f(x_1^i)
-
\mathbb E_{Y\sim\rho}[Y]
\right)^\top n.
\]
Since $\mathbb E_{Y\sim\rho}[Y]^\top n$ is constant with respect to the candidate particles, the gradient of $L_n$ is independent of the choice of the fixed reference distribution. In the implementation, we use the empirical distribution of the posterior terminal estimates $\rho:=\frac{1}{N}\sum_{i=1}^N\delta_{f(\widehat{x}_1^i)}$ as the reference for the transport term. 
%This gives
%\[
%L_n(f_\#q\|\rho)
%=
%-\frac{1}{N}\sum_{i=1}^N
%\left(
%f(x_1^i)-f(\widehat{x}_1^i)
%\right)^\top n.
%\]
The posterior estimates are held fixed during the terminal optimization, so this formulation induces the same gradient as $L_n(f_\#q\|\mu_{t_k})$.
We center the loss at the posterior estimates so that it directly represents the displacement of the candidate objective vectors along the surrogate projection direction.
The proxy $\mu_{t_k}$ remains the reference for the Wasserstein term, where it controls probability-mass allocation.

\textbf{Wasserstein term.} We approximate the discrete optimal transport problem defining $W_2^2(f_\#q,\mu_{t_k})$ by an entropically regularized optimal transport problem. 
The source marginal of the objective-space distribution $f_\# q$ in the optimal transport problem is $a:=\frac{1}{N}\mathbf 1_N$.
In contrast, the target marginal $b_{t_k}:=(b_{t_k}^j)_{j=1}^M$ is determined by the crowding-based
weighting rule. For simplicity, at a fixed step $k$, we write $y:=y_{t_k}$ and $b:=b_{t_k}$.

We define the quadratic cost matrix $C(X)\in\mathbb R^{N\times M}$ by $C_{ij}(X):=\|f(x_1^i)-y^j\|^2$.
The set of admissible couplings is $\mathcal U(a,b)
:=
\left\{
P\in\mathbb R_+^{N\times M}
:
P\mathbf 1_M=a,\;
P^\top\mathbf 1_N=b
\right\}$. The entropically regularized optimal transport problem is
\[
\mathcal W_\varepsilon(X)
:=
\min_{P\in\mathcal U(a,b)}
F_\varepsilon(X,P), 
\quad 
F_\varepsilon(X,P)
:=
\langle P,C(X)\rangle
+
\varepsilon
\sum_{i=1}^N\sum_{j=1}^M
P_{ij}(\log P_{ij}-1).
\]

%During optimization, the source marginal $a$ and the target marginal $b$ are fixed, so $\mathcal U(a,b)$ is independent of $X$.
For $\varepsilon>0$, $F_\varepsilon(X,\cdot)$ is strictly convex and therefore admits a unique minimizer $P_\varepsilon^\star(X)$ \citep{Peyre2019ComputationalOT}. Hence, Danskin's theorem yields
\begin{equation*}
\nabla_X\mathcal W_\varepsilon(X)
=
\nabla_X
F_\varepsilon\bigl(X,P_\varepsilon^\star(X)\bigr),
\end{equation*}
where $P_\varepsilon^\star(X)$ is held fixed in the derivative on the right-hand side. 
Thus, gradients with respect to the terminal particles can be computed without differentiating through the optimal coupling.

In practice, we approximate the coupling using the log-domain Sinkhorn solver implemented in the POT library \citep{POT}, using $\varepsilon=10^{-3}$.
Given the resulting coupling $P$, we optimize the weighted quadratic transport cost
\[
\langle P,C(X)\rangle 
=
\sum_{i=1}^N\sum_{j=1}^M
P_{ij}\left\|f(x_1^i)-y^j\right\|^2
\]
with respect to the terminal particles $X$.

\subsection{Implementation details}
\label{app:implementation-details}
We provide details on the implementation of ParetoTransport and the datasets used in our experiments.
\subsubsection{Training details}
\label{app:training-details}
For the surrogate models, we adopt the training configuration from~\citet{Xue2024OFFMOO}. We use the multiple model setting and train an individual surrogate for each objective. Each surrogate is a three-layer MLP with two $2048$-dimensional hidden layers and ReLU activations. The models are trained on the offline dataset using mean squared error (MSE) loss and the Adam optimizer with initial learning rate $1\times 10^{-3}$ and learning rate decay $0.98$ per epoch. We train for $200$ epochs with a batch size of $128$. For the flow matching model, we use an architecture based on \citet{Yuan2025paretoflow}: a four-layer MLP with three hidden layers and SeLU activations. We use a hidden width of $256$ and train for up to $200$ epochs with early stopping (patience $20$), using the Adam optimizer with learning rate of $1\times10^{-3}$ and a batch size of $512$.

We min-max normalize the design space to $[0,1]^d$ and set the initial distribution to be uniform over this cube. 
%We use independent-coupling conditional flow matching. 
We integrate the learned velocity field along the affine probability path, using an explicit Euler scheme on a uniform grid of $T=100$ points. We generate $N=256$ solutions in parallel. As in~\citet{Yuan2025paretoflow}, guidance is confined to the final phase of the trajectory, $t>0.8$; earlier steps are plain Euler updates. At each guidance step, we run $10$ Adam steps with learning rate $0.01$ to minimize the loss functional $J$ in \eqref{eq:directed-discrepancy}, using $\gamma=0.5$.

\subsubsection{MOO Metrics}
\label{app:moo-metrics}
We provide details on the evaluation metrics.

Hypervolume (HV)~\citep{Zitzler1999HV} measures $m$-dimensional volume of the objective-space region dominated by the solution set $A \subset \mathbb{R}^{m}$ and bounded by a reference point $r \in \mathbb{R}^{m}$. The reference point is selected such that $a_i\leq r_i$ for all $a\in A$ and $i=1,\ldots,m$. The hypervolume is defined as
\begin{align}
    \mathrm{HV}(A, r) := \lambda \Bigg( \bigcup_{a \in A} [a_1, r_1] \times [a_2, r_2] \times \dots \times [a_m, r_m] \Bigg),
\end{align}
where $\lambda$ denotes the Lebesgue measure in $\mathbb{R}^{m}$.

Generational Distance (GD)~\citep{Lamont1999GD, Schutze2012} measures the root mean squared distance from points in the solution set $A$ to their closest points in a discrete approximation $Z$ of the Pareto front,
\begin{align}
    \mathrm{GD}(A,Z) := \sqrt{\frac{1}{|A|} \sum_{a \in A} d(a,Z)^2}.
\end{align}

Analogously, the Inverted Generational Distance (IGD)~\citep{Coello2004IGD, Schutze2012} measures the root mean squared distance from points in a discrete approximation $Z$ of the Pareto front to their closest points in the solution set $A$, $\mathrm{IGD}(A,Z) := \mathrm{GD}(Z,A)$.

HV can be strongly influenced by individual solutions, whereas GD and IGD aggregate distances across the solution set and Pareto-front approximation, respectively, providing complementary measures of convergence and coverage.
\FloatBarrier
\subsubsection{Dataset details}
\label{app:dataset-details}
We provide further details on the datasets\footnote{We focus on tasks with $m\leq3$, and leave the extension to the many-objective setting ($m \geq 4$) for future work. This choice is motivated by the complex structure of high-dimensional Pareto fronts.} in Tables~\ref{tab:ref-synthetic} and~\ref{tab:ref-re}. We use the original reference points provided in the code repository of \citet{Xue2024OFFMOO}.
\input{tables/synthetic_datasets}
\input{tables/re_datasets}
\FloatBarrier
\subsection{Theoretical results}
\label{app:proof}

\textbf{Corollary~\ref{corollary:ot}.} \textit{For $\mu,\nu\in\mathcal P_2(\mathbb R^m)$,
$\mathrm{GD}(\mu\|\nu)\leq W_2(\mu,\nu)$ and
$\mathrm{IGD}(\mu\|\nu)\leq W_2(\mu,\nu)$.}
\begin{proof}
Let $\pi\in\Gamma(\mu,\nu)$. Since
$y\in\operatorname{supp}\nu$ for $\pi$-almost every $(x,y)$, it follows that $d(x,\operatorname{supp}\nu)\leq \|x-y\|$ for $\pi$-a.e. $(x,y)$.
Therefore,
\[
\operatorname{GD}^2(\mu\Vert\nu)
\leq
\int_{\mathbb{R}^m\times\mathbb{R}^m}
    \|x-y\|^2\,d\pi(x,y).
\]
Taking the infimum over $\pi\in\Gamma(\mu,\nu)$ yields $\operatorname{GD}(\mu\Vert\nu)\leq W_2(\mu,\nu)$.
Since $\operatorname{IGD}(\mu\Vert\nu)=\operatorname{GD}(\nu\Vert\mu)$
and $W_2$ is symmetric, the IGD bound follows.
\end{proof}

\begin{remark}[One-sided optimal transport interpretation]
GD and IGD admit a one-sided optimal-transport interpretation in which one marginal is relaxed subject to a hard support constraint. In contrast, $W_2(\mu,\nu)$ fixes both marginals and therefore also accounts for probability-mass allocation. Indeed, as a direct consequence of Lemma~31 of \citet{Vayer2023distOT}, attributed therein to \citet{Canas2012distOT}, we have
\[
    \operatorname{GD}^2(\mu\Vert\nu)
    =
    \min_{\rho\in\mathcal P_2(\operatorname{supp}\nu)}
    W_2^2(\mu,\rho), 
    \qquad
    \operatorname{IGD}^2(\mu\Vert\nu)
    =
    \min_{\rho\in\mathcal P_2(\operatorname{supp}\mu)}
    W_2^2(\nu,\rho).
\]
Thus, Corollary~\ref{corollary:ot} follows also by evaluating the
respective minima at $\rho=\nu$ and $\rho=\mu$.
\end{remark}

\begin{assumption}[Projection condition]
\label{ass:projection}
    The Pareto front $M^\ast\subset\mathbb{R}^m$ is nonempty and closed. There exists a set $U\supset M^\ast$ such that every $y\in U$ has a unique nearest point in $M^\ast$, denoted by $\Pi_{M^\ast}(y):=\operatorname*{argmin}_{z\in M^\ast}\|y-z\|$.
\end{assumption}

\begin{assumption}[Moment and concentration condition]
\label{ass:moment}
The offline front and Pareto front distributions satisfy $\mu_0,\mu^\ast\in\mathcal{P}_2(\mathbb{R}^m)$ and $\mu_0(U)=1$.
\end{assumption}

\begin{assumption}[Invariance under transport]
\label{ass:invariance}
For a fixed $\eta\in(0,1]$, define $T_\eta(y):=(1-\eta)y+\eta\Pi_{M^\ast}(y)$, $y \in U$. $T_\eta$ satisfies $T_\eta(U)\subseteq U$.
\end{assumption}

\textbf{Theorem~\ref{thm:wasserstein-contraction}.}
{\itshape
Let $(\mu_k)_{k\geq 0}$ denote the sequence of intermediate proxy distributions defined by the iterative procedure in Section~\ref{sec:mass_transport}.  If Step (i) achieves exact matching, $W_2(f_\#q_{k+1},\mu_k)=0$ for all $k$, and the technical Assumptions (1)--(3) hold, then
\begin{equation}
\label{eq:app:w2-bound}
W_2(\mu_k,\mu^\ast)
\leq
(1-\gamma)^k
W_2(\mu_0,\bar\mu)
+
W_2(\bar\mu,\mu^\ast).
\end{equation}
}

\begin{proof}
Since $W_2$ is a metric on $\mathcal{P}_2(\mathbb{R}^m)$, exact matching
in Step~(i) implies $f_\#q_{k+1}=\mu_k$, and hence by Step~(ii), $\mu_{k+1} = (T_\eta)_\#\mu_k$.
In particular, $\mu_k=(T_\eta^k)_\#\mu_0$.

\textbf{Step 1:} We first establish $T_\eta^k=(1-\xi_k) \operatorname{Id}+\xi_k\Pi_{M^\ast}$ for $\xi_k:=1-(1-\eta)^k$.
By Assumptions~\ref{ass:projection} and~\ref{ass:invariance}, 
\begin{equation}
\label{eq:projection-invariance}
    \Pi_{M^\ast}(T_\eta(y))=\Pi_{M^\ast}(y), \quad y \in U.
\end{equation}
Indeed, fix $y\in U$, let $z:=\Pi_{M^\ast}(y)$ and
$r:=\|y-z\|=d(y,M^\ast)$, and set
$w:=\Pi_{M^\ast}(T_\eta(y))$. The latter is well defined by
Assumption~\ref{ass:invariance}. Since $w\in M^\ast$, $r\leq\|y-w\|$. 
Since $w$ is a nearest point to $T_\eta(y)$ and $z\in M^\ast$,
\[
\begin{aligned}
    \|y-w\|
    \leq
    \|y-T_\eta(y)\|
    +
    \|T_\eta(y)-w\|
    \leq
    \eta r+(1-\eta)r
    =
    r.
\end{aligned}
\]
Thus $\|y-w\|=r$, so $w$ is also a nearest point in $M^\ast$ to $y$.
By uniqueness in Assumption~\ref{ass:projection}, $w=z$, proving
\eqref{eq:projection-invariance}.
By Assumption~\ref{ass:invariance}, the identity
\eqref{eq:projection-invariance} may be iterated. A direct induction therefore gives
\begin{equation}
    T_\eta^k(y)
    =
    (1-\xi_k)y+\xi_k\Pi_{M^\ast}(y),
    \qquad y\in U.
    \label{eq:explicit-iterate}
\end{equation}

\textbf{Step 2:} Assumptions~\ref{ass:moment} and~\ref{ass:invariance}, together with $\mu_k=(T_\eta^k)_\#\mu_0$, imply $\mu_k(U)=1$ for all $k$.
Moreover, all iterates $\mu_k$ belong to $\mathcal{P}_2(\mathbb{R}^m)$.
Indeed, fixing any $z_0\in M^\ast$, for $y\in U$,
\[
    \|T_\eta(y)\|
    \leq
    \|y\|+ \|y - T_\eta(y)\|
    =
    \|y\|+ \eta \|y - \Pi_{\mathcal{M}^\ast}(y)\|
    \leq
    (1+\eta)\|y\|+\eta\|z_0\|,
\]
so finite second moments are preserved by $(T_\eta)_\#$.

\textbf{Step 3:} We now relate the iteration to a Wasserstein geodesic.
Since $M^\ast$ is nonempty and closed by Assumption~\ref{ass:projection}, there exists a Borel measurable
nearest-point selection $P_{M^\ast}:\mathbb{R}^m\to M^\ast$.
On $U$ the nearest point is unique, hence $P_{M^\ast}(y)=\Pi_{M^\ast}(y)$, $y \in U$.
Since $\mu_0(U)=1$ by Assumption~\ref{ass:moment}, $(P_{M^\ast})_\#\mu_0=(\Pi_{M^\ast})_\#\mu_0=\bar{\mu}$.

By Lemma~31 of \citet{Vayer2023distOT}, it holds that
\begin{equation}
    W_2^2(\mu_0,\bar{\mu})
    =
    \int_{\mathbb{R}^m}
        \|y-P_{M^\ast}(y)\|^2\,d\mu_0(y),
    \label{eq:projection-cost}
\end{equation}
and the coupling $\pi:=(\operatorname{Id},P_{M^\ast})_\#\mu_0\in\Gamma(\mu_0,\bar{\mu})$ is an optimal transport plan between $\mu_0$ and $\bar{\mu}$.
Since $\mu_0\in\mathcal{P}_2(\mathbb{R}^m)$ by Assumption~\ref{ass:moment}, and $\bar{\mu}\in\mathcal{P}_2(\mathbb{R}^m)$, we may apply the displacement-interpolation theorem $5.27$ of \citet{Santambrogio2015optimal} to $\pi$. Thus
\begin{equation}
    \mu^{(s)}
    :=
    \bigl((1-s)\operatorname{pr}_1+s\operatorname{pr}_2\bigr)_\#
    \pi,
    \qquad s\in[0,1],
    \label{eq:displacement-interpolation}
\end{equation}
is a constant-speed $W_2$-geodesic from $\mu_0$ to $\bar{\mu}$, that is,
\begin{equation}
    W_2(\mu^{(s)},\bar{\mu})
    =
    (1-s)W_2(\mu_0,\bar{\mu}),
    \qquad s\in[0,1].
    \label{eq:constant-speed}
\end{equation}

Using the Monge form of $\pi$ together with $(P_{M^\ast})_\#\mu_0=(\Pi_{M^\ast})_\#\mu_0$ and $\mu_0(U)=1$, we may equivalently write
\begin{equation}
    \mu^{(s)}
    =
    \bigl((1-s)\operatorname{Id}
          +s\Pi_{M^\ast}\bigr)_\#\mu_0 .
    \label{eq:geodesic-projection}
\end{equation}
Combining $\mu_k=(T_\eta^k)_\#\mu_0$,
\eqref{eq:explicit-iterate}, and
\eqref{eq:geodesic-projection} yields $\mu_k=\mu^{(\xi_k)}$.
Therefore, by \eqref{eq:constant-speed},
\begin{equation*}
    W_2(\mu_k,\bar{\mu})
    =
    W_2(\mu^{(\xi_k)},\bar{\mu}) 
    =
    (1-\xi_k)W_2(\mu_0,\bar{\mu}) 
    =
    (1-\eta)^k W_2(\mu_0,\bar{\mu}).
\end{equation*}
%Since $\gamma\in(0,1]$ by Assumption~\ref{ass:invariance}, $W_2(\mu_k,\bar{\mu})\longrightarrow0$ for $k\longrightarrow \infty$.
Finally, $\mu^\ast\in\mathcal{P}_2(\mathbb{R}^m)$ by
Assumption~\ref{ass:moment}, so the triangle inequality for $W_2$ gives
\[
\begin{aligned}
    W_2(\mu_k,\mu^\ast)
    \leq
    W_2(\mu_k,\bar{\mu})
    +
    W_2(\bar{\mu},\mu^\ast) 
    =
    (1-\eta)^kW_2(\mu_0,\bar{\mu})
    +
    W_2(\bar{\mu},\mu^\ast),
\end{aligned}
\]
which proves the claim.
\end{proof}

\FloatBarrier
\subsection{Additional Results}
\label{app:additional-results}

\input{tables/zdt_HV}
\input{tables/zdt_GD}
\input{tables/zdt_IGD}
\input{tables/zdt_W2}

\input{tables/dtlz_HV}
\input{tables/dtlz_GD}
\input{tables/dtlz_IGD}
\input{tables/dtlz_W2}

\input{tables/re_HV}
\input{tables/re_GD}
\input{tables/re_IGD}
\input{tables/re_W2}

\FloatBarrier

\subsection{Additional Ablations}
\label{app:additional-ablations}
\subsubsection{Wasserstein and directional improvement components}
\input{tables/ablations_components_zdt}
\input{tables/ablations_components_dtlz}
\FloatBarrier

\subsubsection{Transport Strength $\gamma$}
\normalfont
The Wasserstein and directional improvement terms can induce gradients of different magnitudes with respect to the terminal particles. 
To make $\gamma$ control their relative contribution rather than compensate for differences in their raw gradient scales, we rescale the $L_n$ gradient relative to the $W_2$ gradient.
Following the notation in Appendix~\ref{app:loss-functional}, let
\[
g_W
:=
\nabla_X W_2^2(f_\#q,\mu_{t_k}),
\qquad
g_n
:=
\nabla_X L_n(f_\#q\|\mu_{t_k}).
\]
We define
\[
\widetilde g_n
:=
\frac{\|g_W\|_2}{\|g_n\|_2}g_n,
\]
such that $\|\widetilde g_n\|_2=\|g_W\|_2$. The combined guidance gradient is then $g=g_W+\gamma \widetilde g_n$. Thus, $\gamma$ directly controls the relative strength of the transport contribution.

We study the effect of the transport strength by varying
\(
\gamma
\in
\{0,0.1,0.2,0.5,1,2,5,10\},
\)
and report GD, IGD, $W_2$, and HV for all ZDT and DTLZ tasks in Tables~\ref{tab:zdt-gamma-hv}--\ref{tab:dtlz-gamma-w2}.

\input{tables/ablation_gamma/zdt_gamma_hv}
\input{tables/ablation_gamma/zdt_gamma_gd}
\input{tables/ablation_gamma/zdt_gamma_igd}
\input{tables/ablation_gamma/zdt_gamma_w2}
\input{tables/ablation_gamma/dtlz_gamma_hv}
%\input{tables/ablation_gamma/dtlz_gamma_gd}
\input{tables/ablation_gamma/dtlz_gamma_igd}
\input{tables/ablation_gamma/dtlz_gamma_w2}

\end{document}

%% file: math_commands.tex
\usepackage{amsmath,amsfonts,bm}

\def\eqref#1{equation~\ref{#1}}
\def\1{\bm{1}}

\DeclareMathAlphabet{\mathsfit}{\encodingdefault}{\sfdefault}{m}{sl}
\SetMathAlphabet{\mathsfit}{bold}{\encodingdefault}{\sfdefault}{bx}{n}

%% file: tables/hv_representation.tex
\begin{table}[!htbp]
\centering
\caption{HV ($\uparrow$) results; see Appendix~\ref{app:additional-results} for further results. Bold $=$ best per column. Each method is run for $5$ seeds and evaluated on $256$ designs.}
\resizebox{\linewidth}{!}{%
\begin{tabular}{lccccccc}
\toprule
Method & ZDT2 & ZDT6 & DTLZ6 & DTLZ7 & RE22 & RE23 & RE24 \\
\midrule
D(best) & 4.677 & 4.608 & 10.671 & 8.992 & 4.806 & 4.751 & 4.595 \\
MultiHead-Vallina & 5.612 $\pm$ 0.098 & 4.774 $\pm$ 0.003 & 10.900 $\pm$ 0.019 & 10.715 $\pm$ 0.027 & \textbf{4.840 $\pm$ 0.000} & \textbf{4.840 $\pm$ 0.000} & 4.075 $\pm$ 0.815 \\
MultiHead-PcGrad & 5.605 $\pm$ 0.047 & 4.776 $\pm$ 0.013 & 10.918 $\pm$ 0.011 & 10.598 $\pm$ 0.028 & 4.828 $\pm$ 0.023 & 4.763 $\pm$ 0.170 & 4.544 $\pm$ 0.231 \\
MultiHead-GradNorm & 5.483 $\pm$ 0.115 & 4.709 $\pm$ 0.025 & 10.848 $\pm$ 0.060 & 9.654 $\pm$ 0.735 & 4.721 $\pm$ 0.218 & 2.714 $\pm$ 0.067 & 4.333 $\pm$ 0.094 \\
MultipleModels-Vallina & 5.600 $\pm$ 0.043 & 4.774 $\pm$ 0.004 & 10.928 $\pm$ 0.008 & 10.535 $\pm$ 0.079 & \textbf{4.840 $\pm$ 0.000} & 4.831 $\pm$ 0.018 & 4.831 $\pm$ 0.010 \\
MultipleModels-COM & 5.093 $\pm$ 0.131 & 3.435 $\pm$ 0.750 & 10.786 $\pm$ 0.017 & 9.809 $\pm$ 0.273 & 4.809 $\pm$ 0.012 & 4.817 $\pm$ 0.033 & 4.658 $\pm$ 0.071 \\
MultipleModels-IOM & 5.530 $\pm$ 0.101 & 4.777 $\pm$ 0.005 & 10.886 $\pm$ 0.008 & 10.544 $\pm$ 0.126 & \textbf{4.840 $\pm$ 0.000} & 4.813 $\pm$ 0.034 & 4.789 $\pm$ 0.104 \\
MultipleModels-ICT & 5.428 $\pm$ 0.004 & 4.723 $\pm$ 0.051 & 10.781 $\pm$ 0.022 & 10.342 $\pm$ 0.092 & \textbf{4.840 $\pm$ 0.000} & 4.455 $\pm$ 0.071 & 4.830 $\pm$ 0.011 \\
MultipleModels-RoMA & 5.496 $\pm$ 0.067 & 2.597 $\pm$ 0.027 & 10.742 $\pm$ 0.015 & 10.596 $\pm$ 0.036 & 4.543 $\pm$ 0.664 & 4.836 $\pm$ 0.006 & 3.954 $\pm$ 1.200 \\
MultipleModels-Trimentoring & 5.548 $\pm$ 0.041 & 4.045 $\pm$ 0.122 & 10.778 $\pm$ 0.049 & 10.392 $\pm$ 0.423 & \textbf{4.840 $\pm$ 0.000} & \textbf{4.840 $\pm$ 0.000} & 4.530 $\pm$ 0.406 \\
\midrule
ParetoFlow & 4.695 $\pm$ 0.071 & 4.683 $\pm$ 0.068 & 10.686 $\pm$ 0.039 & 8.875 $\pm$ 0.099 & 4.815 $\pm$ 0.028 & 4.800 $\pm$ 0.022 & 4.832 $\pm$ 0.005 \\
PreferenceGuidedDiffusion & 5.367 $\pm$ 0.064 & 4.823 $\pm$ 0.004 & 10.336 $\pm$ 0.204 & 9.767 $\pm$ 0.141 & 4.836 $\pm$ 0.003 & 4.837 $\pm$ 0.003 & \textbf{4.834 $\pm$ 0.001} \\
\midrule
ParetoTransport & \textbf{5.660 $\pm$ 0.026} & \textbf{4.825 $\pm$ 0.001} & \textbf{11.028 $\pm$ 0.019} & \textbf{10.728 $\pm$ 0.076} & 4.839 $\pm$ 0.002 & \textbf{4.840 $\pm$ 0.000} & 4.831 $\pm$ 0.002 \\
\bottomrule
\end{tabular}
}
\label{tab:HV-representative}
\end{table}

%% file: tables/rank_summary.tex
\begin{table}[!htbp]
\centering
\caption{Average rank per metric across benchmark families (lower is better). Bold $=$ best per column. Each method is run for $5$ seeds and evaluated on $256$ designs.}
\resizebox{\linewidth}{!}{%
\begin{tabular}{lcccccccccccc}
\toprule
 & \multicolumn{4}{c}{ZDT} & \multicolumn{4}{c}{DTLZ} & \multicolumn{4}{c}{RE} \\
\cmidrule(lr){2-5} \cmidrule(lr){6-9} \cmidrule(lr){10-13}
 & HV & GD & IGD & W2 & HV & GD & IGD & W2 & HV & GD & IGD & W2 \\
\midrule
MultiHead-Vallina & 4.200 & 4.200 & 5.200 & 4.400 & 3.571 & 6.000 & 5.286 & 5.571 & 4.750 & 6.417 & 5.917 & 5.833 \\
MultiHead-PcGrad & 6.000 & 5.000 & 5.400 & 4.600 & 2.857 & 4.143 & 4.286 & 4.143 & 7.000 & 7.167 & 7.167 & 6.917 \\
MultiHead-GradNorm & 8.000 & 7.200 & 8.400 & 7.000 & 9.000 & 5.429 & 7.714 & 5.857 & 11.500 & 7.833 & 10.333 & 9.000 \\
MultipleModels-Vallina & 5.600 & 4.000 & 5.800 & 6.400 & 4.429 & 4.286 & 5.429 & 4.000 & \textbf{3.417} & 5.500 & 5.083 & 5.750 \\
MultipleModels-COM & 8.800 & 7.800 & 8.200 & 7.600 & 7.714 & 8.571 & 7.000 & 7.714 & 8.750 & 5.333 & 7.167 & 5.750 \\
MultipleModels-IOM & 4.000 & 5.400 & 4.200 & 4.200 & 5.143 & 5.714 & 5.429 & 5.286 & 3.500 & 6.333 & \textbf{3.167} & 5.083 \\
MultipleModels-ICT & 7.400 & 6.800 & 7.800 & 7.600 & 7.286 & 9.286 & 7.714 & 9.286 & 4.833 & 5.250 & 6.083 & 5.750 \\
MultipleModels-RoMA & 7.000 & 7.600 & 7.600 & 7.600 & 8.714 & 10.857 & 9.714 & 10.000 & 7.750 & 9.167 & 7.583 & 8.667 \\
MultipleModels-Trimentoring & 8.000 & 9.400 & 7.800 & 9.200 & 9.000 & 6.714 & 7.143 & 6.571 & 5.750 & 6.167 & 7.250 & 6.083 \\
\midrule
ParetoFlow & 9.600 & 8.000 & 9.600 & 9.400 & 9.857 & 8.429 & 8.571 & 8.714 & 9.250 & 6.083 & 6.917 & 6.500 \\
PreferenceGuidedDiffusion & 6.800 & 9.000 & 6.400 & 7.800 & 8.571 & 7.286 & 6.857 & 8.571 & 7.167 & 8.333 & 5.917 & 7.909 \\
\midrule
ParetoTransport & \textbf{2.600} & \textbf{3.600} & \textbf{1.600} & \textbf{2.200} & \textbf{1.714} & \textbf{2.000} & \textbf{2.714} & \textbf{1.857} & 4.333 & \textbf{4.417} & 5.417 & \textbf{4.417} \\
\bottomrule
\end{tabular}
}
\label{tab:rank_summary}
\end{table}

%% file: tables/ablation_components.tex
\begin{table}[!htbp]
\centering
\small
\caption{Ablation of Wasserstein matching $W_2$ and directional improvement $L_n$. Results are averaged over all tasks within each benchmark family and reported as changes $\Delta$ relative to the offline front D(best); bold $=$ best per column.}
\label{tab:ablation-components}
\resizebox{0.8\linewidth}{!}{%
\begin{tabular}{ccccccccccc}
\toprule
 &  &  & \multicolumn{4}{c}{ZDT} & \multicolumn{4}{c}{DTLZ} \\
\cmidrule(lr){4-7} \cmidrule(lr){8-11}
 & $W_2$ & $L_n$ & GD ($\downarrow$) & IGD ($\downarrow$) & $W_2$ ($\downarrow$) & HV ($\uparrow$) & GD ($\downarrow$) & IGD ($\downarrow$) & $W_2$ ($\downarrow$) & HV ($\uparrow$) \\
\midrule
D(best) &  &  & 0.262 & 0.261 & 0.374 & 4.812 & 0.373 & 0.247 & 0.427 & 11.542 \\
\midrule
 \multirow{3}{*}{$\Delta$} &  & \checkmark & \textbf{-0.124} & -0.059 & +0.029 & +0.419 & \textbf{-0.151} & -0.123 & -0.101 & +0.285 \\
& \checkmark &  & -0.054 & -0.087 & -0.063 & +0.413 & -0.065 & -0.120 & -0.065 & +0.286 \\
& \checkmark & \checkmark & -0.118 & \textbf{-0.165} & \textbf{-0.113} & \textbf{+0.464} & -0.139 & \textbf{-0.138} & \textbf{-0.131} & \textbf{+0.326} \\
\bottomrule
\end{tabular}
}
\end{table}

%% file: tables/dtlz_gamma_gd.tex
% requires \usepackage{booktabs}
\begin{table}[!htbp]
\caption{Ablation of $\gamma$ on DTLZ family. GD is computed over a population of $256$ candidate designs and averaged over $5$ random seeds.}
\label{tab:dtlz-gamma-gd}
\resizebox{\linewidth}{!}{%
\begin{tabular}{lccccccc}
\toprule
$\gamma$ & DTLZ1 & DTLZ2 & DTLZ3 & DTLZ4 & DTLZ5 & DTLZ6 & DTLZ7 \\
\midrule
0 & 0.419 $\pm$ 0.005 & 0.044 $\pm$ 0.003 & 0.533 $\pm$ 0.008 & 0.124 $\pm$ 0.014 & 0.081 $\pm$ 0.013 & 0.703 $\pm$ 0.007 & 0.252 $\pm$ 0.041 \\
0.1 & 0.428 $\pm$ 0.010 & 0.033 $\pm$ 0.003 & 0.509 $\pm$ 0.020 & 0.115 $\pm$ 0.017 & 0.081 $\pm$ 0.009 & 0.687 $\pm$ 0.010 & 0.164 $\pm$ 0.048 \\
0.2 & 0.415 $\pm$ 0.018 & 0.025 $\pm$ 0.003 & 0.501 $\pm$ 0.015 & 0.096 $\pm$ 0.015 & 0.073 $\pm$ 0.012 & 0.671 $\pm$ 0.010 & 0.101 $\pm$ 0.022 \\
0.5 & 0.397 $\pm$ 0.021 & 0.023 $\pm$ 0.003 & 0.451 $\pm$ 0.030 & 0.063 $\pm$ 0.016 & 0.063 $\pm$ 0.008 & 0.576 $\pm$ 0.030 & 0.063 $\pm$ 0.018 \\
1 & \textbf{0.355 $\pm$ 0.016} & 0.021 $\pm$ 0.004 & 0.396 $\pm$ 0.024 & \textbf{0.050 $\pm$ 0.012} & 0.047 $\pm$ 0.006 & 0.327 $\pm$ 0.012 & 0.059 $\pm$ 0.012 \\
2 & 0.377 $\pm$ 0.025 & \textbf{0.020 $\pm$ 0.003} & \textbf{0.387 $\pm$ 0.013} & \textbf{0.050 $\pm$ 0.012} & 0.038 $\pm$ 0.005 & \textbf{0.322 $\pm$ 0.009} & \textbf{0.040 $\pm$ 0.014} \\
5 & 0.411 $\pm$ 0.017 & 0.026 $\pm$ 0.003 & \textbf{0.387 $\pm$ 0.013} & 0.051 $\pm$ 0.013 & \textbf{0.034 $\pm$ 0.009} & 0.482 $\pm$ 0.020 & 0.042 $\pm$ 0.009\\
10 & 0.428 $\pm$ 0.014 & 0.026 $\pm$ 0.003 & \textbf{0.387 $\pm$ 0.015} & 0.051 $\pm$ 0.013 & 0.035 $\pm$ 0.011 & 0.494 $\pm$ 0.016 & 0.048 $\pm$ 0.014 \\
\bottomrule
\end{tabular}
}
\end{table}

%% file: tables/synthetic_datasets.tex
\begin{table}[!htbp]
\centering
\caption{Synthetic tasks; dataset information and reference point for HV computation.}
\label{tab:ref-synthetic}
\resizebox{\linewidth}{!}{%
\begin{tabular}{lccll}
\toprule
Name & $d$ & $m$ & Pareto Front Shape & Reference Point \\
\midrule
DTLZ1 &  7 & 3 & Linear        & (507.46516068, 502.09050332, 516.68985527) \\
DTLZ2 & 10 & 3 & Concave       & (2.51925352, 2.53011955, 2.66682285) \\
DTLZ3 & 10 & 3 & Concave       & (1548.83497615, 1459.58638514, 1518.61926977) \\
DTLZ4 & 10 & 3 & Concave       & (2.75747516, 2.57674881, 2.52558196) \\
DTLZ5 & 10 & 3 & Concave       & (2.40756462, 2.36937324, 2.45008585) \\
DTLZ6 & 10 & 3 & Concave       & (8.90471718, 8.89404128, 8.88893072) \\
DTLZ7 & 10 & 3 & Disconnected  & (0.99999540, 0.99999772, 30.79227037) \\
ZDT1  & 30 & 2 & Convex        & (0.99999809, 7.82506630) \\
ZDT2  & 30 & 2 & Concave       & (0.99999706, 9.74316166) \\
ZDT3  & 30 & 2 & Disconnected  & (0.99999954, 9.83643758) \\
ZDT4  & 10 & 2 & Convex        & (0.99980978, 273.25533168) \\
ZDT6  & 10 & 2 & Concave       & (1.00000000, 9.34365208) \\
\bottomrule
\end{tabular}
}
\end{table}

%% file: tables/re_datasets.tex
\begin{table}[!htbp]
\centering
\caption{Real-engineering tasks; dataset information and reference point for HV computation.}
\label{tab:ref-re}
\resizebox{\linewidth}{!}{%
\begin{tabular}{lccll}
\toprule
Name & $d$ & $m$ & Pareto Front Shape & Reference Point \\
\midrule
RE21 (Four bar truss design)           & 4 & 2 & Convex                & (2971.23988, 0.0485512338) \\
RE22 (Reinforced concrete beam design) & 3 & 2 & Mixed                 & (829.079443, 2407217.25) \\
RE23 (Pressure vessel design)          & 4 & 2 & Mixed, Disconnected   & (713710.875, 1288669.78054) \\
RE24 (Hatch cover design)              & 2 & 2 & Convex                & (5997.8316325, 43.67584229) \\
RE25 (Coil compression spring design)  & 3 & 2 & Mixed, Disconnected   & (124.795202, 10038735.0) \\
RE31 (Two bar truss design)            & 3 & 3 & Unknown               & (808.852742, 6893375.82, 6793450.00) \\
RE32 (Welded beam design)              & 4 & 3 & Unknown               & (290.661885, 16552.4628, 388265024.0) \\
RE33 (Disc brake design)               & 4 & 3 & Unknown               & (8.01164324, 8.83604223, 2343.29711914) \\
RE34 (Vehicle crashworthiness design)  & 5 & 3 & Unknown               & (1702.51811, 11.6807224, 0.263918844) \\
RE35 (Speed reducer design)            & 7 & 3 & Unknown               & (7050.78959905, 1696.66697789, 397.83456421) \\
RE36 (Gear train design)               & 4 & 3 & Concave, Disconnected & (10.21185714, 60.0, 0.97335988) \\
RE37 (Rocket injector design)          & 4 & 3 & Unknown               & (0.98949120096, 0.956587924661, 0.987530948586) \\
\bottomrule
\end{tabular}
}
\end{table}

%% file: tables/zdt_HV.tex
\begin{table}[!htbp]
\centering
\caption{HV results on ZDT ($\uparrow$). Bold $=$ best per column. Each algorithm is run for $5$ seeds and evaluated on $256$ designs.}
\resizebox{\linewidth}{!}{%
\begin{tabular}{lcccccc}
\toprule
Method & ZDT1 & ZDT2 & ZDT3 & ZDT4 & ZDT6 & Avg. Rank \\
\midrule
D(best) & 4.168 & 4.677 & 5.147 & 5.457 & 4.608 & -- \\
MultiHead-Vallina & 4.823 $\pm$ 0.030 & 5.612 $\pm$ 0.098 & 5.637 $\pm$ 0.131 & 4.938 $\pm$ 0.237 & 4.774 $\pm$ 0.003 & 4.200 \\
MultiHead-PcGrad & \textbf{4.831 $\pm$ 0.028} & 5.605 $\pm$ 0.047 & 5.524 $\pm$ 0.048 & 3.239 $\pm$ 0.678 & 4.776 $\pm$ 0.013 & 6.000 \\
MultiHead-GradNorm & 4.741 $\pm$ 0.034 & 5.483 $\pm$ 0.115 & 5.589 $\pm$ 0.192 & 3.489 $\pm$ 0.665 & 4.709 $\pm$ 0.025 & 8.000 \\
MultipleModels-Vallina & 4.812 $\pm$ 0.028 & 5.600 $\pm$ 0.043 & 5.536 $\pm$ 0.140 & 4.991 $\pm$ 0.067 & 4.774 $\pm$ 0.004 & 5.600 \\
MultipleModels-COM & 4.679 $\pm$ 0.112 & 5.093 $\pm$ 0.131 & 5.530 $\pm$ 0.072 & 5.260 $\pm$ 0.193 & 3.435 $\pm$ 0.750 & 8.800 \\
MultipleModels-IOM & 4.763 $\pm$ 0.033 & 5.530 $\pm$ 0.101 & 5.681 $\pm$ 0.038 & 5.338 $\pm$ 0.174 & 4.777 $\pm$ 0.005 & 4.000 \\
MultipleModels-ICT & 4.748 $\pm$ 0.004 & 5.428 $\pm$ 0.004 & 5.560 $\pm$ 0.049 & 4.830 $\pm$ 0.243 & 4.723 $\pm$ 0.051 & 7.400 \\
MultipleModels-RoMA & 4.817 $\pm$ 0.029 & 5.496 $\pm$ 0.067 & \textbf{5.894 $\pm$ 0.027} & 3.679 $\pm$ 0.049 & 2.597 $\pm$ 0.027 & 7.000 \\
MultipleModels-Trimentoring & 4.821 $\pm$ 0.002 & 5.548 $\pm$ 0.041 & 4.569 $\pm$ 0.083 & 4.557 $\pm$ 0.556 & 4.045 $\pm$ 0.122 & 8.000 \\
ParetoFlow & 4.136 $\pm$ 0.030 & 4.695 $\pm$ 0.071 & 5.171 $\pm$ 0.061 & 5.104 $\pm$ 0.125 & 4.683 $\pm$ 0.068 & 9.600 \\
PreferenceGuidedDiffusion & 4.567 $\pm$ 0.045 & 5.367 $\pm$ 0.064 & 5.540 $\pm$ 0.066 & 4.999 $\pm$ 0.085 & 4.823 $\pm$ 0.004 & 6.800 \\
\midrule
ParetoTransport & 4.829 $\pm$ 0.034 & \textbf{5.660 $\pm$ 0.026} & 5.531 $\pm$ 0.028 & \textbf{5.495 $\pm$ 0.000} & \textbf{4.825 $\pm$ 0.001} & \textbf{2.600} \\
\bottomrule
\end{tabular}
}
\label{tab:zdt_HV}
\end{table}

%% file: tables/zdt_GD.tex
\begin{table}[!htbp]
\centering
\caption{GD results on ZDT ($\downarrow$). Bold $=$ best per column. Each algorithm is run for $5$ seeds and evaluated on $256$ designs.}
\resizebox{\linewidth}{!}{%
\begin{tabular}{lcccccc}
\toprule
Method & ZDT1 & ZDT2 & ZDT3 & ZDT4 & ZDT6 & Avg. Rank \\
\midrule
D(best) & 0.378 & 0.433 & 0.382 & 0.051 & 0.064 & -- \\
MultiHead-Vallina & 0.032 $\pm$ 0.017 & \textbf{0.043 $\pm$ 0.048} & 0.154 $\pm$ 0.031 & 0.597 $\pm$ 0.108 & 0.113 $\pm$ 0.016 & 4.200 \\
MultiHead-PcGrad & \textbf{0.011 $\pm$ 0.002} & 0.091 $\pm$ 0.048 & 0.255 $\pm$ 0.018 & 0.935 $\pm$ 0.221 & 0.074 $\pm$ 0.120 & 5.000 \\
MultiHead-GradNorm & 0.098 $\pm$ 0.084 & 0.103 $\pm$ 0.058 & 0.169 $\pm$ 0.010 & 0.960 $\pm$ 0.107 & 0.098 $\pm$ 0.053 & 7.200 \\
MultipleModels-Vallina & 0.042 $\pm$ 0.019 & 0.074 $\pm$ 0.017 & 0.190 $\pm$ 0.021 & 0.494 $\pm$ 0.025 & 0.096 $\pm$ 0.011 & 4.000 \\
MultipleModels-COM & 0.072 $\pm$ 0.038 & 0.237 $\pm$ 0.072 & 0.234 $\pm$ 0.018 & 0.413 $\pm$ 0.077 & 0.301 $\pm$ 0.092 & 7.800 \\
MultipleModels-IOM & 0.084 $\pm$ 0.013 & 0.109 $\pm$ 0.037 & 0.149 $\pm$ 0.014 & 0.411 $\pm$ 0.037 & 0.112 $\pm$ 0.020 & 5.400 \\
MultipleModels-ICT & 0.049 $\pm$ 0.011 & 0.103 $\pm$ 0.009 & 0.202 $\pm$ 0.004 & 0.646 $\pm$ 0.075 & 0.252 $\pm$ 0.063 & 6.800 \\
MultipleModels-RoMA & 0.043 $\pm$ 0.018 & 0.113 $\pm$ 0.009 & \textbf{0.086 $\pm$ 0.011} & 0.972 $\pm$ 0.022 & 0.614 $\pm$ 0.028 & 7.600 \\
MultipleModels-Trimentoring & 0.049 $\pm$ 0.008 & 0.185 $\pm$ 0.032 & 0.702 $\pm$ 0.041 & 0.851 $\pm$ 0.216 & 0.513 $\pm$ 0.023 & 9.400 \\
ParetoFlow & 0.363 $\pm$ 0.019 & 0.455 $\pm$ 0.010 & 0.354 $\pm$ 0.019 & 0.395 $\pm$ 0.008 & 0.087 $\pm$ 0.024 & 8.000 \\
PreferenceGuidedDiffusion & 0.275 $\pm$ 0.041 & 0.184 $\pm$ 0.032 & 0.247 $\pm$ 0.010 & 0.565 $\pm$ 0.006 & 0.515 $\pm$ 0.041 & 9.000 \\
\midrule
ParetoTransport & 0.026 $\pm$ 0.012 & 0.102 $\pm$ 0.010 & 0.261 $\pm$ 0.013 & \textbf{0.174 $\pm$ 0.008} & \textbf{0.049 $\pm$ 0.002} & \textbf{3.600} \\
\bottomrule
\end{tabular}
}
\label{tab:zdt_GD}
\end{table}

%% file: tables/zdt_IGD.tex
\begin{table}[!htbp]
\centering
\caption{IGD results on ZDT ($\downarrow$). Bold $=$ best per column. Each algorithm is run for $5$ seeds and evaluated on $256$ designs.}
\resizebox{\linewidth}{!}{%
\begin{tabular}{lcccccc}
\toprule
Method & ZDT1 & ZDT2 & ZDT3 & ZDT4 & ZDT6 & Avg. Rank \\
\midrule
D(best) & 0.359 & 0.441 & 0.367 & 0.068 & 0.071 & -- \\
MultiHead-Vallina & 0.049 $\pm$ 0.076 & 0.094 $\pm$ 0.169 & 0.162 $\pm$ 0.108 & 0.607 $\pm$ 0.060 & 0.022 $\pm$ 0.002 & 5.200 \\
MultiHead-PcGrad & 0.036 $\pm$ 0.057 & 0.057 $\pm$ 0.055 & 0.194 $\pm$ 0.027 & 0.939 $\pm$ 0.231 & 0.020 $\pm$ 0.003 & 5.400 \\
MultiHead-GradNorm & 0.090 $\pm$ 0.044 & 0.100 $\pm$ 0.052 & 0.223 $\pm$ 0.167 & 0.879 $\pm$ 0.061 & 0.057 $\pm$ 0.016 & 8.400 \\
MultipleModels-Vallina & 0.058 $\pm$ 0.060 & 0.047 $\pm$ 0.020 & 0.232 $\pm$ 0.159 & 0.584 $\pm$ 0.020 & 0.022 $\pm$ 0.001 & 5.800 \\
MultipleModels-COM & 0.159 $\pm$ 0.159 & 0.288 $\pm$ 0.046 & 0.196 $\pm$ 0.051 & 0.423 $\pm$ 0.152 & 0.317 $\pm$ 0.132 & 8.200 \\
MultipleModels-IOM & 0.059 $\pm$ 0.020 & 0.081 $\pm$ 0.034 & 0.098 $\pm$ 0.014 & 0.524 $\pm$ 0.031 & 0.022 $\pm$ 0.001 & 4.200 \\
MultipleModels-ICT & 0.072 $\pm$ 0.011 & 0.489 $\pm$ 0.039 & 0.147 $\pm$ 0.009 & 0.622 $\pm$ 0.087 & 0.155 $\pm$ 0.185 & 7.800 \\
MultipleModels-RoMA & 0.031 $\pm$ 0.024 & 0.420 $\pm$ 0.161 & \textbf{0.017 $\pm$ 0.005} & 1.024 $\pm$ 0.014 & 0.505 $\pm$ 0.002 & 7.600 \\
MultipleModels-Trimentoring & 0.022 $\pm$ 0.002 & 0.084 $\pm$ 0.027 & 0.614 $\pm$ 0.035 & 0.699 $\pm$ 0.146 & 0.420 $\pm$ 0.049 & 7.800 \\
ParetoFlow & 0.342 $\pm$ 0.010 & 0.465 $\pm$ 0.023 & 0.336 $\pm$ 0.026 & 0.541 $\pm$ 0.019 & 0.224 $\pm$ 0.025 & 9.600 \\
PreferenceGuidedDiffusion & 0.155 $\pm$ 0.036 & 0.138 $\pm$ 0.019 & 0.173 $\pm$ 0.042 & 0.312 $\pm$ 0.033 & 0.023 $\pm$ 0.010 & 6.400 \\
\midrule
ParetoTransport & \textbf{0.018 $\pm$ 0.025} & \textbf{0.017 $\pm$ 0.008} & 0.155 $\pm$ 0.012 & \textbf{0.185 $\pm$ 0.002} & \textbf{0.015 $\pm$ 0.001} & \textbf{1.600} \\
\bottomrule
\end{tabular}
}
\label{tab:zdt_IGD}
\end{table}

%% file: tables/zdt_W2.tex
\begin{table}[!htbp]
\centering
\caption{$W_2$ results on ZDT ($\downarrow$). Bold $=$ best per column. Each algorithm is run for $5$ seeds and evaluated on $256$ designs.}
\resizebox{\linewidth}{!}{%
\begin{tabular}{lcccccc}
\toprule
Method & ZDT1 & ZDT2 & ZDT3 & ZDT4 & ZDT6 & Avg. Rank \\
\midrule
D(best) & 0.385 & 0.475 & 0.405 & 0.433 & 0.171 & -- \\
MultiHead-Vallina & 0.194 $\pm$ 0.144 & \textbf{0.130 $\pm$ 0.219} & 0.306 $\pm$ 0.149 & 0.802 $\pm$ 0.084 & 0.195 $\pm$ 0.036 & 4.400 \\
MultiHead-PcGrad & 0.088 $\pm$ 0.106 & 0.196 $\pm$ 0.175 & 0.290 $\pm$ 0.025 & 1.085 $\pm$ 0.203 & \textbf{0.087 $\pm$ 0.120} & 4.600 \\
MultiHead-GradNorm & 0.222 $\pm$ 0.133 & 0.221 $\pm$ 0.057 & 0.343 $\pm$ 0.106 & 1.038 $\pm$ 0.056 & 0.259 $\pm$ 0.126 & 7.000 \\
MultipleModels-Vallina & 0.328 $\pm$ 0.081 & 0.224 $\pm$ 0.128 & 0.375 $\pm$ 0.106 & 0.724 $\pm$ 0.019 & 0.186 $\pm$ 0.022 & 6.400 \\
MultipleModels-COM & 0.363 $\pm$ 0.154 & 0.440 $\pm$ 0.066 & 0.342 $\pm$ 0.052 & 0.629 $\pm$ 0.091 & 0.442 $\pm$ 0.176 & 7.600 \\
MultipleModels-IOM & 0.224 $\pm$ 0.049 & 0.134 $\pm$ 0.052 & 0.175 $\pm$ 0.017 & 0.662 $\pm$ 0.024 & 0.258 $\pm$ 0.036 & 4.200 \\
MultipleModels-ICT & 0.332 $\pm$ 0.045 & 0.522 $\pm$ 0.002 & 0.261 $\pm$ 0.008 & 0.836 $\pm$ 0.062 & 0.395 $\pm$ 0.144 & 7.600 \\
MultipleModels-RoMA & 0.089 $\pm$ 0.033 & 0.507 $\pm$ 0.040 & \textbf{0.137 $\pm$ 0.033} & 1.122 $\pm$ 0.019 & 0.776 $\pm$ 0.022 & 7.600 \\
MultipleModels-Trimentoring & 0.218 $\pm$ 0.027 & 0.448 $\pm$ 0.061 & 0.731 $\pm$ 0.048 & 0.960 $\pm$ 0.153 & 0.657 $\pm$ 0.021 & 9.200 \\
ParetoFlow & 0.491 $\pm$ 0.012 & 0.633 $\pm$ 0.006 & 0.471 $\pm$ 0.014 & 0.650 $\pm$ 0.005 & 0.417 $\pm$ 0.004 & 9.400 \\
PreferenceGuidedDiffusion & 0.333 $\pm$ 0.058 & 0.344 $\pm$ 0.022 & 0.447 $\pm$ 0.045 & 0.579 $\pm$ 0.007 & 0.591 $\pm$ 0.032 & 7.800 \\
\midrule
ParetoTransport & \textbf{0.051 $\pm$ 0.026} & 0.185 $\pm$ 0.007 & 0.289 $\pm$ 0.012 & \textbf{0.362 $\pm$ 0.005} & 0.153 $\pm$ 0.001 & \textbf{2.200} \\
\bottomrule
\end{tabular}
}
\label{tab:zdt_W2}
\end{table}

%% file: tables/dtlz_HV.tex
\begin{table}[!htbp]
\centering
\caption{HV results on DTLZ ($\uparrow$). Bold $=$ best per column. Each algorithm is run for $5$ seeds and evaluated on $256$ designs.}
\resizebox{\linewidth}{!}{%
\begin{tabular}{lcccccccc}
\toprule
Method & DTLZ1 & DTLZ2 & DTLZ3 & DTLZ4 & DTLZ5 & DTLZ6 & DTLZ7 & Avg. Rank \\
\midrule
D(best) & 10.614 & 12.427 & 9.897 & 17.503 & 10.692 & 10.671 & 8.992 & -- \\
MultiHead-Vallina & 10.643 $\pm$ 0.002 & 12.444 $\pm$ 0.000 & 9.882 $\pm$ 0.004 & 17.567 $\pm$ 0.091 & 10.758 $\pm$ 0.001 & 10.900 $\pm$ 0.019 & 10.715 $\pm$ 0.027 & 3.714 \\
MultiHead-PcGrad & 10.644 $\pm$ 0.001 & 12.443 $\pm$ 0.001 & 9.895 $\pm$ 0.004 & 17.533 $\pm$ 0.014 & \textbf{10.759 $\pm$ 0.001} & 10.918 $\pm$ 0.011 & 10.598 $\pm$ 0.028 & 2.857 \\
MultiHead-GradNorm & 10.537 $\pm$ 0.239 & 12.439 $\pm$ 0.001 & 9.657 $\pm$ 0.307 & 17.451 $\pm$ 0.190 & 9.833 $\pm$ 2.028 & 10.848 $\pm$ 0.060 & 9.654 $\pm$ 0.735 & 9.000 \\
MultipleModels-Vallina & 10.644 $\pm$ 0.001 & 12.445 $\pm$ 0.000 & 9.886 $\pm$ 0.003 & 17.046 $\pm$ 0.179 & 10.758 $\pm$ 0.000 & 10.928 $\pm$ 0.008 & 10.535 $\pm$ 0.079 & 4.429 \\
MultipleModels-COM & 10.642 $\pm$ 0.002 & 12.433 $\pm$ 0.001 & 9.877 $\pm$ 0.006 & 17.077 $\pm$ 0.381 & 10.747 $\pm$ 0.003 & 10.786 $\pm$ 0.017 & 9.809 $\pm$ 0.273 & 7.714 \\
MultipleModels-IOM & 10.643 $\pm$ 0.002 & 12.443 $\pm$ 0.001 & \textbf{9.898 $\pm$ 0.000} & 16.937 $\pm$ 0.203 & 10.757 $\pm$ 0.001 & 10.886 $\pm$ 0.008 & 10.544 $\pm$ 0.126 & 5.143 \\
MultipleModels-ICT & 10.644 $\pm$ 0.001 & 12.429 $\pm$ 0.003 & 9.855 $\pm$ 0.050 & 17.109 $\pm$ 0.251 & 10.688 $\pm$ 0.011 & 10.781 $\pm$ 0.022 & 10.342 $\pm$ 0.092 & 7.286 \\
MultipleModels-RoMA & 10.641 $\pm$ 0.002 & 12.349 $\pm$ 0.014 & 9.878 $\pm$ 0.002 & 15.201 $\pm$ 0.490 & 10.594 $\pm$ 0.033 & 10.742 $\pm$ 0.015 & 10.596 $\pm$ 0.036 & 8.714 \\
MultipleModels-Trimentoring & 10.642 $\pm$ 0.007 & 11.854 $\pm$ 0.892 & 9.326 $\pm$ 0.020 & 17.107 $\pm$ 0.040 & 10.650 $\pm$ 0.004 & 10.778 $\pm$ 0.049 & 10.392 $\pm$ 0.423 & 9.000 \\
ParetoFlow & 10.626 $\pm$ 0.006 & 12.253 $\pm$ 0.075 & 9.792 $\pm$ 0.054 & 17.508 $\pm$ 0.005 & 10.546 $\pm$ 0.046 & 10.686 $\pm$ 0.039 & 8.875 $\pm$ 0.099 & 9.857 \\
PreferenceGuidedDiffusion & 10.645 $\pm$ 0.001 & 12.387 $\pm$ 0.024 & 9.889 $\pm$ 0.002 & 14.504 $\pm$ 0.074 & 10.369 $\pm$ 0.357 & 10.336 $\pm$ 0.204 & 9.767 $\pm$ 0.141 & 8.571 \\
\midrule
ParetoTransport & \textbf{10.646 $\pm$ 0.000} & \textbf{12.448 $\pm$ 0.000} & 9.897 $\pm$ 0.001 & \textbf{17.605 $\pm$ 0.009} & 10.754 $\pm$ 0.001 & \textbf{11.028 $\pm$ 0.019} & \textbf{10.728 $\pm$ 0.076} & \textbf{1.714} \\
\bottomrule
\end{tabular}
}
\label{tab:dtlz_HV}
\end{table}

%% file: tables/dtlz_GD.tex
\begin{table}[!htbp]
\centering
\caption{GD results on DTLZ ($\downarrow$). Bold $=$ best per column. Each algorithm is run for $5$ seeds and evaluated on $256$ designs.}
\resizebox{\linewidth}{!}{%
\begin{tabular}{lcccccccc}
\toprule
Method & DTLZ1 & DTLZ2 & DTLZ3 & DTLZ4 & DTLZ5 & DTLZ6 & DTLZ7 & Avg. Rank \\
\midrule
D(best) & 0.601 & 0.102 & 0.341 & 0.185 & 0.131 & 0.759 & 0.490 & -- \\
MultiHead-Vallina & 0.469 $\pm$ 0.021 & 0.066 $\pm$ 0.014 & 0.531 $\pm$ 0.035 & 0.275 $\pm$ 0.037 & 0.064 $\pm$ 0.006 & 0.555 $\pm$ 0.027 & 0.190 $\pm$ 0.022 & 6.000 \\
MultiHead-PcGrad & 0.414 $\pm$ 0.032 & 0.066 $\pm$ 0.015 & 0.435 $\pm$ 0.020 & 0.321 $\pm$ 0.048 & 0.049 $\pm$ 0.007 & 0.521 $\pm$ 0.036 & 0.078 $\pm$ 0.048 & 4.000 \\
MultiHead-GradNorm & 0.402 $\pm$ 0.029 & 0.067 $\pm$ 0.024 & 0.401 $\pm$ 0.116 & 0.307 $\pm$ 0.034 & 0.244 $\pm$ 0.234 & 0.596 $\pm$ 0.082 & 0.209 $\pm$ 0.065 & 5.143 \\
MultipleModels-Vallina & 0.407 $\pm$ 0.029 & 0.035 $\pm$ 0.009 & 0.486 $\pm$ 0.029 & 0.557 $\pm$ 0.102 & \textbf{0.011 $\pm$ 0.005} & 0.522 $\pm$ 0.010 & 0.162 $\pm$ 0.020 & 4.143 \\
MultipleModels-COM & 0.459 $\pm$ 0.016 & 0.162 $\pm$ 0.023 & 0.527 $\pm$ 0.017 & 0.285 $\pm$ 0.036 & 0.145 $\pm$ 0.027 & 0.684 $\pm$ 0.008 & 0.498 $\pm$ 0.021 & 8.571 \\
MultipleModels-IOM & 0.436 $\pm$ 0.021 & 0.078 $\pm$ 0.020 & 0.424 $\pm$ 0.009 & 0.278 $\pm$ 0.034 & 0.090 $\pm$ 0.022 & 0.606 $\pm$ 0.012 & 0.305 $\pm$ 0.053 & 5.714 \\
MultipleModels-ICT & 0.450 $\pm$ 0.019 & 0.279 $\pm$ 0.012 & 0.493 $\pm$ 0.027 & 0.614 $\pm$ 0.036 & 0.174 $\pm$ 0.019 & 0.692 $\pm$ 0.023 & 0.426 $\pm$ 0.068 & 9.286 \\
MultipleModels-RoMA & 0.446 $\pm$ 0.009 & 0.303 $\pm$ 0.040 & 0.561 $\pm$ 0.023 & 0.707 $\pm$ 0.050 & 0.273 $\pm$ 0.042 & 0.724 $\pm$ 0.008 & 0.397 $\pm$ 0.030 & 10.857 \\
MultipleModels-Trimentoring & \textbf{0.250 $\pm$ 0.095} & 0.275 $\pm$ 0.058 & \textbf{0.175 $\pm$ 0.008} & 0.594 $\pm$ 0.040 & 0.206 $\pm$ 0.022 & 0.711 $\pm$ 0.061 & 0.288 $\pm$ 0.136 & 6.714 \\
ParetoFlow & 0.510 $\pm$ 0.023 & 0.130 $\pm$ 0.015 & 0.513 $\pm$ 0.023 & 0.193 $\pm$ 0.004 & 0.150 $\pm$ 0.010 & 0.725 $\pm$ 0.040 & 0.479 $\pm$ 0.016 & 8.429 \\
PreferenceGuidedDiffusion & 0.402 $\pm$ 0.011 & 0.205 $\pm$ 0.051 & 0.523 $\pm$ 0.011 & 0.212 $\pm$ 0.022 & 0.252 $\pm$ 0.027 & 0.667 $\pm$ 0.006 & 0.359 $\pm$ 0.028 & 7.143 \\
\midrule
ParetoTransport & 0.403 $\pm$ 0.011 & \textbf{0.023 $\pm$ 0.004} & 0.408 $\pm$ 0.015 & \textbf{0.063 $\pm$ 0.012} & 0.063 $\pm$ 0.007 & \textbf{0.493 $\pm$ 0.027} & \textbf{0.050 $\pm$ 0.016} & \textbf{2.000} \\
\bottomrule
\end{tabular}
}
\label{tab:dtlz_GD}
\end{table}

%% file: tables/dtlz_IGD.tex
\begin{table}[!htbp]
\centering
\caption{IGD results on DTLZ ($\downarrow$). Bold $=$ best per column. Each algorithm is run for $5$ seeds and evaluated on $256$ designs.}
\resizebox{\linewidth}{!}{%
\begin{tabular}{lcccccccc}
\toprule
Method & DTLZ1 & DTLZ2 & DTLZ3 & DTLZ4 & DTLZ5 & DTLZ6 & DTLZ7 & Avg. Rank \\
\midrule
D(best) & 0.259 & 0.059 & 0.083 & 0.209 & 0.042 & 0.684 & 0.390 & -- \\
MultiHead-Vallina & 0.132 $\pm$ 0.038 & 0.024 $\pm$ 0.001 & 0.215 $\pm$ 0.020 & \textbf{0.158 $\pm$ 0.056} & \textbf{0.003 $\pm$ 0.000} & 0.467 $\pm$ 0.016 & \textbf{0.076 $\pm$ 0.040} & 5.143 \\
MultiHead-PcGrad & 0.114 $\pm$ 0.018 & 0.025 $\pm$ 0.002 & 0.103 $\pm$ 0.039 & \textbf{0.158 $\pm$ 0.006} & \textbf{0.003 $\pm$ 0.000} & 0.448 $\pm$ 0.020 & 0.353 $\pm$ 0.084 & 4.143 \\
MultiHead-GradNorm & 0.108 $\pm$ 0.027 & 0.032 $\pm$ 0.001 & 0.150 $\pm$ 0.043 & 0.211 $\pm$ 0.042 & 0.161 $\pm$ 0.334 & 0.494 $\pm$ 0.058 & 0.555 $\pm$ 0.272 & 7.714 \\
MultipleModels-Vallina & 0.101 $\pm$ 0.027 & 0.023 $\pm$ 0.001 & 0.185 $\pm$ 0.025 & 0.407 $\pm$ 0.096 & \textbf{0.003 $\pm$ 0.000} & 0.453 $\pm$ 0.009 & 0.247 $\pm$ 0.127 & 5.571 \\
MultipleModels-COM & 0.125 $\pm$ 0.016 & 0.042 $\pm$ 0.003 & 0.183 $\pm$ 0.020 & 0.318 $\pm$ 0.065 & 0.008 $\pm$ 0.001 & 0.528 $\pm$ 0.020 & 0.225 $\pm$ 0.039 & 7.143 \\
MultipleModels-IOM & 0.128 $\pm$ 0.023 & 0.027 $\pm$ 0.002 & \textbf{0.072 $\pm$ 0.015} & 0.336 $\pm$ 0.030 & \textbf{0.003 $\pm$ 0.000} & 0.483 $\pm$ 0.014 & 0.158 $\pm$ 0.108 & 5.571 \\
MultipleModels-ICT & 0.095 $\pm$ 0.015 & 0.061 $\pm$ 0.007 & 0.185 $\pm$ 0.028 & 0.312 $\pm$ 0.009 & 0.076 $\pm$ 0.014 & 0.571 $\pm$ 0.014 & 0.309 $\pm$ 0.024 & 7.714 \\
MultipleModels-RoMA & 0.130 $\pm$ 0.032 & 0.183 $\pm$ 0.015 & 0.215 $\pm$ 0.055 & 0.779 $\pm$ 0.048 & 0.101 $\pm$ 0.020 & 0.616 $\pm$ 0.017 & 0.097 $\pm$ 0.007 & 9.714 \\
MultipleModels-Trimentoring & \textbf{0.076 $\pm$ 0.006} & 0.196 $\pm$ 0.114 & 0.124 $\pm$ 0.004 & 0.362 $\pm$ 0.066 & 0.160 $\pm$ 0.000 & 0.570 $\pm$ 0.038 & 0.126 $\pm$ 0.090 & 7.000 \\
ParetoFlow & 0.087 $\pm$ 0.041 & 0.121 $\pm$ 0.009 & 0.217 $\pm$ 0.045 & 0.208 $\pm$ 0.001 & 0.077 $\pm$ 0.009 & 0.608 $\pm$ 0.040 & 0.406 $\pm$ 0.022 & 8.571 \\
PreferenceGuidedDiffusion & 0.083 $\pm$ 0.029 & 0.120 $\pm$ 0.031 & 0.152 $\pm$ 0.023 & 0.457 $\pm$ 0.005 & 0.123 $\pm$ 0.016 & 0.427 $\pm$ 0.043 & 0.268 $\pm$ 0.034 & 7.000 \\
\midrule
ParetoTransport & 0.082 $\pm$ 0.016 & \textbf{0.021 $\pm$ 0.001} & 0.124 $\pm$ 0.009 & 0.161 $\pm$ 0.013 & 0.006 $\pm$ 0.001 & \textbf{0.284 $\pm$ 0.001} & 0.132 $\pm$ 0.083 & \textbf{2.714} \\
\bottomrule
\end{tabular}
}
\label{tab:dtlz_IGD}
\end{table}

%% file: tables/dtlz_W2.tex
\begin{table}[!htbp]
\centering
\caption{$W_2$ results on DTLZ ($\downarrow$). Bold $=$ best per column. Each algorithm is run for $5$ seeds and evaluated on $256$ designs.}
\resizebox{\linewidth}{!}{%
\begin{tabular}{lcccccccc}
\toprule
Method & DTLZ1 & DTLZ2 & DTLZ3 & DTLZ4 & DTLZ5 & DTLZ6 & DTLZ7 & Avg. Rank \\
\midrule
D(best) & 0.602 & 0.113 & 0.341 & 0.461 & 0.199 & 0.764 & 0.508 & -- \\
MultiHead-Vallina & 0.469 $\pm$ 0.021 & 0.107 $\pm$ 0.022 & 0.531 $\pm$ 0.035 & 0.351 $\pm$ 0.051 & 0.110 $\pm$ 0.012 & 0.558 $\pm$ 0.026 & 0.294 $\pm$ 0.038 & 5.714 \\
MultiHead-PcGrad & 0.415 $\pm$ 0.032 & 0.100 $\pm$ 0.017 & 0.435 $\pm$ 0.020 & 0.380 $\pm$ 0.048 & 0.096 $\pm$ 0.012 & 0.528 $\pm$ 0.030 & 0.415 $\pm$ 0.083 & 4.286 \\
MultiHead-GradNorm & 0.402 $\pm$ 0.029 & 0.095 $\pm$ 0.026 & 0.402 $\pm$ 0.116 & 0.423 $\pm$ 0.038 & 0.323 $\pm$ 0.297 & 0.607 $\pm$ 0.082 & 0.683 $\pm$ 0.097 & 5.571 \\
MultipleModels-Vallina & 0.408 $\pm$ 0.029 & 0.063 $\pm$ 0.013 & 0.487 $\pm$ 0.029 & 0.659 $\pm$ 0.106 & \textbf{0.037 $\pm$ 0.013} & 0.526 $\pm$ 0.010 & 0.412 $\pm$ 0.053 & 4.286 \\
MultipleModels-COM & 0.459 $\pm$ 0.016 & 0.179 $\pm$ 0.023 & 0.527 $\pm$ 0.017 & 0.528 $\pm$ 0.055 & 0.188 $\pm$ 0.026 & 0.688 $\pm$ 0.008 & 0.526 $\pm$ 0.021 & 7.714 \\
MultipleModels-IOM & 0.436 $\pm$ 0.021 & 0.101 $\pm$ 0.021 & 0.424 $\pm$ 0.009 & 0.541 $\pm$ 0.025 & 0.125 $\pm$ 0.023 & 0.609 $\pm$ 0.013 & 0.398 $\pm$ 0.091 & 5.429 \\
MultipleModels-ICT & 0.451 $\pm$ 0.019 & 0.312 $\pm$ 0.012 & 0.493 $\pm$ 0.027 & 0.756 $\pm$ 0.059 & 0.255 $\pm$ 0.020 & 0.697 $\pm$ 0.024 & 0.567 $\pm$ 0.025 & 9.286 \\
MultipleModels-RoMA & 0.446 $\pm$ 0.009 & 0.342 $\pm$ 0.036 & 0.561 $\pm$ 0.023 & 0.931 $\pm$ 0.029 & 0.282 $\pm$ 0.041 & 0.726 $\pm$ 0.009 & 0.498 $\pm$ 0.036 & 10.000 \\
MultipleModels-Trimentoring & \textbf{0.250 $\pm$ 0.095} & 0.419 $\pm$ 0.119 & \textbf{0.175 $\pm$ 0.008} & 0.693 $\pm$ 0.045 & 0.319 $\pm$ 0.028 & 0.714 $\pm$ 0.061 & 0.349 $\pm$ 0.123 & 6.714 \\
ParetoFlow & 0.510 $\pm$ 0.023 & 0.200 $\pm$ 0.020 & 0.513 $\pm$ 0.023 & 0.463 $\pm$ 0.003 & 0.196 $\pm$ 0.018 & 0.732 $\pm$ 0.040 & 0.543 $\pm$ 0.012 & 8.714 \\
PreferenceGuidedDiffusion & 0.403 $\pm$ 0.011 & 0.255 $\pm$ 0.054 & 0.523 $\pm$ 0.011 & 0.566 $\pm$ 0.016 & 0.405 $\pm$ 0.048 & 0.689 $\pm$ 0.005 & 0.544 $\pm$ 0.037 & 8.429 \\
\midrule
ParetoTransport & 0.403 $\pm$ 0.012 & \textbf{0.039 $\pm$ 0.003} & 0.408 $\pm$ 0.015 & \textbf{0.339 $\pm$ 0.020} & 0.081 $\pm$ 0.011 & \textbf{0.501 $\pm$ 0.026} & \textbf{0.274 $\pm$ 0.122} & \textbf{1.857} \\
\bottomrule
\end{tabular}
}
\label{tab:dtlz_W2}
\end{table}

%% file: tables/re_HV.tex
\begin{table}[!htbp]
\centering
\caption{HV results on RE ($\uparrow$). Bold $=$ best per column. Each algorithm is run for $5$ seeds and evaluated on $256$ designs.}
\resizebox{\linewidth}{!}{%
\begin{tabular}{lccccccccccccc}
\toprule
Method & RE21 & RE22 & RE23 & RE24 & RE25 & RE31 & RE32 & RE33 & RE34 & RE35 & RE36 & RE37 & Avg. Rank \\
\midrule
D(best) & 4.176 & 4.806 & 4.751 & 4.595 & 4.794 & 10.607 & 10.611 & 10.533 & 9.652 & 10.217 & 7.760 & 5.916 & -- \\
MultiHead-Vallina & \textbf{4.598 $\pm$ 0.000} & \textbf{4.840 $\pm$ 0.000} & \textbf{4.840 $\pm$ 0.000} & 4.075 $\pm$ 0.815 & 4.540 $\pm$ 0.287 & 10.647 $\pm$ 0.001 & 10.632 $\pm$ 0.009 & 10.619 $\pm$ 0.007 & 10.109 $\pm$ 0.004 & 10.556 $\pm$ 0.024 & 10.201 $\pm$ 0.074 & \textbf{6.740 $\pm$ 0.003} & 4.750 \\
MultiHead-PcGrad & \textbf{4.598 $\pm$ 0.001} & 4.828 $\pm$ 0.023 & 4.763 $\pm$ 0.170 & 4.544 $\pm$ 0.231 & 4.739 $\pm$ 0.206 & \textbf{10.648 $\pm$ 0.000} & 10.628 $\pm$ 0.002 & 10.588 $\pm$ 0.025 & 10.101 $\pm$ 0.007 & 10.498 $\pm$ 0.085 & 10.189 $\pm$ 0.070 & 6.650 $\pm$ 0.075 & 7.000 \\
MultiHead-GradNorm & 4.120 $\pm$ 0.606 & 4.721 $\pm$ 0.218 & 2.714 $\pm$ 0.067 & 4.333 $\pm$ 0.094 & 4.689 $\pm$ 0.295 & 10.228 $\pm$ 0.577 & 10.505 $\pm$ 0.190 & 6.970 $\pm$ 0.349 & 8.408 $\pm$ 2.057 & 8.225 $\pm$ 3.414 & 8.496 $\pm$ 1.387 & 5.898 $\pm$ 1.595 & 11.500 \\
MultipleModels-Vallina & \textbf{4.598 $\pm$ 0.001} & \textbf{4.840 $\pm$ 0.000} & 4.831 $\pm$ 0.018 & 4.831 $\pm$ 0.010 & 4.838 $\pm$ 0.002 & \textbf{10.648 $\pm$ 0.000} & 10.639 $\pm$ 0.002 & 10.618 $\pm$ 0.004 & \textbf{10.112 $\pm$ 0.003} & 10.547 $\pm$ 0.035 & 10.230 $\pm$ 0.091 & 6.739 $\pm$ 0.006 & \textbf{3.417} \\
MultipleModels-COM & 4.238 $\pm$ 0.012 & 4.809 $\pm$ 0.012 & 4.817 $\pm$ 0.033 & 4.658 $\pm$ 0.071 & 4.803 $\pm$ 0.004 & 10.595 $\pm$ 0.005 & 10.617 $\pm$ 0.004 & 10.611 $\pm$ 0.008 & 9.858 $\pm$ 0.079 & 10.471 $\pm$ 0.019 & 9.540 $\pm$ 0.049 & 6.245 $\pm$ 0.071 & 8.750 \\
MultipleModels-IOM & 4.594 $\pm$ 0.007 & \textbf{4.840 $\pm$ 0.000} & 4.813 $\pm$ 0.034 & 4.789 $\pm$ 0.104 & \textbf{4.841 $\pm$ 0.000} & \textbf{10.648 $\pm$ 0.000} & \textbf{10.648 $\pm$ 0.000} & 10.620 $\pm$ 0.001 & 10.105 $\pm$ 0.005 & 10.569 $\pm$ 0.009 & 10.226 $\pm$ 0.138 & 6.723 $\pm$ 0.008 & 3.500 \\
MultipleModels-ICT & \textbf{4.598 $\pm$ 0.001} & \textbf{4.840 $\pm$ 0.000} & 4.455 $\pm$ 0.071 & 4.830 $\pm$ 0.011 & 4.834 $\pm$ 0.003 & \textbf{10.648 $\pm$ 0.000} & 10.639 $\pm$ 0.003 & 10.621 $\pm$ 0.004 & 10.101 $\pm$ 0.003 & 10.490 $\pm$ 0.006 & \textbf{10.312 $\pm$ 0.012} & 6.732 $\pm$ 0.007 & 4.833 \\
MultipleModels-RoMA & 4.566 $\pm$ 0.003 & 4.543 $\pm$ 0.664 & 4.836 $\pm$ 0.006 & 3.954 $\pm$ 1.200 & 4.832 $\pm$ 0.010 & 10.644 $\pm$ 0.004 & 10.635 $\pm$ 0.005 & 10.592 $\pm$ 0.021 & 9.917 $\pm$ 0.010 & 10.533 $\pm$ 0.022 & 9.490 $\pm$ 0.442 & 6.611 $\pm$ 0.054 & 7.750 \\
MultipleModels-Trimentoring & 4.595 $\pm$ 0.001 & \textbf{4.840 $\pm$ 0.000} & \textbf{4.840 $\pm$ 0.000} & 4.530 $\pm$ 0.406 & 4.807 $\pm$ 0.002 & \textbf{10.648 $\pm$ 0.000} & 10.631 $\pm$ 0.010 & \textbf{10.626 $\pm$ 0.000} & 10.093 $\pm$ 0.004 & 10.521 $\pm$ 0.020 & 7.270 $\pm$ 2.131 & 6.718 $\pm$ 0.017 & 5.750 \\
ParetoFlow & 4.133 $\pm$ 0.093 & 4.815 $\pm$ 0.028 & 4.800 $\pm$ 0.022 & 4.832 $\pm$ 0.005 & 4.823 $\pm$ 0.005 & 10.528 $\pm$ 0.103 & 10.630 $\pm$ 0.012 & 10.340 $\pm$ 0.087 & 9.693 $\pm$ 0.037 & 10.410 $\pm$ 0.044 & 8.785 $\pm$ 0.221 & 5.595 $\pm$ 0.405 & 9.250 \\
PreferenceGuidedDiffusion & 4.479 $\pm$ 0.046 & 4.836 $\pm$ 0.003 & 4.837 $\pm$ 0.003 & \textbf{4.834 $\pm$ 0.001} & 4.839 $\pm$ 0.000 & 10.613 $\pm$ 0.021 & 10.646 $\pm$ 0.001 & 10.430 $\pm$ 0.114 & 9.444 $\pm$ 0.164 & 10.248 $\pm$ 0.110 & 9.426 $\pm$ 0.223 & 6.182 $\pm$ 0.131 & 7.167 \\
\midrule
ParetoTransport & 4.588 $\pm$ 0.004 & 4.839 $\pm$ 0.002 & \textbf{4.840 $\pm$ 0.000} & 4.831 $\pm$ 0.002 & 4.840 $\pm$ 0.001 & 10.613 $\pm$ 0.000 & 10.644 $\pm$ 0.003 & 10.621 $\pm$ 0.004 & 10.045 $\pm$ 0.014 & \textbf{10.581 $\pm$ 0.005} & 10.118 $\pm$ 0.207 & 6.670 $\pm$ 0.009 & 4.333 \\
\bottomrule
\end{tabular}
}
\label{tab:re_HV}
\end{table}

%% file: tables/re_GD.tex
\begin{table}[!htbp]
\centering
\caption{GD results on RE ($\downarrow$). Bold $=$ best per column. Each algorithm is run for $5$ seeds and evaluated on $256$ designs.}
\resizebox{\linewidth}{!}{%
\begin{tabular}{lccccccccccccc}
\toprule
Method & RE21 & RE22 & RE23 & RE24 & RE25 & RE31 & RE32 & RE33 & RE34 & RE35 & RE36 & RE37 & Avg. Rank \\
\midrule
D(best) & 0.111 & 0.001 & 0.034 & 0.040 & 0.030 & 0.017 & 0.051 & 0.007 & 0.065 & 0.077 & 0.377 & 0.093 & -- \\
MultiHead-Vallina & \textbf{0.001 $\pm$ 0.000} & 272.464 $\pm$ 557.607 & 0.219 $\pm$ 0.071 & 0.460 $\pm$ 0.364 & 0.461 $\pm$ 0.059 & 0.054 $\pm$ 0.023 & 0.006 $\pm$ 0.010 & 0.093 $\pm$ 0.033 & \textbf{0.007 $\pm$ 0.001} & \textbf{0.005 $\pm$ 0.003} & 0.160 $\pm$ 0.031 & 0.016 $\pm$ 0.001 & 6.417 \\
MultiHead-PcGrad & 0.002 $\pm$ 0.000 & 96.838 $\pm$ 169.131 & 0.217 $\pm$ 0.125 & 0.251 $\pm$ 0.320 & 0.255 $\pm$ 0.254 & 0.039 $\pm$ 0.028 & 0.154 $\pm$ 0.344 & 0.084 $\pm$ 0.029 & 0.029 $\pm$ 0.010 & 0.095 $\pm$ 0.088 & 0.177 $\pm$ 0.035 & \textbf{0.013 $\pm$ 0.001} & 7.167 \\
MultiHead-GradNorm & 0.081 $\pm$ 0.157 & 0.919 $\pm$ 2.027 & 0.020 $\pm$ 0.020 & 0.240 $\pm$ 0.061 & 0.211 $\pm$ 0.144 & 0.019 $\pm$ 0.018 & 0.133 $\pm$ 0.170 & 0.335 $\pm$ 0.055 & 0.160 $\pm$ 0.082 & 0.242 $\pm$ 0.338 & 0.198 $\pm$ 0.169 & 0.081 $\pm$ 0.146 & 7.833 \\
MultipleModels-Vallina & \textbf{0.001 $\pm$ 0.000} & 0.084 $\pm$ 0.179 & 0.041 $\pm$ 0.049 & 0.443 $\pm$ 0.066 & 0.153 $\pm$ 0.112 & 0.040 $\pm$ 0.020 & 0.190 $\pm$ 0.184 & 0.102 $\pm$ 0.038 & \textbf{0.007 $\pm$ 0.001} & 0.039 $\pm$ 0.081 & 0.143 $\pm$ 0.026 & 0.016 $\pm$ 0.002 & 5.500 \\
MultipleModels-COM & 0.062 $\pm$ 0.002 & 6.222 $\pm$ 13.724 & 0.023 $\pm$ 0.009 & \textbf{0.035 $\pm$ 0.014} & 0.125 $\pm$ 0.043 & \textbf{0.015 $\pm$ 0.013} & 0.039 $\pm$ 0.049 & 0.012 $\pm$ 0.004 & 0.041 $\pm$ 0.005 & 0.035 $\pm$ 0.012 & 0.162 $\pm$ 0.021 & 0.054 $\pm$ 0.003 & 5.333 \\
MultipleModels-IOM & 0.006 $\pm$ 0.005 & 3.791 $\pm$ 4.118 & 0.043 $\pm$ 0.031 & 0.118 $\pm$ 0.055 & 4.658 $\pm$ 10.412 & 0.037 $\pm$ 0.019 & 0.211 $\pm$ 0.058 & 0.053 $\pm$ 0.036 & 0.010 $\pm$ 0.002 & 0.072 $\pm$ 0.005 & \textbf{0.092 $\pm$ 0.016} & 0.023 $\pm$ 0.003 & 6.333 \\
MultipleModels-ICT & \textbf{0.001 $\pm$ 0.000} & 17.170 $\pm$ 15.728 & 0.006 $\pm$ 0.000 & 0.437 $\pm$ 0.033 & 0.284 $\pm$ 0.025 & 0.050 $\pm$ 0.006 & \textbf{0.001 $\pm$ 0.000} & 0.057 $\pm$ 0.042 & 0.022 $\pm$ 0.002 & 0.015 $\pm$ 0.002 & 0.106 $\pm$ 0.009 & 0.019 $\pm$ 0.002 & 5.250 \\
MultipleModels-RoMA & 0.042 $\pm$ 0.008 & 198.883 $\pm$ 411.102 & 0.087 $\pm$ 0.187 & 0.648 $\pm$ 0.260 & 0.260 $\pm$ 0.221 & \textbf{0.015 $\pm$ 0.007} & 0.355 $\pm$ 0.065 & 0.105 $\pm$ 0.033 & 0.105 $\pm$ 0.006 & 0.026 $\pm$ 0.006 & 0.574 $\pm$ 0.052 & 0.058 $\pm$ 0.002 & 9.167 \\
MultipleModels-Trimentoring & 0.003 $\pm$ 0.001 & 0.009 $\pm$ 0.021 & 0.072 $\pm$ 0.027 & 0.367 $\pm$ 0.216 & 0.029 $\pm$ 0.006 & 0.071 $\pm$ 0.010 & 0.352 $\pm$ 0.073 & 0.009 $\pm$ 0.003 & 0.013 $\pm$ 0.002 & 0.019 $\pm$ 0.009 & 0.623 $\pm$ 0.268 & 0.017 $\pm$ 0.002 & 6.167 \\
ParetoFlow & 0.042 $\pm$ 0.014 & \textbf{0.000 $\pm$ 0.000} & 0.073 $\pm$ 0.006 & 0.117 $\pm$ 0.004 & 0.046 $\pm$ 0.006 & 0.024 $\pm$ 0.023 & 0.058 $\pm$ 0.005 & \textbf{0.007 $\pm$ 0.000} & 0.058 $\pm$ 0.007 & 0.065 $\pm$ 0.004 & 0.278 $\pm$ 0.022 & 0.118 $\pm$ 0.009 & 6.083 \\
PreferenceGuidedDiffusion & 0.115 $\pm$ 0.003 & inf $\pm$ nan & 0.050 $\pm$ 0.004 & 0.470 $\pm$ 0.060 & 0.035 $\pm$ 0.016 & 0.193 $\pm$ 0.074 & 0.084 $\pm$ 0.015 & 0.017 $\pm$ 0.003 & 0.101 $\pm$ 0.009 & 0.011 $\pm$ 0.009 & 0.469 $\pm$ 0.013 & 0.119 $\pm$ 0.008 & 8.333 \\
\midrule
ParetoTransport & 0.002 $\pm$ 0.000 & 0.009 $\pm$ 0.010 & \textbf{0.002 $\pm$ 0.001} & 0.074 $\pm$ 0.018 & \textbf{0.023 $\pm$ 0.026} & 0.130 $\pm$ 0.064 & 0.157 $\pm$ 0.083 & 0.030 $\pm$ 0.008 & 0.014 $\pm$ 0.004 & 0.014 $\pm$ 0.022 & 0.148 $\pm$ 0.034 & 0.032 $\pm$ 0.003 & \textbf{4.417} \\
\bottomrule
\end{tabular}
}
\label{tab:re_GD}
\end{table}

%% file: tables/re_IGD.tex
\begin{table}[!htbp]
\centering
\caption{IGD results on RE ($\downarrow$). Bold $=$ best per column. Each algorithm is run for $5$ seeds and evaluated on $256$ designs.}
\resizebox{\linewidth}{!}{%
\begin{tabular}{lccccccccccccc}
\toprule
Method & RE21 & RE22 & RE23 & RE24 & RE25 & RE31 & RE32 & RE33 & RE34 & RE35 & RE36 & RE37 & Avg. Rank \\
\midrule
D(best) & 0.108 & 0.002 & 0.046 & 0.063 & 0.018 & 0.402 & 0.064 & 0.058 & 0.070 & 0.084 & 0.435 & 0.121 & -- \\
MultiHead-Vallina & \textbf{0.002 $\pm$ 0.000} & 0.067 $\pm$ 0.011 & 0.011 $\pm$ 0.004 & 0.361 $\pm$ 0.331 & 0.145 $\pm$ 0.125 & 0.237 $\pm$ 0.086 & 0.254 $\pm$ 0.123 & 0.060 $\pm$ 0.025 & \textbf{0.023 $\pm$ 0.002} & 0.140 $\pm$ 0.100 & 0.035 $\pm$ 0.013 & 0.040 $\pm$ 0.001 & 5.917 \\
MultiHead-PcGrad & 0.003 $\pm$ 0.000 & 0.138 $\pm$ 0.084 & 0.081 $\pm$ 0.101 & 0.127 $\pm$ 0.111 & 0.045 $\pm$ 0.093 & 0.243 $\pm$ 0.078 & 0.312 $\pm$ 0.006 & 0.048 $\pm$ 0.010 & 0.028 $\pm$ 0.005 & 0.138 $\pm$ 0.129 & 0.052 $\pm$ 0.026 & 0.057 $\pm$ 0.007 & 7.167 \\
MultiHead-GradNorm & 0.121 $\pm$ 0.153 & 0.150 $\pm$ 0.054 & 0.566 $\pm$ 0.030 & 0.226 $\pm$ 0.042 & 0.072 $\pm$ 0.131 & 0.182 $\pm$ 0.038 & 0.254 $\pm$ 0.130 & 0.509 $\pm$ 0.087 & 0.215 $\pm$ 0.206 & 0.585 $\pm$ 0.359 & 0.237 $\pm$ 0.152 & 0.197 $\pm$ 0.176 & 10.333 \\
MultipleModels-Vallina & \textbf{0.002 $\pm$ 0.000} & 0.100 $\pm$ 0.013 & \textbf{0.004 $\pm$ 0.002} & 0.024 $\pm$ 0.030 & 0.001 $\pm$ 0.001 & 0.283 $\pm$ 0.091 & 0.308 $\pm$ 0.008 & 0.094 $\pm$ 0.045 & 0.024 $\pm$ 0.001 & 0.136 $\pm$ 0.110 & 0.042 $\pm$ 0.026 & 0.040 $\pm$ 0.002 & 5.083 \\
MultipleModels-COM & 0.056 $\pm$ 0.002 & 0.055 $\pm$ 0.043 & 0.397 $\pm$ 0.012 & 0.038 $\pm$ 0.026 & 0.016 $\pm$ 0.002 & 0.345 $\pm$ 0.055 & \textbf{0.054 $\pm$ 0.042} & 0.077 $\pm$ 0.063 & 0.044 $\pm$ 0.003 & 0.027 $\pm$ 0.002 & 0.079 $\pm$ 0.011 & 0.059 $\pm$ 0.003 & 7.167 \\
MultipleModels-IOM & 0.003 $\pm$ 0.000 & 0.038 $\pm$ 0.041 & 0.187 $\pm$ 0.197 & 0.013 $\pm$ 0.026 & 0.002 $\pm$ 0.001 & \textbf{0.173 $\pm$ 0.020} & 0.066 $\pm$ 0.028 & \textbf{0.030 $\pm$ 0.003} & 0.024 $\pm$ 0.002 & \textbf{0.020 $\pm$ 0.003} & 0.022 $\pm$ 0.008 & \textbf{0.039 $\pm$ 0.001} & \textbf{3.167} \\
MultipleModels-ICT & \textbf{0.002 $\pm$ 0.000} & 0.115 $\pm$ 0.054 & 0.042 $\pm$ 0.010 & 0.020 $\pm$ 0.034 & 0.009 $\pm$ 0.006 & 0.218 $\pm$ 0.070 & 0.305 $\pm$ 0.002 & 0.145 $\pm$ 0.096 & 0.031 $\pm$ 0.003 & 0.152 $\pm$ 0.010 & \textbf{0.020 $\pm$ 0.004} & 0.040 $\pm$ 0.002 & 6.083 \\
MultipleModels-RoMA & 0.022 $\pm$ 0.002 & 0.193 $\pm$ 0.300 & 0.092 $\pm$ 0.173 & 0.391 $\pm$ 0.512 & 0.004 $\pm$ 0.004 & 0.325 $\pm$ 0.047 & 0.055 $\pm$ 0.009 & 0.038 $\pm$ 0.008 & 0.086 $\pm$ 0.003 & 0.029 $\pm$ 0.004 & 0.199 $\pm$ 0.067 & 0.052 $\pm$ 0.002 & 7.583 \\
MultipleModels-Trimentoring & 0.003 $\pm$ 0.001 & 0.147 $\pm$ 0.046 & 0.007 $\pm$ 0.002 & 0.144 $\pm$ 0.176 & 0.015 $\pm$ 0.001 & 0.276 $\pm$ 0.049 & 0.121 $\pm$ 0.022 & 0.172 $\pm$ 0.019 & 0.052 $\pm$ 0.014 & 0.028 $\pm$ 0.006 & 0.469 $\pm$ 0.340 & 0.041 $\pm$ 0.002 & 7.250 \\
ParetoFlow & 0.115 $\pm$ 0.024 & 0.010 $\pm$ 0.009 & 0.042 $\pm$ 0.002 & \textbf{0.007 $\pm$ 0.005} & 0.006 $\pm$ 0.002 & 0.395 $\pm$ 0.033 & 0.062 $\pm$ 0.001 & 0.042 $\pm$ 0.007 & 0.066 $\pm$ 0.004 & 0.062 $\pm$ 0.018 & 0.240 $\pm$ 0.022 & 0.173 $\pm$ 0.026 & 6.917 \\
PreferenceGuidedDiffusion & 0.041 $\pm$ 0.005 & \textbf{0.001 $\pm$ 0.000} & 0.020 $\pm$ 0.002 & \textbf{0.007 $\pm$ 0.003} & \textbf{0.000 $\pm$ 0.000} & 0.287 $\pm$ 0.004 & 0.072 $\pm$ 0.016 & 0.031 $\pm$ 0.008 & 0.079 $\pm$ 0.009 & 0.268 $\pm$ 0.043 & 0.198 $\pm$ 0.032 & 0.093 $\pm$ 0.005 & 5.917 \\
\midrule
ParetoTransport & 0.006 $\pm$ 0.000 & 0.018 $\pm$ 0.016 & 0.009 $\pm$ 0.002 & 0.010 $\pm$ 0.003 & \textbf{0.000 $\pm$ 0.000} & 0.323 $\pm$ 0.006 & 0.079 $\pm$ 0.035 & 0.068 $\pm$ 0.017 & 0.029 $\pm$ 0.002 & 0.052 $\pm$ 0.044 & 0.058 $\pm$ 0.038 & 0.060 $\pm$ 0.003 & 5.417 \\
\bottomrule
\end{tabular}
}
\label{tab:re_IGD}
\end{table}

%% file: tables/re_W2.tex
\begin{table}[!htbp]
\centering
\caption{$W_2$ results on RE ($\downarrow$). Bold $=$ best per column. Each algorithm is run for $5$ seeds and evaluated on $256$ designs.}
\resizebox{\linewidth}{!}{%
\begin{tabular}{lccccccccccccc}
\toprule
Method & RE21 & RE22 & RE23 & RE24 & RE25 & RE31 & RE32 & RE33 & RE34 & RE35 & RE36 & RE37 & Avg. Rank \\
\midrule
D(best) & 0.120 & 0.032 & 0.350 & 0.063 & 0.064 & 0.452 & 0.296 & 0.172 & 0.252 & 0.191 & 0.479 & 0.237 & -- \\
MultiHead-Vallina & 0.009 $\pm$ 0.004 & 272.533 $\pm$ 557.568 & 0.323 $\pm$ 0.079 & 0.479 $\pm$ 0.366 & 0.465 $\pm$ 0.054 & 0.348 $\pm$ 0.064 & 0.266 $\pm$ 0.103 & 0.259 $\pm$ 0.041 & 0.148 $\pm$ 0.010 & 0.201 $\pm$ 0.101 & 0.210 $\pm$ 0.020 & 0.108 $\pm$ 0.006 & 5.833 \\
MultiHead-PcGrad & 0.014 $\pm$ 0.005 & 96.952 $\pm$ 169.051 & 0.348 $\pm$ 0.126 & 0.267 $\pm$ 0.320 & 0.263 $\pm$ 0.261 & 0.386 $\pm$ 0.047 & 0.445 $\pm$ 0.293 & 0.237 $\pm$ 0.020 & 0.174 $\pm$ 0.014 & 0.325 $\pm$ 0.164 & 0.222 $\pm$ 0.030 & 0.119 $\pm$ 0.014 & 6.917 \\
MultiHead-GradNorm & 0.226 $\pm$ 0.172 & 1.045 $\pm$ 1.962 & 0.571 $\pm$ 0.028 & 0.258 $\pm$ 0.060 & 0.236 $\pm$ 0.172 & 0.401 $\pm$ 0.064 & 0.359 $\pm$ 0.162 & 0.656 $\pm$ 0.175 & 0.332 $\pm$ 0.119 & 0.690 $\pm$ 0.388 & 0.444 $\pm$ 0.071 & 0.275 $\pm$ 0.161 & 9.000 \\
MultipleModels-Vallina & \textbf{0.007 $\pm$ 0.002} & 0.233 $\pm$ 0.159 & 0.212 $\pm$ 0.064 & 0.455 $\pm$ 0.067 & 0.246 $\pm$ 0.150 & 0.401 $\pm$ 0.051 & 0.427 $\pm$ 0.114 & 0.293 $\pm$ 0.043 & 0.151 $\pm$ 0.011 & 0.244 $\pm$ 0.127 & 0.217 $\pm$ 0.023 & \textbf{0.102 $\pm$ 0.001} & 5.750 \\
MultipleModels-COM & 0.069 $\pm$ 0.003 & 6.319 $\pm$ 13.674 & 0.404 $\pm$ 0.004 & \textbf{0.055 $\pm$ 0.018} & 0.224 $\pm$ 0.070 & 0.440 $\pm$ 0.017 & \textbf{0.174 $\pm$ 0.074} & 0.197 $\pm$ 0.062 & 0.219 $\pm$ 0.012 & \textbf{0.115 $\pm$ 0.011} & 0.229 $\pm$ 0.017 & 0.163 $\pm$ 0.007 & 5.750 \\
MultipleModels-IOM & 0.015 $\pm$ 0.007 & 3.876 $\pm$ 4.104 & 0.292 $\pm$ 0.160 & 0.134 $\pm$ 0.052 & 4.673 $\pm$ 10.441 & \textbf{0.311 $\pm$ 0.074} & 0.402 $\pm$ 0.056 & 0.172 $\pm$ 0.034 & 0.165 $\pm$ 0.016 & 0.209 $\pm$ 0.010 & \textbf{0.142 $\pm$ 0.022} & 0.124 $\pm$ 0.003 & 5.083 \\
MultipleModels-ICT & \textbf{0.007 $\pm$ 0.001} & 17.261 $\pm$ 15.678 & 0.337 $\pm$ 0.017 & 0.451 $\pm$ 0.032 & 0.396 $\pm$ 0.030 & 0.371 $\pm$ 0.059 & 0.314 $\pm$ 0.000 & 0.267 $\pm$ 0.106 & 0.146 $\pm$ 0.012 & 0.266 $\pm$ 0.044 & 0.179 $\pm$ 0.009 & 0.107 $\pm$ 0.004 & 5.750 \\
MultipleModels-RoMA & 0.064 $\pm$ 0.013 & 198.919 $\pm$ 411.081 & 0.309 $\pm$ 0.169 & 0.667 $\pm$ 0.262 & 0.287 $\pm$ 0.213 & 0.424 $\pm$ 0.018 & 0.456 $\pm$ 0.068 & 0.258 $\pm$ 0.051 & 0.305 $\pm$ 0.007 & 0.194 $\pm$ 0.059 & 0.670 $\pm$ 0.061 & 0.191 $\pm$ 0.007 & 8.667 \\
MultipleModels-Trimentoring & \textbf{0.007 $\pm$ 0.001} & 0.190 $\pm$ 0.012 & 0.164 $\pm$ 0.021 & 0.379 $\pm$ 0.218 & 0.112 $\pm$ 0.003 & 0.447 $\pm$ 0.018 & 0.512 $\pm$ 0.127 & 0.304 $\pm$ 0.009 & 0.190 $\pm$ 0.040 & 0.149 $\pm$ 0.020 & 0.729 $\pm$ 0.333 & \textbf{0.102 $\pm$ 0.010} & 6.083 \\
ParetoFlow & 0.187 $\pm$ 0.004 & \textbf{0.034 $\pm$ 0.015} & 0.360 $\pm$ 0.002 & 0.137 $\pm$ 0.004 & 0.076 $\pm$ 0.006 & 0.433 $\pm$ 0.006 & 0.313 $\pm$ 0.006 & \textbf{0.129 $\pm$ 0.004} & 0.191 $\pm$ 0.035 & 0.233 $\pm$ 0.009 & 0.413 $\pm$ 0.031 & 0.327 $\pm$ 0.040 & 6.500 \\
PreferenceGuidedDiffusion & 0.147 $\pm$ 0.034 & -- & 0.228 $\pm$ 0.067 & 0.483 $\pm$ 0.061 & \textbf{0.046 $\pm$ 0.014} & 0.466 $\pm$ 0.034 & 0.271 $\pm$ 0.009 & 0.226 $\pm$ 0.037 & 0.273 $\pm$ 0.029 & 0.460 $\pm$ 0.044 & 0.533 $\pm$ 0.012 & 0.248 $\pm$ 0.016 & 7.909 \\
\midrule
ParetoTransport & 0.074 $\pm$ 0.011 & 0.119 $\pm$ 0.012 & \textbf{0.159 $\pm$ 0.053} & 0.094 $\pm$ 0.019 & 0.049 $\pm$ 0.041 & 0.452 $\pm$ 0.053 & 0.301 $\pm$ 0.044 & 0.217 $\pm$ 0.024 & \textbf{0.117 $\pm$ 0.027} & 0.171 $\pm$ 0.072 & 0.231 $\pm$ 0.045 & 0.144 $\pm$ 0.005 & \textbf{4.417} \\
\bottomrule
\end{tabular}
}
\label{tab:re_W2}
\end{table}

%% file: tables/ablations_components_zdt.tex
\begin{table}[!htbp]
\caption{Ablation of the Wasserstein matching term $W_2$ and directional improvement term $L_n$ on ZDT. Bold $=$ best per column. Each experiment is run for $5$ seeds and evaluated on $256$ designs.}
\label{tab:zdt-components}
\resizebox{\linewidth}{!}{%
\begin{tabular}{cclccccc}
\toprule
$W_2$ & $L_n$ & Metric & ZDT1 & ZDT2 & ZDT3 & ZDT4 & ZDT6 \\
\midrule
 & \multirow{4}{*}{\checkmark} & HV ($\uparrow$) & 4.750 $\pm$ 0.011 & 5.499 $\pm$ 0.019 & \textbf{5.611} $\pm$ 0.198 & 5.479 $\pm$ 0.009 & 4.816 $\pm$ 0.008 \\
 &  & GD ($\downarrow$) & 0.015 $\pm$ 0.003 & \textbf{0.033} $\pm$ 0.036 & \textbf{0.183} $\pm$ 0.026 & 0.443 $\pm$ 0.012 & 0.016 $\pm$ 0.016 \\
 &  & IGD ($\downarrow$) & 0.116 $\pm$ 0.060 & 0.368 $\pm$ 0.142 & 0.164 $\pm$ 0.113 & 0.346 $\pm$ 0.085 & 0.019 $\pm$ 0.002 \\
 &  & W2 ($\downarrow$) & 0.350 $\pm$ 0.026 & 0.507 $\pm$ 0.010 & 0.357 $\pm$ 0.138 & 0.642 $\pm$ 0.012 & 0.156 $\pm$ 0.060 \\
 \midrule
\multirow{4}{*}{\checkmark} &  & HV ($\uparrow$) & 4.750 $\pm$ 0.018 & 5.617 $\pm$ 0.040 & 5.508 $\pm$ 0.040 & \textbf{5.491} $\pm$ 0.006 & 4.755 $\pm$ 0.086 \\
 &  & GD ($\downarrow$) & 0.130 $\pm$ 0.009 & 0.169 $\pm$ 0.016 & 0.303 $\pm$ 0.021 & \textbf{0.382} $\pm$ 0.010 & 0.054 $\pm$ 0.040 \\
 &  & IGD ($\downarrow$) & 0.178 $\pm$ 0.006 & 0.155 $\pm$ 0.023 & 0.229 $\pm$ 0.018 & 0.264 $\pm$ 0.048 & 0.045 $\pm$ 0.025 \\
 &  & W2 ($\downarrow$) & 0.225 $\pm$ 0.005 & 0.242 $\pm$ 0.023 & 0.357 $\pm$ 0.022 & 0.545 $\pm$ 0.037 & 0.184 $\pm$ 0.015 \\
 \midrule
\multirow{4}{*}{\checkmark} & \multirow{4}{*}{\checkmark} & HV ($\uparrow$) & \textbf{4.833} $\pm$ 0.023 & \textbf{5.660} $\pm$ 0.039 & 5.577 $\pm$ 0.075 & 5.485 $\pm$ 0.000 & \textbf{4.825} $\pm$ 0.001 \\
 &  & GD ($\downarrow$) & \textbf{0.013} $\pm$ 0.004 & 0.039 $\pm$ 0.013 & 0.250 $\pm$ 0.028 & 0.404 $\pm$ 0.018 & \textbf{0.013} $\pm$ 0.003 \\
 &  & IGD ($\downarrow$) & \textbf{0.032} $\pm$ 0.053 & \textbf{0.023} $\pm$ 0.011 & \textbf{0.150} $\pm$ 0.023 & \textbf{0.261} $\pm$ 0.023 & \textbf{0.016} $\pm$ 0.001 \\
 &  & W2 ($\downarrow$) & \textbf{0.181} $\pm$ 0.108 & \textbf{0.131} $\pm$ 0.024 & \textbf{0.305} $\pm$ 0.025 & \textbf{0.542} $\pm$ 0.044 & \textbf{0.147} $\pm$ 0.026 \\
\bottomrule
\end{tabular}
}
\end{table}

%% file: tables/ablations_components_dtlz.tex
\begin{table}[!htbp]
\caption{Ablation of the Wasserstein matching term $W_2$ and directional improvement term $L_n$ on DTLZ. Bold $=$ best per column. Each experiment is run for $5$ seeds and evaluated on $256$ designs.}
\label{tab:dtlz-components}
\resizebox{\linewidth}{!}{%
\begin{tabular}{cclccccccc}
\toprule
$W_2$ & $L_n$ & Metric & DTLZ1 & DTLZ2 & DTLZ3 & DTLZ4 & DTLZ5 & DTLZ6 & DTLZ7 \\
\midrule
 & \multirow{4}{*}{\checkmark} & HV ($\uparrow$) & 10.645 $\pm$ 0.001 & 12.440 $\pm$ 0.001 & \textbf{9.898} $\pm$ 0.001 & 17.324 $\pm$ 0.250 & 10.739 $\pm$ 0.006 & \textbf{11.004} $\pm$ 0.008 & 10.738 $\pm$ 0.009 \\
 &  & GD ($\downarrow$) & 0.448 $\pm$ 0.024 & 0.027 $\pm$ 0.002 & \textbf{0.394} $\pm$ 0.013 & \textbf{0.050} $\pm$ 0.010 & \textbf{0.026} $\pm$ 0.004 & \textbf{0.514} $\pm$ 0.008 & 0.095 $\pm$ 0.026 \\
 &  & IGD ($\downarrow$) & \textbf{0.062} $\pm$ 0.012 & 0.038 $\pm$ 0.001 & \textbf{0.107} $\pm$ 0.009 & 0.238 $\pm$ 0.045 & 0.015 $\pm$ 0.003 & 0.295 $\pm$ 0.003 & 0.112 $\pm$ 0.012 \\
 &  & W2 ($\downarrow$) & 0.449 $\pm$ 0.025 & 0.121 $\pm$ 0.018 & \textbf{0.394} $\pm$ 0.013 & 0.410 $\pm$ 0.008 & 0.102 $\pm$ 0.011 & \textbf{0.520} $\pm$ 0.009 & 0.281 $\pm$ 0.038 \\
 \midrule
\multirow{4}{*}{\checkmark} &  & HV ($\uparrow$) & 10.645 $\pm$ 0.001 & 12.442 $\pm$ 0.001 & 9.891 $\pm$ 0.004 & \textbf{17.619} $\pm$ 0.023 & 10.730 $\pm$ 0.002 & 10.896 $\pm$ 0.037 & 10.577 $\pm$ 0.052 \\
 &  & GD ($\downarrow$) & 0.419 $\pm$ 0.005 & 0.044 $\pm$ 0.003 & 0.533 $\pm$ 0.008 & 0.124 $\pm$ 0.014 & 0.081 $\pm$ 0.013 & 0.703 $\pm$ 0.007 & 0.252 $\pm$ 0.041 \\
 &  & IGD ($\downarrow$) & 0.081 $\pm$ 0.031 & 0.036 $\pm$ 0.001 & 0.153 $\pm$ 0.047 & 0.158 $\pm$ 0.012 & 0.020 $\pm$ 0.001 & \textbf{0.287} $\pm$ 0.001 & 0.149 $\pm$ 0.019 \\
 &  & W2 ($\downarrow$) & 0.420 $\pm$ 0.005 & 0.065 $\pm$ 0.003 & 0.533 $\pm$ 0.008 & 0.333 $\pm$ 0.013 & 0.107 $\pm$ 0.019 & 0.712 $\pm$ 0.006 & 0.365 $\pm$ 0.028 \\
 \midrule
\multirow{4}{*}{\checkmark} & \multirow{4}{*}{\checkmark} & HV ($\uparrow$) & \textbf{10.646} $\pm$ 0.000 & \textbf{12.448} $\pm$ 0.000 & 9.896 $\pm$ 0.002 & 17.611 $\pm$ 0.006 & \textbf{10.754} $\pm$ 0.001 & 10.966 $\pm$ 0.042 & \textbf{10.759} $\pm$ 0.057 \\
 &  & GD ($\downarrow$) & \textbf{0.397} $\pm$ 0.021 & \textbf{0.023} $\pm$ 0.003 & 0.451 $\pm$ 0.030 & 0.063 $\pm$ 0.016 & 0.063 $\pm$ 0.008 & 0.576 $\pm$ 0.030 & \textbf{0.063} $\pm$ 0.018 \\
 &  & IGD ($\downarrow$) & 0.087 $\pm$ 0.020 & \textbf{0.021} $\pm$ 0.001 & 0.120 $\pm$ 0.026 & \textbf{0.152} $\pm$ 0.005 & \textbf{0.007} $\pm$ 0.001 & 0.290 $\pm$ 0.006 & \textbf{0.087} $\pm$ 0.062 \\
 &  & W2 ($\downarrow$) & \textbf{0.398} $\pm$ 0.021 & \textbf{0.040} $\pm$ 0.002 & 0.452 $\pm$ 0.030 & \textbf{0.298} $\pm$ 0.010 & \textbf{0.087} $\pm$ 0.019 & 0.582 $\pm$ 0.030 & \textbf{0.216} $\pm$ 0.080 \\
\bottomrule
\end{tabular}
}
\end{table}

%% file: tables/ablation_gamma/zdt_gamma_hv.tex
% requires \usepackage{booktabs,graphicx}
\begin{table}
\caption{Ablation of the transport weight $\gamma$ on ZDT tasks. HV is computed over a population of $256$ candidate designs and averaged over $5$ random seeds.}
\label{tab:zdt-gamma-hv}
\resizebox{\linewidth}{!}{%
\begin{tabular}{lccccc}
\toprule
$\gamma$ & ZDT1 & ZDT2 & ZDT3 & ZDT4 & ZDT6 \\
\midrule
0 & 4.750 $\pm$ 0.018 & 5.617 $\pm$ 0.040 & 5.508 $\pm$ 0.040 & \textbf{5.491 $\pm$ 0.006} & 4.755 $\pm$ 0.086 \\
0.1 & 4.761 $\pm$ 0.015 & 5.656 $\pm$ 0.025 & 5.550 $\pm$ 0.056 & 5.486 $\pm$ 0.005 & 4.822 $\pm$ 0.007 \\
0.2 & 4.796 $\pm$ 0.018 & 5.656 $\pm$ 0.026 & 5.528 $\pm$ 0.020 & 5.487 $\pm$ 0.005 & 4.822 $\pm$ 0.007 \\
0.5 & \textbf{4.833 $\pm$ 0.023} & \textbf{5.660 $\pm$ 0.039} & 5.577 $\pm$ 0.075 & 5.485 $\pm$ 0.000 & \textbf{4.825 $\pm$ 0.001} \\
1 & 4.831 $\pm$ 0.028 & 5.658 $\pm$ 0.038 & 5.590 $\pm$ 0.115 & 5.487 $\pm$ 0.005 & 4.822 $\pm$ 0.007 \\
2 & 4.805 $\pm$ 0.040 & 5.641 $\pm$ 0.049 & 5.578 $\pm$ 0.139 & 5.488 $\pm$ 0.006 & 4.822 $\pm$ 0.006 \\
5 & 4.770 $\pm$ 0.022 & 5.550 $\pm$ 0.035 & 5.593 $\pm$ 0.144 & 5.486 $\pm$ 0.005 & \textbf{4.825 $\pm$ 0.001} \\
10 & 4.761 $\pm$ 0.015 & 5.497 $\pm$ 0.022 & \textbf{5.607 $\pm$ 0.180} & 5.484 $\pm$ 0.002 & 4.823 $\pm$ 0.006 \\
\bottomrule
\end{tabular}%
}
\end{table}

%% file: tables/ablation_gamma/zdt_gamma_gd.tex
% requires \usepackage{booktabs,graphicx}
\begin{table}
\caption{Ablation of the transport weight $\gamma$ on ZDT tasks. GD is computed over a population of $256$ candidate designs and averaged over $5$ random seeds.}
\label{tab:zdt-gamma-gd}
\resizebox{\linewidth}{!}{%
\begin{tabular}{lccccc}
\toprule
$\gamma$ & ZDT1 & ZDT2 & ZDT3 & ZDT4 & ZDT6 \\
\midrule
0 & 0.130 $\pm$ 0.009 & 0.169 $\pm$ 0.016 & 0.303 $\pm$ 0.021 & \textbf{0.382 $\pm$ 0.010} & 0.054 $\pm$ 0.040 \\
0.1 & 0.109 $\pm$ 0.011 & 0.115 $\pm$ 0.017 & 0.312 $\pm$ 0.025 & 0.400 $\pm$ 0.028 & \textbf{0.010 $\pm$ 0.002} \\
0.2 & 0.064 $\pm$ 0.018 & 0.075 $\pm$ 0.019 & 0.302 $\pm$ 0.027 & 0.401 $\pm$ 0.027 & 0.012 $\pm$ 0.003 \\
0.5 & 0.013 $\pm$ 0.004 & 0.039 $\pm$ 0.013 & 0.250 $\pm$ 0.028 & 0.404 $\pm$ 0.018 & 0.013 $\pm$ 0.003 \\
1 & \textbf{0.008 $\pm$ 0.004} & \textbf{0.024 $\pm$ 0.011} & 0.185 $\pm$ 0.030 & 0.419 $\pm$ 0.018 & 0.013 $\pm$ 0.003 \\
2 & 0.010 $\pm$ 0.001 & 0.028 $\pm$ 0.027 & 0.185 $\pm$ 0.026 & 0.415 $\pm$ 0.016 & 0.014 $\pm$ 0.003 \\
5 & 0.014 $\pm$ 0.003 & 0.027 $\pm$ 0.030 & 0.184 $\pm$ 0.030 & 0.419 $\pm$ 0.016 & 0.013 $\pm$ 0.004 \\
10 & 0.015 $\pm$ 0.003 & 0.030 $\pm$ 0.032 & \textbf{0.183 $\pm$ 0.028} & 0.428 $\pm$ 0.016 & 0.013 $\pm$ 0.004 \\
\bottomrule
\end{tabular}%
}
\end{table}

%% file: tables/ablation_gamma/zdt_gamma_igd.tex
% requires \usepackage{booktabs,graphicx}
\begin{table}
\caption{Ablation of the transport weight $\gamma$ on ZDT tasks. IGD is computed over a population of $256$ candidate designs and averaged over $5$ random seeds.}
\label{tab:zdt-gamma-igd}
\resizebox{\linewidth}{!}{%
\begin{tabular}{lccccc}
\toprule
$\gamma$ & ZDT1 & ZDT2 & ZDT3 & ZDT4 & ZDT6 \\
\midrule
0 & 0.178 $\pm$ 0.006 & 0.155 $\pm$ 0.023 & 0.229 $\pm$ 0.018 & 0.264 $\pm$ 0.048 & 0.045 $\pm$ 0.025 \\
0.1 & 0.151 $\pm$ 0.008 & 0.057 $\pm$ 0.033 & 0.205 $\pm$ 0.028 & \textbf{0.230 $\pm$ 0.036} & 0.016 $\pm$ 0.002 \\
0.2 & 0.071 $\pm$ 0.016 & 0.025 $\pm$ 0.008 & 0.192 $\pm$ 0.019 & 0.240 $\pm$ 0.027 & 0.016 $\pm$ 0.002 \\
0.5 & \textbf{0.032 $\pm$ 0.053} & \textbf{0.023 $\pm$ 0.011} & 0.150 $\pm$ 0.023 & 0.261 $\pm$ 0.023 & 0.016 $\pm$ 0.001 \\
1 & 0.044 $\pm$ 0.066 & \textbf{0.023 $\pm$ 0.014} & \textbf{0.143 $\pm$ 0.057} & 0.274 $\pm$ 0.033 & 0.016 $\pm$ 0.002 \\
2 & 0.066 $\pm$ 0.090 & 0.032 $\pm$ 0.019 & 0.161 $\pm$ 0.092 & 0.288 $\pm$ 0.020 & 0.016 $\pm$ 0.002 \\
5 & 0.097 $\pm$ 0.070 & 0.125 $\pm$ 0.118 & 0.167 $\pm$ 0.098 & 0.324 $\pm$ 0.058 & \textbf{0.015 $\pm$ 0.001} \\
10 & 0.098 $\pm$ 0.067 & 0.369 $\pm$ 0.142 & 0.161 $\pm$ 0.109 & 0.342 $\pm$ 0.059 & \textbf{0.015 $\pm$ 0.001} \\
\bottomrule
\end{tabular}%
}
\end{table}

%% file: tables/ablation_gamma/zdt_gamma_w2.tex
% requires \usepackage{booktabs,graphicx}
\begin{table}
\caption{Ablation of the transport weight $\gamma$ on ZDT tasks. $W_2$ is computed over a population of $256$ candidate designs and averaged over $5$ random seeds.}
\label{tab:zdt-gamma-w2}
\resizebox{\linewidth}{!}{%
\begin{tabular}{lccccc}
\toprule
$\gamma$ & ZDT1 & ZDT2 & ZDT3 & ZDT4 & ZDT6 \\
\midrule
0 & 0.225 $\pm$ 0.005 & 0.242 $\pm$ 0.023 & 0.357 $\pm$ 0.022 & 0.545 $\pm$ 0.037 & 0.184 $\pm$ 0.015 \\
0.1 & 0.213 $\pm$ 0.010 & 0.216 $\pm$ 0.025 & 0.368 $\pm$ 0.027 & 0.544 $\pm$ 0.039 & 0.188 $\pm$ 0.007 \\
0.2 & \textbf{0.148 $\pm$ 0.018} & 0.143 $\pm$ 0.024 & 0.347 $\pm$ 0.029 & \textbf{0.538 $\pm$ 0.043} & 0.170 $\pm$ 0.020 \\
0.5 & 0.181 $\pm$ 0.108 & \textbf{0.131 $\pm$ 0.024} & \textbf{0.305 $\pm$ 0.025} & 0.542 $\pm$ 0.044 & 0.147 $\pm$ 0.026 \\
1 & 0.218 $\pm$ 0.130 & 0.139 $\pm$ 0.077 & 0.346 $\pm$ 0.092 & 0.554 $\pm$ 0.038 & 0.151 $\pm$ 0.034 \\
2 & 0.189 $\pm$ 0.151 & 0.352 $\pm$ 0.038 & 0.379 $\pm$ 0.116 & 0.570 $\pm$ 0.013 & 0.147 $\pm$ 0.048 \\
5 & 0.304 $\pm$ 0.079 & 0.468 $\pm$ 0.035 & 0.375 $\pm$ 0.112 & 0.605 $\pm$ 0.012 & 0.133 $\pm$ 0.028 \\
10 & 0.330 $\pm$ 0.049 & 0.506 $\pm$ 0.008 & 0.368 $\pm$ 0.123 & 0.623 $\pm$ 0.005 & \textbf{0.069 $\pm$ 0.023} \\
\bottomrule
\end{tabular}%
}
\end{table}

%% file: tables/ablation_gamma/dtlz_gamma_hv.tex
% requires \usepackage{booktabs,graphicx}
\begin{table}
\caption{Ablation of the transport weight $\gamma$ on DTLZ tasks. HV is computed over a population of $256$ candidate designs and averaged over $5$ random seeds.}
\label{tab:dtlz-gamma-hv}
\resizebox{\linewidth}{!}{%
\begin{tabular}{lccccccc}
\toprule
$\gamma$ & DTLZ1 & DTLZ2 & DTLZ3 & DTLZ4 & DTLZ5 & DTLZ6 & DTLZ7 \\
\midrule
0 & 10.645 $\pm$ 0.001 & 12.442 $\pm$ 0.001 & 9.891 $\pm$ 0.004 & 17.619 $\pm$ 0.023 & 10.730 $\pm$ 0.002 & 10.896 $\pm$ 0.037 & 10.577 $\pm$ 0.052 \\
0.1 & 10.645 $\pm$ 0.001 & 12.445 $\pm$ 0.000 & 9.891 $\pm$ 0.005 & 17.621 $\pm$ 0.023 & 10.738 $\pm$ 0.002 & 10.909 $\pm$ 0.043 & 10.716 $\pm$ 0.056 \\
0.2 & 10.646 $\pm$ 0.000 & 12.447 $\pm$ 0.000 & 9.891 $\pm$ 0.005 & \textbf{17.622 $\pm$ 0.027} & 10.748 $\pm$ 0.002 & 10.917 $\pm$ 0.035 & 10.732 $\pm$ 0.049 \\
0.5 & 10.646 $\pm$ 0.000 & \textbf{12.448 $\pm$ 0.000} & 9.896 $\pm$ 0.002 & 17.611 $\pm$ 0.006 & 10.754 $\pm$ 0.001 & 10.966 $\pm$ 0.042 & \textbf{10.759 $\pm$ 0.057} \\
1 & \textbf{10.647 $\pm$ 0.000} & \textbf{12.448 $\pm$ 0.000} & \textbf{9.898 $\pm$ 0.001} & 17.611 $\pm$ 0.011 & \textbf{10.757 $\pm$ 0.000} & 11.047 $\pm$ 0.012 & 10.748 $\pm$ 0.063 \\
2 & 10.646 $\pm$ 0.001 & 12.447 $\pm$ 0.000 & 9.897 $\pm$ 0.002 & 17.534 $\pm$ 0.032 & 10.756 $\pm$ 0.000 & \textbf{11.062 $\pm$ 0.004} & 10.724 $\pm$ 0.068 \\
5 & \textbf{10.647 $\pm$ 0.000} & 12.442 $\pm$ 0.000 & \textbf{9.898 $\pm$ 0.000} & 17.231 $\pm$ 0.240 & 10.742 $\pm$ 0.004 & 11.013 $\pm$ 0.007 & 10.730 $\pm$ 0.046 \\
10 & 10.646 $\pm$ 0.000 & 12.441 $\pm$ 0.001 & \textbf{9.898 $\pm$ 0.001} & 17.215 $\pm$ 0.261 & 10.738 $\pm$ 0.006 & 11.010 $\pm$ 0.006 & 10.744 $\pm$ 0.009 \\
\bottomrule
\end{tabular}%
}
\end{table}

%% file: tables/ablation_gamma/dtlz_gamma_igd.tex
% requires \usepackage{booktabs,graphicx}
\begin{table}
\caption{Ablation of the transport weight $\gamma$ on DTLZ tasks. IGD is computed over a population of $256$ candidate designs and averaged over $5$ random seeds.}
\label{tab:dtlz-gamma-igd}
\resizebox{\linewidth}{!}{%
\begin{tabular}{lccccccc}
\toprule
$\gamma$ & DTLZ1 & DTLZ2 & DTLZ3 & DTLZ4 & DTLZ5 & DTLZ6 & DTLZ7 \\
\midrule
0 & 0.081 $\pm$ 0.031 & 0.036 $\pm$ 0.001 & 0.153 $\pm$ 0.047 & 0.158 $\pm$ 0.012 & 0.020 $\pm$ 0.001 & 0.287 $\pm$ 0.001 & 0.149 $\pm$ 0.019 \\
0.1 & 0.078 $\pm$ 0.017 & 0.028 $\pm$ 0.001 & 0.170 $\pm$ 0.033 & 0.156 $\pm$ 0.009 & 0.016 $\pm$ 0.002 & 0.292 $\pm$ 0.010 & 0.096 $\pm$ 0.019 \\
0.2 & 0.071 $\pm$ 0.010 & 0.024 $\pm$ 0.001 & 0.136 $\pm$ 0.055 & 0.158 $\pm$ 0.007 & 0.010 $\pm$ 0.001 & 0.288 $\pm$ 0.001 & 0.102 $\pm$ 0.049 \\
0.5 & 0.087 $\pm$ 0.020 & 0.021 $\pm$ 0.001 & 0.120 $\pm$ 0.026 & \textbf{0.152 $\pm$ 0.005} & 0.007 $\pm$ 0.001 & 0.290 $\pm$ 0.006 & \textbf{0.087 $\pm$ 0.062} \\
1 & 0.070 $\pm$ 0.006 & \textbf{0.018 $\pm$ 0.001} & \textbf{0.104 $\pm$ 0.021} & \textbf{0.152 $\pm$ 0.010} & \textbf{0.005 $\pm$ 0.000} & \textbf{0.283 $\pm$ 0.000} & 0.100 $\pm$ 0.063 \\
2 & 0.070 $\pm$ 0.009 & 0.024 $\pm$ 0.000 & 0.116 $\pm$ 0.013 & 0.200 $\pm$ 0.034 & \textbf{0.005 $\pm$ 0.001} & \textbf{0.283 $\pm$ 0.000} & 0.145 $\pm$ 0.073 \\
5 & 0.066 $\pm$ 0.008 & 0.035 $\pm$ 0.001 & 0.113 $\pm$ 0.006 & 0.231 $\pm$ 0.017 & 0.014 $\pm$ 0.002 & 0.290 $\pm$ 0.001 & 0.125 $\pm$ 0.053 \\
10 & \textbf{0.060 $\pm$ 0.011} & 0.036 $\pm$ 0.001 & 0.114 $\pm$ 0.023 & 0.247 $\pm$ 0.041 & 0.016 $\pm$ 0.002 & 0.292 $\pm$ 0.002 & 0.111 $\pm$ 0.013 \\
\bottomrule
\end{tabular}%
}
\end{table}

%% file: tables/ablation_gamma/dtlz_gamma_w2.tex
% requires \usepackage{booktabs,graphicx}
\begin{table}
\caption{Ablation of the transport weight $\gamma$ on DTLZ tasks. $W_2$ is computed over a population of $256$ candidate designs and averaged over $5$ random seeds.}
\label{tab:dtlz-gamma-w2}
\resizebox{\linewidth}{!}{%
\begin{tabular}{lccccccc}
\toprule
$\gamma$ & DTLZ1 & DTLZ2 & DTLZ3 & DTLZ4 & DTLZ5 & DTLZ6 & DTLZ7 \\
\midrule
0 & 0.420 $\pm$ 0.005 & 0.065 $\pm$ 0.003 & 0.533 $\pm$ 0.008 & 0.333 $\pm$ 0.013 & 0.107 $\pm$ 0.019 & 0.712 $\pm$ 0.006 & 0.365 $\pm$ 0.028 \\
0.1 & 0.429 $\pm$ 0.010 & 0.054 $\pm$ 0.007 & 0.509 $\pm$ 0.020 & 0.336 $\pm$ 0.021 & 0.125 $\pm$ 0.023 & 0.695 $\pm$ 0.011 & 0.262 $\pm$ 0.042 \\
0.2 & 0.415 $\pm$ 0.018 & 0.042 $\pm$ 0.003 & 0.501 $\pm$ 0.015 & 0.321 $\pm$ 0.011 & 0.111 $\pm$ 0.028 & 0.679 $\pm$ 0.010 & 0.232 $\pm$ 0.060 \\
0.5 & 0.398 $\pm$ 0.021 & 0.040 $\pm$ 0.002 & 0.452 $\pm$ 0.030 & \textbf{0.298 $\pm$ 0.010} & 0.087 $\pm$ 0.019 & 0.582 $\pm$ 0.030 & \textbf{0.216 $\pm$ 0.080} \\
1 & \textbf{0.356 $\pm$ 0.016} & \textbf{0.038 $\pm$ 0.005} & 0.397 $\pm$ 0.024 & 0.321 $\pm$ 0.018 & \textbf{0.068 $\pm$ 0.011} & 0.333 $\pm$ 0.011 & 0.223 $\pm$ 0.091 \\
2 & 0.377 $\pm$ 0.025 & 0.063 $\pm$ 0.004 & 0.388 $\pm$ 0.013 & 0.397 $\pm$ 0.032 & 0.084 $\pm$ 0.010 & \textbf{0.329 $\pm$ 0.008} & 0.271 $\pm$ 0.081 \\
5 & 0.411 $\pm$ 0.017 & 0.105 $\pm$ 0.011 & \textbf{0.387 $\pm$ 0.013} & 0.413 $\pm$ 0.010 & 0.107 $\pm$ 0.011 & 0.492 $\pm$ 0.018 & 0.242 $\pm$ 0.052 \\
10 & 0.429 $\pm$ 0.014 & 0.116 $\pm$ 0.014 & 0.388 $\pm$ 0.015 & 0.411 $\pm$ 0.008 & 0.108 $\pm$ 0.008 & 0.503 $\pm$ 0.014 & 0.224 $\pm$ 0.045 \\
\bottomrule
\end{tabular}%
}
\end{table}

%% file: iclr2027_conference.bib
@book{shalevshwartz2014understanding,
  title     = {Understanding Machine Learning: From Theory to Algorithms},
  author    = {Shalev-Shwartz, Shai and Ben-David, Shai},
  year      = {2014},
  publisher = {Cambridge University Press}
}

@book{mohri2018foundations,
  title     = {Foundations of Machine Learning},
  author    = {Mohri, Mehryar and Rostamizadeh, Afshin and Talwalkar, Ameet},
  edition   = {2},
  year      = {2018},
  publisher = {MIT Press}
}

@inproceedings{Lipman2023flowmatching,
title={{Flow Matching for Generative Modeling}},
author={Yaron Lipman and Ricky T. Q. Chen and Heli Ben-Hamu and Maximilian Nickel and Matthew Le},
booktitle={The Eleventh International Conference on Learning Representations },
year={2023},
url={https://openreview.net/forum?id=PqvMRDCJT9t}
}

@misc{Li2026hardflow,
      title={{HardFlow: Hard-Constrained Sampling for Flow-Matching Models via Trajectory Optimization}}, 
      author={Zeyang Li and Kaveh Alim and Navid Azizan},
      year={2026},
      eprint={2511.08425},
      archivePrefix={arXiv},
      primaryClass={cs.LG},
      url={https://arxiv.org/abs/2511.08425}, 
}

@misc{Webber2026flowmpc,
      title={{Solving Inverse Problems with Flow-based Models via Model Predictive Control}}, 
      author={George Webber and Alexander Denker and Riccardo Barbano and Andrew J Reader},
      year={2026},
      eprint={2601.23231},
      archivePrefix={arXiv},
      primaryClass={eess.IV},
      url={https://arxiv.org/abs/2601.23231}, 
}

@inproceedings{Pourya2026flower,
  title={{Flower: A Flow-Matching Solver for Inverse Problems}},
  author={Mehrsa Pourya and Bassam El Rawas and Michael Unser},
  booktitle={The Fourteenth International Conference on Learning Representations},
  year={2026},
  url={https://openreview.net/forum?id=QGd34p02mI}
}

@InProceedings{Liu2023flowgrad,
    author    = {Liu, Xingchao and Wu, Lemeng and Zhang, Shujian and Gong, Chengyue and Ping, Wei and Liu, Qiang},
    title     = {{FlowGrad: Controlling the Output of Generative ODEs With Gradients}},
    booktitle = {Proceedings of the IEEE/CVF Conference on Computer Vision and Pattern Recognition (CVPR)},
    month     = {June},
    year      = {2023},
    pages     = {24335-24344}
}

@inproceedings{Wang2025ocflow,
title={{Training Free Guided Flow-Matching with Optimal Control}},
author={Luran Wang and Chaoran Cheng and Yizhen Liao and Yanru Qu and Ge Liu},
booktitle={The Thirteenth International Conference on Learning Representations},
year={2025},
url={https://openreview.net/forum?id=61ss5RA1MM}
}

@inproceedings{
    Yuan2025paretoflow,
    title={{ParetoFlow: Guided Flows in Multi-Objective Optimization}},
    author={Ye Yuan and Can Chen and Christopher Pal and Xue Liu},
    booktitle={The Thirteenth International Conference on Learning Representations},
    year={2025},
    url={https://openreview.net/forum?id=mLyyB4le5u}
}

@inproceedings{Annadani2025pgd,
 author = {Annadani, Yashas and Belakaria, Syrine and Ermon, Stefano and Bauer, Stefan and Engelhardt, Barbara},
 booktitle = {Advances in Neural Information Processing Systems},
 doi = {10.52202/085713-0533},
 editor = {D. Belgrave and C. Zhang and H. Lin and R. Pascanu and P. Koniusz and M. Ghassemi and N. Chen},
 pages = {15738--15761},
 publisher = {Curran Associates, Inc.},
 title = {{Preference-Guided Diffusion for Multi-Objective Offline Optimization}},
 url = {https://proceedings.neurips.cc/paper_files/paper/2025/file/175fa4fc8f275f877ec85340131c5d7a-Paper-Conference.pdf},
 volume = {38, Main Conference},
 year = {2025}
}

@inproceedings{
  Hotegni2026spread,
  title={{SPREAD: Sampling-based Pareto front Refinement via Efficient Adaptive Diffusion}},
  author={Hotegni, Sedjro Salomon and Peitz, Sebastian},
  booktitle={The Fourteenth International Conference on Learning Representations},
  year={2026},
  url={https://openreview.net/forum?id=4731mIqv89}
}

@inproceedings{
	Shrestha2026paretoconditioned,
	title={{Pareto-Conditioned Diffusion Models for Offline Multi-Objective Optimization}},
	author={Jatan Shrestha and Santeri Heiskanen and Kari Hepola and Severi Rissanen and Pekka J{\"a}{\"a}skel{\"a}inen and Joni Pajarinen},
	booktitle={The Fourteenth International Conference on Learning Representations},
	year={2026},
	url={https://openreview.net/forum?id=S2Q00li155}
}

@article{Zeni2025naturematerial,
  author  = {Zeni, Claudio and Pinsler, Robert and Z{\"u}gner, Daniel and Fowler, Andrew and Horton, Matthew and Fu, Xiang and Wang, Zilong and Shysheya, Aliaksandra and Crabb{\'e}, Jonathan and Ueda, Shoko and Sordillo, Roberto and Sun, Lixin and Smith, Jake and Nguyen, Bichlien and Schulz, Hannes and Lewis, Sarah and Huang, Chin-Wei and Lu, Ziheng and Zhou, Yichi and Yang, Han and Hao, Hongxia and Li, Jielan and Yang, Chunlei and Li, Wenjie and Tomioka, Ryota and Xie, Tian},
  title   = {{A generative model for inorganic materials design}},
  journal = {Nature},
  year    = {2025},
  volume  = {639},
  number  = {8055},
  pages   = {624--632},
  issn    = {1476-4687},
  doi     = {10.1038/s41586-025-08628-5},
  url     = {https://doi.org/10.1038/s41586-025-08628-5}
}

@article{Du2024naturesurvey,
  author  = {Du, Yuanqi and Jamasb, Arian R. and Guo, Jeff and Fu, Tianfan and Harris, Charles and Wang, Yingheng and Duan, Chenru and Li{\`o}, Pietro and Schwaller, Philippe and Blundell, Tom L.},
  title   = {{Machine learning-aided generative molecular design}},
  journal = {Nature Machine Intelligence},
  year    = {2024},
  volume  = {6},
  number  = {6},
  pages   = {589--604},
  issn    = {2522-5839},
  doi     = {10.1038/s42256-024-00843-5},
  url     = {https://doi.org/10.1038/s42256-024-00843-5}
}

@article{Pogue2023naturesuperconducting,
  author  = {Pogue, Elizabeth A. and New, Alexander and McElroy, Kyle and Le, Nam Q. and Pekala, Michael J. and McCue, Ian and Gienger, Eddie and Domenico, Janna and Hedrick, Elizabeth and McQueen, Tyrel M. and Wilfong, Brandon and Piatko, Christine D. and Ratto, Christopher R. and Lennon, Andrew and Chung, Christine and Montalbano, Timothy and Bassen, Gregory and Stiles, Christopher D.},
  title   = {{Closed-loop superconducting materials discovery}},
  journal = {npj Computational Materials},
  year    = {2023},
  volume  = {9},
  number  = {181},
  issn    = {2057-3960},
  doi     = {10.1038/s41524-023-01131-3},
  url     = {https://doi.org/10.1038/s41524-023-01131-3}
}

@article{Wang2021naturedistillation,
  author  = {Wang, Jike and Hsieh, Chang-Yu and Wang, Mingyang and Wang, Xiaorui and Wu, Zhenxing and Jiang, Dejun and Liao, Benben and Zhang, Xujun and Yang, Bo and He, Qiaojun and Cao, Dongsheng and Chen, Xi and Hou, Tingjun},
  title   = {{Multi-constraint molecular generation based on conditional transformer, knowledge distillation and reinforcement learning}},
  journal = {Nature Machine Intelligence},
  year    = {2021},
  volume  = {3},
  number  = {10},
  pages   = {914--922},
  issn    = {2522-5839},
  doi     = {10.1038/s42256-021-00403-1},
  url     = {https://doi.org/10.1038/s42256-021-00403-1}
}

@inproceedings{
Albergo2023building,
title={{Building Normalizing Flows with Stochastic Interpolants}},
author={Michael Samuel Albergo and Eric Vanden-Eijnden},
booktitle={The Eleventh International Conference on Learning Representations },
year={2023},
url={https://openreview.net/forum?id=li7qeBbCR1t}
}

@inproceedings{
Liu2022flow,
title={{Flow Straight and Fast: Learning to Generate and Transfer Data with Rectified Flow}},
author={Xingchao Liu and Chengyue Gong and qiang liu},
booktitle={NeurIPS 2022 Workshop on Score-Based Methods},
year={2022},
url={https://openreview.net/forum?id=gWxpdtQpiYV}
}

@article{Das1998nbi,
author = {Das, Indraneel and Dennis, J. E.},
title = {{Normal-Boundary Intersection: A New Method for Generating the Pareto Surface in Nonlinear Multicriteria Optimization Problems}},
journal = {SIAM Journal on Optimization},
volume = {8},
number = {3},
pages = {631-657},
year = {1998},
doi = {10.1137/S1052623496307510},
URL = {https://doi.org/10.1137/S1052623496307510},
eprint = {https://doi.org/10.1137/S1052623496307510}

}

@InProceedings{Coello2004IGD,
author="Coello Coello, Carlos A.
and Reyes Sierra, Margarita",
editor="Monroy, Ra{\'u}l
and Arroyo-Figueroa, Gustavo
and Sucar, Luis Enrique
and Sossa, Humberto",
title="A Study of the Parallelization of a Coevolutionary Multi-objective Evolutionary Algorithm",
booktitle="MICAI 2004: Advances in Artificial Intelligence",
year="2004",
publisher="Springer Berlin Heidelberg",
address="Berlin, Heidelberg",
pages="688--697",
isbn="978-3-540-24694-7"
}

@ARTICLE{Schutze2012,
  author={Sch{\"u}tze, Oliver and Esquivel, Xavier and Lara, Adriana and Coello, Carlos A. Coello},
  journal={IEEE Transactions on Evolutionary Computation}, 
  title={{Using the Averaged Hausdorff Distance as a Performance Measure in Evolutionary Multiobjective Optimization}}, 
  year={2012},
  volume={16},
  number={4},
  pages={504-522},
  doi={10.1109/TEVC.2011.2161872}
  }

@INPROCEEDINGS{Lamont1999GD,
  author={Van Veldhuizen, D.A. and Lamont, G.B.},
  booktitle={Proceedings of the 2000 Congress on Evolutionary Computation. CEC00 (Cat. No.00TH8512)}, 
  title={{On measuring multiobjective evolutionary algorithm performance}}, 
  year={2000},
  volume={1},
  number={},
  pages={204-211 vol.1},
  doi={10.1109/CEC.2000.870296}
  }

@InProceedings{Benhamu2024dflow,
  title = 	 {{D-Flow: Differentiating through Flows for Controlled Generation}},
  author =       {Ben-Hamu, Heli and Puny, Omri and Gat, Itai and Karrer, Brian and Singer, Uriel and Lipman, Yaron},
  booktitle = 	 {Proceedings of the 41st International Conference on Machine Learning},
  pages = 	 {3462--3483},
  year = 	 {2024},
  volume = 	 {235},
  series = 	 {Proceedings of Machine Learning Research},
  month = 	 {21--27 Jul},
  publisher =    {PMLR}
}

@article{CHEN2011classicalmoo,
title = {{Convergence of multi-objective evolutionary algorithms to a uniformly distributed representation of the {P}areto front}},
journal = {Information Sciences},
volume = {181},
number = {16},
pages = {3336-3355},
year = {2011},
issn = {0020-0255},
doi = {https://doi.org/10.1016/j.ins.2011.04.004},
url = {https://www.sciencedirect.com/science/article/pii/S0020025511001721},
author = {Yu Chen and Xiufen Zou and Weicheng Xie},
}

@InProceedings{Feng2025GuidanceFM,
  title = 	 {{On the Guidance of Flow Matching}},
  author =       {Feng, Ruiqi and Yu, Chenglei and Deng, Wenhao and Hu, Peiyan and Wu, Tailin},
  booktitle = 	 {Proceedings of the 42nd International Conference on Machine Learning},
  pages = 	 {16993--17029},
  year = 	 {2025},
  editor = 	 {Singh, Aarti and Fazel, Maryam and Hsu, Daniel and Lacoste-Julien, Simon and Berkenkamp, Felix and Maharaj, Tegan and Wagstaff, Kiri and Zhu, Jerry},
  volume = 	 {267},
  series = 	 {Proceedings of Machine Learning Research},
  month = 	 {13--19 Jul},
  publisher =    {PMLR},
  url = 	 {https://proceedings.mlr.press/v267/feng25s.html},
}

@inproceedings{Chung2023DPS,
title={{Diffusion Posterior Sampling for General Noisy Inverse Problems}},
author={Hyungjin Chung and Jeongsol Kim and Michael Thompson Mccann and Marc Louis Klasky and Jong Chul Ye},
booktitle={The Eleventh International Conference on Learning Representations },
year={2023},
url={https://openreview.net/forum?id=OnD9zGAGT0k}
}

@inproceedings{Kim2021noisescore,
title={{Noise2Score: Tweedie{\textquoteright}s Approach to Self-Supervised Image Denoising without Clean Images}},
author={Kwanyoung Kim and Jong Chul Ye},
booktitle={Advances in Neural Information Processing Systems},
editor={A. Beygelzimer and Y. Dauphin and P. Liang and J. Wortman Vaughan},
year={2021},
url={https://openreview.net/forum?id=ZqEUs3sTRU0}
}

@INPROCEEDINGS{Kim2025FlowDPS,
  author={Kim, Jeongsol and Kim, Bryan Sangwoo and Ye, Jong Chul},
  booktitle={2025 IEEE/CVF International Conference on Computer Vision (ICCV)}, 
  title={{FlowDPS: Flow-Driven Posterior Sampling for Inverse Problems}}, 
  year={2025},
  volume={},
  number={},
  pages={12328-12337},
  doi={10.1109/ICCV51701.2025.01146}
  }

@inproceedings{Kim2025testtime,
title={{Test-time Alignment of Diffusion Models without Reward Over-optimization}},
author={Sunwoo Kim and Minkyu Kim and Dongmin Park},
booktitle={The Thirteenth International Conference on Learning Representations},
year={2025},
url={https://openreview.net/forum?id=vi3DjUhFVm}
}

@book{villani2009optimal,
  author    = {Villani, C{\'e}dric},
  title     = {{Optimal Transport: Old and New}},
  series    = {Grundlehren der mathematischen Wissenschaften},
  volume    = {338},
  publisher = {Springer},
  address   = {Berlin, Heidelberg},
  year      = {2009},
  doi       = {10.1007/978-3-540-71050-9},
  isbn      = {978-3-540-71049-3}
}

@article{Peyre2019ComputationalOT,
    author = {Peyr{\'e}, Gabriel and Cuturi, Marco},
    title = {{Computational Optimal Transport with Applications to Data Sciences}},
    journal = {Foundations and Trends in Machine Learning},
    volume = {11},
    number = {5-6},
    pages = {355-607},
    year = {2019},
    month = {02},
    issn = {1935-8237},
    doi = {10.1561/2200000073},
    url = {https://doi.org/10.1561/2200000073},
    eprint = {https://www.emerald.com/ftmal/article-pdf/11/5-6/355/11154291/2200000073en.pdf},
}

@article{POT,
  author  = {R{\'e}mi Flamary and Nicolas Courty and Alexandre Gramfort and
             Mokhtar Z. Alaya and Aur{\'e}lie Boisbunon and Stanislas Chambon and
             Laetitia Chapel and Adrien Corenflos and Kilian Fatras and
             Nemo Fournier and L{\'e}o Gautheron and Nathalie T.H. Gayraud and
             Hicham Janati and Alain Rakotomamonjy and Ievgen Redko and
             Antoine Rolet and Antony Schutz and Vivien Seguy and
             Danica J. Sutherland and Romain Tavenard and Alexander Tong and
             Titouan Vayer},
  title   = {{POT: Python Optimal Transport}},
  journal = {Journal of Machine Learning Research},
  year    = {2021},
  volume  = {22},
  number  = {78},
  pages   = {1--8},
  url     = {http://jmlr.org/papers/v22/20-451.html}
}

@InProceedings{Trabucco2021COM,
  title = 	 {{Conservative Objective Models for Effective Offline Model-Based Optimization}},
  author =       {Trabucco, Brandon and Kumar, Aviral and Geng, Xinyang and Levine, Sergey},
  booktitle = 	 {Proceedings of the 38th International Conference on Machine Learning},
  pages = 	 {10358--10368},
  year = 	 {2021},
  editor = 	 {Meila, Marina and Zhang, Tong},
  volume = 	 {139},
  series = 	 {Proceedings of Machine Learning Research},
  month = 	 {18--24 Jul},
  publisher =    {PMLR},
  url = 	 {https://proceedings.mlr.press/v139/trabucco21a.html}
}

@InProceedings{Trabucco2022DesignBench,
  title = 	 {{Design-Bench: Benchmarks for Data-Driven Offline Model-Based Optimization}},
  author =       {Trabucco, Brandon and Geng, Xinyang and Kumar, Aviral and Levine, Sergey},
  booktitle = 	 {Proceedings of the 39th International Conference on Machine Learning},
  pages = 	 {21658--21676},
  year = 	 {2022},
  editor = 	 {Chaudhuri, Kamalika and Jegelka, Stefanie and Song, Le and Szepesvari, Csaba and Niu, Gang and Sabato, Sivan},
  volume = 	 {162},
  series = 	 {Proceedings of Machine Learning Research},
  month = 	 {17--23 Jul},
  publisher =    {PMLR},
  url = 	 {https://proceedings.mlr.press/v162/trabucco22a.html},
}

@inproceedings{Yuan2023ICT,
 author = {Yuan, Ye and Chen, Can (Sam) and Liu, Zixuan and Neiswanger, Willie and Liu, Xue (Steve)},
 booktitle = {Advances in Neural Information Processing Systems},
 editor = {A. Oh and T. Naumann and A. Globerson and K. Saenko and M. Hardt and S. Levine},
 pages = {55718--55733},
 publisher = {Curran Associates, Inc.},
 title = {{Importance-aware Co-Teaching for Offline Model-based Optimization}},
 url = {https://proceedings.neurips.cc/paper_files/paper/2023/file/ae8b0b5838ba510daff1198474e7b984-Paper-Conference.pdf},
 volume = {36},
 year = {2023}
}

@inproceedings{Chen2023TriMentoring,
 author = {Chen, Can (Sam) and Beckham, Christopher and Liu, Zixuan and Liu, Xue (Steve) and Pal, Chris},
 booktitle = {Advances in Neural Information Processing Systems},
 editor = {A. Oh and T. Naumann and A. Globerson and K. Saenko and M. Hardt and S. Levine},
 pages = {76619--76636},
 publisher = {Curran Associates, Inc.},
 title = {{Parallel-mentoring for Offline Model-based Optimization}},
 url = {https://proceedings.neurips.cc/paper_files/paper/2023/file/f189e7580acad0fc7fd45405817ddee3-Paper-Conference.pdf},
 volume = {36},
 year = {2023}
}

@inproceedings{Qi2022IOM,
 author = {Qi, Han and Su, Yi and Kumar, Aviral and Levine, Sergey},
 booktitle = {Advances in Neural Information Processing Systems},
 editor = {S. Koyejo and S. Mohamed and A. Agarwal and D. Belgrave and K. Cho and A. Oh},
 pages = {13226--13237},
 publisher = {Curran Associates, Inc.},
 title = {{Data-Driven Offline Decision-Making via Invariant Representation Learning}},
 url = {https://proceedings.neurips.cc/paper_files/paper/2022/file/559726fdfb19005e368be4ce3d40e3e5-Paper-Conference.pdf},
 volume = {35},
 year = {2022}
}

@inproceedings{Yu2021RoMA,
 author = {Yu, Sihyun and Ahn, Sungsoo and Song, Le and Shin, Jinwoo},
 booktitle = {Advances in Neural Information Processing Systems},
 editor = {M. Ranzato and A. Beygelzimer and Y. Dauphin and P.S. Liang and J. Wortman Vaughan},
 pages = {4619--4631},
 publisher = {Curran Associates, Inc.},
 title = {{RoMA: Robust Model Adaptation for Offline Model-based Optimization}},
 url = {https://proceedings.neurips.cc/paper_files/paper/2021/file/24b43fb034a10d78bec71274033b4096-Paper.pdf},
 volume = {34},
 year = {2021}
}

@InProceedings{Xue2024OFFMOO,
  title = {{Offline Multi-Objective Optimization}},
  author = {Xue, Ke and Tan, Rongxi and Huang, Xiaobin and Qian, Chao},
  booktitle = {Proceedings of the 41st International Conference on Machine Learning},
  pages = {55595--55624},
  year = {2024},
  editor = {Salakhutdinov, Ruslan and Kolter, Zico and Heller, Katherine and Weller, Adrian and Oliver, Nuria and Scarlett, Jonathan and Berkenkamp, Felix},
  volume = {235},
  series = {Proceedings of Machine Learning Research},
  month = {21--27 Jul},
  publisher = {PMLR},
  url = {https://proceedings.mlr.press/v235/xue24b.html}
}

@misc{Kim2025MOOReview,
      title={{Offline Model-Based Optimization: Comprehensive Review}}, 
      author={Kim, Minsu and Gu, Jiayao and Yuan, Ye and Yun, Taeyoung and Liu, Zixuan and Bengio, Yoshua and Chen, Can},
      year={2025},
      howpublished={arXiv preprint arXiv:2503.17286},
      eprint={2503.17286},
      archivePrefix={arXiv},
      primaryClass={cs.LG},
      url={https://arxiv.org/abs/2503.17286}, 
}

@article{Deb2002NSGA2,
author = {Deb, K. and Pratap, A. and Agarwal, S. and Meyarivan, T.},
title = {{A fast and elitist multiobjective genetic algorithm: NSGA-II}},
year = {2002},
issue_date = {April 2002},
publisher = {IEEE Press},
volume = {6},
number = {2},
issn = {1089-778X},
url = {https://doi.org/10.1109/4235.996017},
doi = {10.1109/4235.996017},
journal = {Trans. Evol. Comp},
month = apr,
pages = {182–197},
numpages = {16}
}

@InProceedings{Chen2018GradNorm,
  title = 	 {{GradNorm: Gradient Normalization for Adaptive Loss Balancing in Deep Multitask Networks}},
  author =       {Chen, Zhao and Badrinarayanan, Vijay and Lee, Chen-Yu and Rabinovich, Andrew},
  booktitle = 	 {Proceedings of the 35th International Conference on Machine Learning},
  pages = 	 {794--803},
  year = 	 {2018},
  editor = 	 {Dy, Jennifer and Krause, Andreas},
  volume = 	 {80},
  series = 	 {Proceedings of Machine Learning Research},
  month = 	 {10--15 Jul},
  publisher =    {PMLR},
  url = 	 {https://proceedings.mlr.press/v80/chen18a.html}
}

@inproceedings{Yu2020PCGrad,
 author = {Yu, Tianhe and Kumar, Saurabh and Gupta, Abhishek and Levine, Sergey and Hausman, Karol and Finn, Chelsea},
 booktitle = {Advances in Neural Information Processing Systems},
 editor = {H. Larochelle and M. Ranzato and R. Hadsell and M.F. Balcan and H. Lin},
 pages = {5824--5836},
 publisher = {Curran Associates, Inc.},
 title = {{Gradient Surgery for Multi-Task Learning}},
 url = {https://proceedings.neurips.cc/paper_files/paper/2020/file/3fe78a8acf5fda99de95303940a2420c-Paper.pdf},
 volume = {33},
 year = {2020}
}

@ARTICLE{Zitzler1999HV,
  author={Zitzler, E. and Thiele, L.},
  journal={IEEE Transactions on Evolutionary Computation}, 
  title={{Multiobjective evolutionary algorithms: a comparative case study and the strength Pareto approach}}, 
  year={1999},
  volume={3},
  number={4},
  pages={257-271},
  doi={10.1109/4235.797969}}

@article{Tanabe2020RE,
title = {{An easy-to-use real-world multi-objective optimization problem suite}},
journal = {Applied Soft Computing},
volume = {89},
pages = {106078},
year = {2020},
issn = {1568-4946},
doi = {https://doi.org/10.1016/j.asoc.2020.106078},
url = {https://www.sciencedirect.com/science/article/pii/S1568494620300181},
author = {Tanabe, Ryoji and Ishibuchi, Hisao}
}

@ARTICLE{Zitzler2000ZDT,
  author={Zitzler, Eckart and Deb, Kalyanmoy and Thiele, Lothar},
  journal={Evolutionary Computation}, 
  title={{Comparison of Multiobjective Evolutionary Algorithms: Empirical Results}}, 
  year={2000},
  volume={8},
  number={2},
  pages={173-195},
  doi={10.1162/106365600568202}
  }

@INPROCEEDINGS{Deb2002DTLZ,
  author={Deb, Kalyanmoy and Thiele, Lothar and Laumanns, Marco and Zitzler, Eckart},
  booktitle={Proceedings of the 2002 Congress on Evolutionary Computation. CEC'02 (Cat. No.02TH8600)}, 
  title={{Scalable multi-objective optimization test problems}}, 
  year={2002},
  volume={1},
  number={},
  pages={825-830 vol.1},
  doi={10.1109/CEC.2002.1007032}
  }

@inproceedings{Song2019Manifold,
 author = {Song, Yang and Ermon, Stefano},
 booktitle = {Advances in Neural Information Processing Systems},
 editor = {H. Wallach and H. Larochelle and A. Beygelzimer and F. d\textquotesingle Alch\'{e}-Buc and E. Fox and R. Garnett},
 pages = {},
 publisher = {Curran Associates, Inc.},
 title = {{Generative Modeling by Estimating Gradients of the Data Distribution}},
 url = {https://proceedings.neurips.cc/paper_files/paper/2019/file/3001ef257407d5a371a96dcd947c7d93-Paper.pdf},
 volume = {32},
 year = {2019}
}

@Article{Jones1998,
  author    = {Jones, Donald R. and Schonlau, Matthias and Welch, William J.},
  journal   = {Journal of Global Optimization},
  title     = {{Efficient Global Optimization of Expensive Black-Box Functions}},
  year      = {1998},
  issn      = {0925-5001},
  number    = {4},
  pages     = {455--492},
  volume    = {13},
  doi       = {10.1023/a:1008306431147},
  publisher = {Springer Science and Business Media LLC},
}

@Article{Knowles2006,
  author    = {Knowles, J.},
  journal   = {IEEE Transactions on Evolutionary Computation},
  title     = {{ParEGO: a hybrid algorithm with on-line landscape approximation for expensive multiobjective optimization problems}},
  year      = {2006},
  issn      = {1089-778X},
  number    = {1},
  pages     = {50--66},
  volume    = {10},
  doi       = {10.1109/tevc.2005.851274},
  publisher = {Institute of Electrical and Electronics Engineers (IEEE)},
}

@article{jin2011,
  title={{Surrogate-assisted evolutionary computation: Recent advances and future challenges}},
  author={Jin, Yaochu},
  journal={Swarm and Evolutionary Computation},
  volume={1},
  number={2},
  pages={61--70},
  year={2011},
  publisher={Elsevier}
}

@article{wang2018,
  title={{Offline data-driven evolutionary optimization using selective surrogate ensembles}},
  author={Wang, Handing and Jin, Yaochu and Sun, Chaoli and Doherty, John},
  journal={IEEE Transactions on Evolutionary Computation},
  volume={23},
  number={2},
  pages={203--216},
  year={2018},
  publisher={IEEE}
}

@article{yang2019,
  title={{Offline data-driven multiobjective optimization: Knowledge transfer between surrogates and generation of final solutions}},
  author={Yang, Cuie and Ding, Jinliang and Jin, Yaochu and Chai, Tianyou},
  journal={IEEE Transactions on Evolutionary Computation},
  volume={24},
  number={3},
  pages={409--423},
  year={2019},
  publisher={IEEE}
}

@inproceedings{Kumar2020Gen,
 author = {Kumar, Aviral and Levine, Sergey},
 booktitle = {Advances in Neural Information Processing Systems},
 editor = {H. Larochelle and M. Ranzato and R. Hadsell and M.F. Balcan and H. Lin},
 pages = {5126--5137},
 publisher = {Curran Associates, Inc.},
 title = {{Model Inversion Networks for Model-Based Optimization}},
 url = {https://proceedings.neurips.cc/paper_files/paper/2020/file/373e4c5d8edfa8b74fd4b6791d0cf6dc-Paper.pdf},
 volume = {33},
 year = {2020}
}

@inproceedings{Canas2012distOT,
 author = {Canas, Guillermo and Rosasco, Lorenzo},
 booktitle = {Advances in Neural Information Processing Systems},
 editor = {F. Pereira and C.J. Burges and L. Bottou and K. Weinberger},
 pages = {},
 publisher = {Curran Associates, Inc.},
 title = {Learning Probability Measures with respect to Optimal Transport Metrics},
 url = {https://proceedings.neurips.cc/paper_files/paper/2012/file/c54e7837e0cd0ced286cb5995327d1ab-Paper.pdf},
 volume = {25},
 year = {2012}
}

@article{Vayer2023distOT,
  author = {Titouan Vayer and R{\'e}mi Gribonval},
  title   = {Controlling Wasserstein Distances by Kernel Norms with Application to Compressive Statistical Learning},
  journal = {Journal of Machine Learning Research},
  year    = {2023},
  volume  = {24},
  number  = {149},
  pages   = {1--51},
  url     = {https://jmlr.org/papers/v24/21-1516.html}
}

@book{Santambrogio2015optimal,
  author    = {Filippo Santambrogio},
  title     = {Optimal Transport for Applied Mathematicians: Calculus of Variations, PDEs, and Modeling},
  publisher = {Birkh{\"a}user},
  year      = {2015},
  series    = {Progress in Nonlinear Differential Equations and Their Applications},
  volume    = {87},
  doi       = {10.1007/978-3-319-20828-2},
  isbn      = {978-3-319-20828-2}
}
